\documentclass{article}
\PassOptionsToPackage{numbers,sort&compress}{natbib}
\usepackage[preprint]{neurips_2026}
\newcommand{\MinimaxHorizon}{50000}
\newcommand{\MinimaxRows}{3}
\newcommand{\MinimaxArms}{12}
\newcommand{\MinimaxTrials}{40}

\newcommand{\MinimaxEpsilons}{0.00, 0.25, 0.50, 0.75, 1.00}
\newcommand{\MinimaxScalingHorizons}{3125, 6250, 12500, 25000, 50000}

\usepackage[utf8]{inputenc}
\usepackage[T1]{fontenc}
\usepackage{microtype}
\usepackage{comment}
\usepackage{xparse}
\usepackage[dvipsnames]{xcolor}
\definecolor{paleBlueA}{RGB}{15,132,246}
\definecolor{paleBlueB}{RGB}{5,125,245}
\definecolor{paleBlueC}{RGB}{220,240,255}
\definecolor{pillarblue}{RGB}{20,40,120}
\definecolor{softblue}{RGB}{60,120,180}
\definecolor{royalblue}{RGB}{65,105,225}
\definecolor{brickred}{RGB}{178,34,34}
\definecolor{forestgreen}{RGB}{34,139,34}
\definecolor{royalpurple}{RGB}{128,0,128}
\definecolor{slate}{RGB}{112,128,144}
\usepackage{graphicx}
\usepackage{subcaption}
\usepackage{booktabs}
\usepackage[inline]{enumitem}
\usepackage{amsmath,amssymb,amsfonts,mathtools,amsthm}
\usepackage{thmtools}
\usepackage{upgreek}
\allowdisplaybreaks
\usepackage{algorithm}
\usepackage{algpseudocode}

\usepackage[most]{tcolorbox}
\usepackage[textsize=tiny]{todonotes}
\usepackage{float}
\usepackage{wrapfig}
\usepackage{tikz}
\usetikzlibrary{arrows.meta, positioning, decorations.pathmorphing, shadows.blur, calc, backgrounds}

\usepackage{booktabs}
\usepackage[table]{xcolor}
\usepackage{pifont}
\usepackage{makecell}
\usepackage{array}

\definecolor{mygreen}{RGB}{0,140,72}
\definecolor{myred}{RGB}{180,35,24}

\newcommand{\cmark}{\textcolor{mygreen}{\ding{51}}}
\newcommand{\xmark}{\textcolor{myred}{\ding{55}}}

\newcolumntype{C}[1]{>{\centering\arraybackslash}m{#1}}
\usepackage[
  colorlinks=true,
  linkcolor=paleBlueA,
  citecolor=paleBlueB,
  urlcolor=blue,
  breaklinks=true
]{hyperref}
\hypersetup{
  linktocpage=true,
  pdfstartview=FitH,
  pdfpagemode=UseOutlines,
  pageanchor=true,
  plainpages=false,
  bookmarksnumbered=true,
  bookmarksopen=false,
  bookmarksopenlevel=1,
  hypertexnames=false,
  pdfhighlight=/O
}
\usepackage[capitalize,noabbrev]{cleveref}
\makeatletter
\newtheoremstyle{assumptionAdvStyle}{6pt}{6pt}{\itshape}{}{\bfseries}{.}{0.5em}{Assumption~\textcolor{softblue}{A\theassumptionAdv}\thmnote{~(#3)}}
\makeatother
\makeatletter
\newtheoremstyle{assumptionStocStyle}{6pt}{6pt}{\itshape}{}{\bfseries}{.}{0.5em}{Assumption~\textcolor{softblue}{S\theassumptionStoc}\thmnote{~(#3)}}
\makeatother
\theoremstyle{plain}
\newtheorem{theorem}{Theorem}[section]

\newtheorem{lemma}{Lemma}[section]

\theoremstyle{definition}
\newtheorem{definition}{Definition}[section]

\theoremstyle{remark}
\newtheorem{remark}{Remark}[section]
\theoremstyle{assumptionAdvStyle}
\newtheorem{assumptionAdv}{Assumption}[section]
\theoremstyle{assumptionStocStyle}
\newtheorem{assumptionStoc}{Assumption (Stoc.)}[section]

\newcommand{\Pillar}[1]{\textcolor{pillarblue}{\textbf{Pillar~#1}}}
\newtcolorbox{yellownote}{
  colback=yellow!8!white,
  colframe=softblue!40!gray,
  boxrule=0.5pt,
  arc=2pt,
  left=8pt,
  right=8pt,
  top=8pt,
  bottom=8pt,
  fonttitle=\bfseries,
  enhanced,
  shadow={0.5pt}{-0.5pt}{0pt}{black!15}
}
\title{Learning Safely Without Knowing the World:\\COMPASS‑Hedge}
\author{%
  Ting Hu\thanks{Equal contribution.} \\
  Department of Economics \\
  University of Wisconsin--Madison \\
  \texttt{ting.hu@wisc.edu} 
  \and
  {\bf Luanda Cai\footnotemark[1]} \\
  Department of Finance \\
  University of Wisconsin--Madison \\
  \texttt{luanda.cai@wisc.edu} 
  \And
  Emmanouil-Vasileios Vlatakis-Gkaragkounis \\
  Department of Computer Sciences \\
  University of Wisconsin--Madison \\
  \texttt{vlatakis@wisc.edu} \\
}
\begin{document}
\maketitle
\begin{abstract}
Online learning algorithms often faces a fundamental trilemma: balancing regret guarantees between adversarial and stochastic settings and providing baseline safety against a fixed comparator. While existing methods excel in one or two of these regimes, they typically fail to unify all three without sacrificing optimal rates or requiring oracle access to problem-dependent parameters. 
In this work, we bridge this gap by introducing COMPASS-Hedge. To the best of our knowledge, our algorithm is the first full-information anytime method to simultaneously achieve, up to logarithmic factors:
\emph{i) Minimax-optimal regret in adversarial environments} ;
\emph{ii) Instance-optimal, gap-dependent regret in stochastic environments;}
and \emph{iii) $\tilde{\mathcal{O}}(1)$ regret relative to a designated baseline policy.}
Crucially, COMPASS-Hedge is parameter-free and requires no prior knowledge of the environment's nature or the magnitude of the stochastic sub optimality gaps. Our approach hinges on a novel integration of adaptive pseudo-regret scaling and phase-based aggression, coupled with a comparator-aware mixing strategy. To the best of our knowledge, this provides the first ``best-of-three-world'' guarantee in the full-information setting, establishing that baseline safety does not have to come at the cost of worst-case robustness or stochastic efficiency.
\end{abstract}
\vspace{-1em}
\section{Introduction}
\vspace{-0.5em}
Online learning with expert advice provides a robust framework for sequential decision-making under uncertainty. In this setting, a learner must dynamically aggregate the opinions of a diverse pool of $A$ experts. Lacking prior knowledge of which expert is most reliable—or even the nature of the environment itself—the learner’s success hinges on balancing the exploitation of historically top-performing strategies with the exploration of those that may excel in the future. This model is remarkably universal: whether navigating the fluctuations of a volatile stock market \cite{cover1991universal} or the risk-sensitive outcomes of clinical trials \cite{villar2015multi}, the objective remains to minimize cumulative regret:
\begin{equation}
R_T(\mu)
\;=\;
\sum_{t=1}^T \langle \mu^t, c^t \rangle
\;-\;
\sum_{t=1}^T \langle \mu, c^t \rangle ,\tag{Regret}
\end{equation}
where $\mu^t$ is the learner’s strategy, $c^t$ the loss vector, and $\mu$ a fixed comparator.
Then, the classical desideratum is to ensure that $R_T(\mu)$ grows sublinearly in $T$ for any $\mu$, guaranteeing that the learner’s average performance asymptotically matches that of the best fixed strategy chosen in hindsight.

Historically, research has diverged into two distinct regimes. In adversarial version, where losses can be chosen maliciously, algorithms such as Hedge achieve minimax-optimal regret of order $O(\sqrt{T \log A})$ \citep{FreundSchapire1997,CesaBianchiLugosi2006}. Conversely, in stochastic environments where losses are i.i.d., substantially sharper guarantees are possible. In the presence of minimum suboptimality gap $\Delta$, the instance-optimal rate is
\(
O\!\left( \log A/ \Delta\right)
\)
-- significantly  smaller than the adversarial  rate $O(\sqrt{T \log A})$ \citep{mourtada2019optimality}.

The quest for a ``universal'' learner — one that performs optimally regardless of whether the environment is benign or malicious—has been a central theme in online learning research over the past decade, leading to the development of \textit{Best-of-Both-Worlds (BOBW)} algorithms. These methods simultaneously achieve minimax-optimal adversarial robustness and instance-optimal stochastic adaptivity, both in full-information and bandit models \citep{ bubeck2012best, zimmert2021tsallis}.

As online learning is increasingly deployed in life-critical domains, a third  need has emerged:
\begin{itemize}[leftmargin=2em,topsep=0pt,itemsep=0pt]
    \item[] \textit{{\bf Baseline safety:} The cumulative regret relative to a designated, trusted baseline policy $\mu^c$ remains near-constant $($e.g., $O(\mathrm{poly}\log T)$ or $\tilde O(1)$$)$.     }
\end{itemize}
This provides a  performance floor - a safety net - for the learner to aggressively pursue strategic exploration without the risk of catastrophic underperformance relative to the baseline.
In practice, $\mu^c$ often represents a conservative but trusted strategy, such as a standard-of-care medical protocol or a legacy financial index. 
While satisfying any two of these requirements—\emph{minimax robustness, instance-optimality, or baseline safety}—is often feasible, unifying all three remains a significant challenge, even under full-information feedback.
This tension motivates the central crux of our work:
\begin{center} \vspace{-0.5em}\textit{Can a single, parameter-free algorithm simultaneously guarantee minimax robustness, instance-optimal adaptivity, and near-constant baseline safety?} \end{center}\vspace{-0.5em}

\textbf{The Best-of-Three-Worlds Challenge. (BOTW) } 
Despite the maturity of best-of-two-worlds research, the aforementioned triadic tension remains unanswered with the existing literature presents a fragmented landscape:
\begin{itemize}[itemsep=0pt, topsep=-2pt, parsep=0pt]
    \item \textbf{Adversarial + Safety.}
    Phased-aggression methods \citep{even2008regret, MullerEtAl2025} achieve constant regret relative to a baseline and $\Theta(\sqrt{T})$ adversarial robustness, but fail to capture fast $1/\Delta$ rates in stochastic environments.
    \item \textbf{Stochastic + Safety.}
    In the appendix of \citet{MullerEtAl2025}, the authors propose an alternative algorithm for the stochastic setting with safety constraints. However, the resulting regret guarantees scale as $O(\sqrt{T})$, rather than achieving the instance-optimal rates. Moreover, this result relies on a different algorithmic design than the adversarially robust method above, thereby sacrificing universality.\footnote{While \citet{MullerEtAl2025} primarily study the bandit feedback model and impose an additional \emph{smoothed mixed comparator} assumption—requiring strictly positive probability on each expert—their approach can be viewed as a special case of full-information learning.}    
    \item \textbf{Adversarial + Stochastic (BOBW).}
    State-of-the-art BOBW methods \citep{mourtada2019optimality} adapt efficiently to stochastic gaps, but can suffer $\Omega(\sqrt{T})$ regret relative to a safe baseline policy $\mu^c$.
\end{itemize}
\vspace{-0.75em}
\paragraph{Technical Bottleneck: \emph{The Trade-off between Conservation and Aggression.}}
The fundamental difficulty in unifying these three ``worlds'' stems from a deep-seated
tension in the per-round update mechanism.
Baseline safety necessitates conservative mixing strategies that tightly control
deviation from a trusted comparator.
In contrast, achieving stochastic optimality requires selective and increasingly
aggressive updates that rapidly concentrate probability mass on the empirical leader
in order to exploit suboptimality gaps $\Delta_a$.
Moreover, state-of-the-art best-of-both-worlds methods lack explicit safeguards
for protecting a designated baseline under adversarial conditions. \footnote{A potential workaround could be to assume oracle access to environment-specific quantities,
such as suboptimality gaps or problem-dependent confidence bounds.
While such assumptions yield valid guarantees in principle,
they fundamentally alter algorithm's scope by embedding
environment-specific knowledge into its design.
Indeed, if tight upper bounds on stochastic instance-optimal rates were given as input,
\citep{MullerEtAl2025} can be adapted to satisfy safety constraints with minor
modifications. 

However, this reliance on problem-dependent parameters breaks universality and renders
algorithms brittle in realistic settings where such quantities are unknown,
time-varying, or ill-defined. Hence, designing a single, parameter-free algorithm that reconciles these objectives remains an 
open challenge.
}
\vspace{-1em}
\paragraph{Our Contributions.} We resolve the tension between robustness, adaptivity, and safety by introducing
\textsc{COMPASS-Hedge}
(\textbf{C}omparator-\textbf{O}riented \textbf{M}ixing with
\textbf{P}hase-\textbf{A}daptive \textbf{S}afeguard \textbf{S}caling).
To the best of our knowledge, this is the first parameter-free algorithm to satisfy
the Best-of-Three-Worlds desiderata without prior knowledge of the environment’s nature
or the magnitude of stochastic gaps.
Formally, under standard assumptions on adaptivity and stochasticity, we show that \textsc{COMPASS-Hedge} simultaneously achieves:
\begin{itemize}[leftmargin=1.2em, itemsep=0pt, topsep=0pt]
    \item \textbf{Baseline Safety:}
    $\mathcal{R}(\mu^c)=O(\log^2 T)$ uniformly over all environments.
    \item \textbf{Adversarial Robustness:}
    $O(\sqrt{T\log A}\,\log^2 T)$ regret against any comparator.
    \item \textbf{Stochastic Adaptivity:}
    Instance-optimal regret
    $O\!\left({\log A \log^2 T}/{\Delta}\right)$ in i.i.d.\ settings.
\end{itemize}
\begin{table}[t]
\centering
\caption{The three-world landscape. 
Here $\tilde O(\cdot)$ hides polylogarithmic factors in $T$.}
\label{tab:comparison}
\setlength{\tabcolsep}{5.5pt}
\renewcommand{\arraystretch}{1.12}
\begin{tabular}{>{\raggedright\arraybackslash}p{5.1cm} C{2.1cm} C{2.45cm} C{2.15cm}}
\toprule
\rowcolor{black!6}
\textbf{Representative works}
& \makecell{\textbf{Minimax}\\[-0.5mm]\small $\tilde O(\sqrt{T})$}
& \makecell{\textbf{Stochastic}\\[-0.5mm]\small $\tilde O(\log A/\Delta)$}
& \makecell{\textbf{Baseline}\\[-0.5mm]\small $\tilde O(1)$} \\
\midrule
\makecell[l]{
\small Even-Dar et al.~\citep{even2008regret}; Müller et al.~\citep{MullerEtAl2025}}
& \cmark & \xmark & \cmark \\
\midrule
\makecell[l]{
\small Müller et al.~\citep{MullerEtAl2025}, App.~B.3}
& \xmark & \xmark & \cmark \\
\midrule
\makecell[l]{
\small Mourtada--Gaïffas~\citep{mourtada2019optimality}}
& \cmark & \cmark & \xmark \\
\midrule
\rowcolor{blue!7}
\makecell[l]{\textsc{COMPASS-Hedge} (ours)}
& \cmark & \cmark & \cmark \\

\bottomrule
\end{tabular}
\vspace{-1.5em}
\end{table}
\textbf{Technical Innovation.} The design of \textsc{COMPASS-Hedge} introduces a novel ``safe-to-aggressive'' transition
mechanism that avoids the pitfalls of fixed-parameter safety.
Our approach rests on three fundamental pillars:
\paragraph{1. Autonomous Aggression via Geometric Scaling.}
The core challenge in Best-of-Three-Worlds learning is calibrating the learner’s
\emph{aggression}—the degree to which it deviates from the baseline policy $\mu^c$
to exploit potential suboptimality gaps.
While existing phased-aggression frameworks require a predefined regret budget
\citep{even2008regret, MullerEtAl2025}, \textsc{COMPASS-Hedge} employs an online
geometric estimation scheme. 
We maintain a dynamic estimator of the pseudo-regret and double it only when the
cumulative observed loss exceeds the current threshold\footnote{Thus, our scheme bypasses the impossibility result of \citet{besson2018doubling},
which rules out horizon-free logarithmic guarantees for the \emph{oblivious} 
doubling trick. In contrast, our regret-scaled, data-dependent mechanism avoids any fixed
geometric schedule and thus preserves logarithmic bounds.}.
This induces a self-correcting feedback loop: in benign environments, the learner
remains conservative and safe; in adversarial settings, it automatically scales its
learning rate to recover $O(\sqrt{T})$ adversarial robustness without manual tuning, 
a technique that may be of independent interest for other adaptive online learning problems where the scale of the regret is not known a priori.
\paragraph{2. Bridging the Expected\textendash Pseudo Regret Gap.}
A major theoretical hurdle in using empirical observations to drive the
aforementioned estimation scheme is the discrepancy between $i)$ the
\emph{\small expected adversarial regret} (the theoretical objective), $ii)$
the \emph{\small observed regret} (the quantity available to the algorithm),
and $iii)$ the \emph{\small pseudo-regret} (the quantity appearing in safety guarantees).
We rigorously control this discrepancy, denoted by $\Delta_{\mathrm{gap}}$,
providing a unified treatment across environments:
{\vspace{-0.225em}
\[
\Delta_{\mathrm{gap}}(T)
\;=\;
\mathbb{E}\!\left[
\max_{\mu \in \Delta_A}
\sum_{t=1}^T \langle c^t, \hat \mu^t - \mu \rangle
\right]
\;-\;
\max_{\mu \in \Delta_A}
\mathbb{E}\!\left[
\sum_{t=1}^T \langle c^t, \hat \mu^t - \mu \rangle
\right].
\vspace{-0.225em}
\]}Leveraging concentration inequalities, we show that $\Delta_{\mathrm{gap}} (T) = O(\frac{\log A}{\Delta})$ in stochastic
regimes with a unique optimal expert, while $\Delta_{\mathrm{gap}} (T) = O(\sqrt{T})$
under adversarial losses.

To bridge these regimes, we distinguish between two adversarial models.
For \emph{oblivious adversaries}—where losses are fixed in advance or generated
from a random instance selected at $t=0$—our proof techique follows directly.
For \emph{adaptive adversaries}, we invoke Assumption~\ref{eq:A1}, which limits the
adversary’s cumulative adaptivity.
This prevents the environment from ``tricking'' the doubling mechanism by
reacting too aggressively to the learner’s scaling, and establishes that the
empirical scheme remains a faithful proxy for the underlying safety constraints.
\paragraph{3. Horizon-Free Self-Scaling and Comparator Mixing.}
To ensure that the algorithm is horizon-free, \textsc{COMPASS-Hedge} employs
a decaying step size (e.g., $1/\sqrt{t}$) within the underlying Hedge framework.
We extend the phased-aggression logic by integrating it with a
\emph{comparator-aware mixing strategy}.
The key insight is that the mixture weight assigned to the baseline policy
$\mu^c$ is inversely proportional to the current ``aggression budget.''
\vspace{-0.25em}
\newpage
Two main phases are introduced in our algorithm:
\begin{itemize}[itemsep=0pt, topsep=0pt]
    \item \textbf{Initial Phase.}
    When the estimated regret against the baseline is large, the mixture
    closely tracks $\mu^c$, providing immediate protection.    
    \item \textbf{Adaptive Phase.}
    If the algorithm gains evidence against the baseline,
    the mixture relaxes
\end{itemize}
\vspace{-0.25em}
Just as in Follow-The-Regularized-Leader (FTRL), where the step size controls
the strength of regularization, our approach uses the doubling-trick estimator
as a \emph{self-scaling regularizer} for the mixture.
This ensures that the algorithm only ``breaks away'' from the safety of the
baseline when the statistical evidence for a superior strategy is overwhelming.

\textbf{\emph{Summary. }}{\it\textsc{COMPASS-Hedge} unifies these objectives through a symbiotic architecture: \Pillar{3} establishes a horizon-free, comparator-aware mixing foundation, which is operationalized by \Pillar{1}'s parameter-free estimation of pseudo-regret. Finally, \Pillar{2} serves as the theoretical anchor; by rigorously bounding the expected-pseudo regret gap, it ensures our empirical scaling remains robust and universal across both stochastic and adversarial regimes.}
\section{Related work and motivating scenarios}
\vspace{-0.5em}\paragraph{Regret and Best of Both Worlds.}
The regret properties of online learning have been extensively studied; for general overviews, see \citet{hazan2016introduction, shalev2012online}.
The challenge of unifying robustness, adaptivity, and safety lies at the intersection of three mature yet largely disjoint research directions.
While recent BOBW-algorithms achieve stochastic instance-optimality alongside adversarial robustness—both in the full-information setting \citep{mourtada2019optimality} and under bandit feedback \citep{zimmert2021tsallis}—the trajectory of this field has increasingly pivoted toward multi-armed bandit variants \citep[e.g.,][]{li2025efficient}. 
Conversely, the safety-oriented literature even  in the full-information regime has remained more restrictive. 
\vspace{-1.25em}\paragraph{Bicriteria Measures and Baseline Safety.}A seminal precursor to our work is \citet{even2008regret}, who investigated a {bicriteria} performance measure: simultaneously minimizing regret against the best expert while guaranteeing performance against the {average} of all experts (or a fixed weighting). They formalized the fundamental observation that for comparisons against a fixed baseline, the ``gold standard'' is \emph{constant regret} (independent of $T$). However, while achieving zero regret against the average in isolation is trivial (via uniform weights), obtaining this simultaneously with worst-case guarantees is non-trivial. \citet{even2008regret} addressed this via modified MWU (e.g., Prod) or conservative phased strategies in Exponential Weights—techniques later adapted by \citet{MullerEtAl2025} for safety in bandit settings. However, a parameter-free unification that achieves "Best-of-Three-Worlds" guarantees
remained an open challenge in the full-information regime.
\vspace{-1.25em}\paragraph{The Practical Imperative of Full Information.} Finally, we note that in many safety-critical domains,
full-information feedback is the natural modeling choice, as counterfactual costs
are observable post hoc, rendering bandit formulations unnecessarily restrictive.
For example, in \textit{algorithmic finance}, all asset returns are publicly observed,
calling for strategies that exploit market structure without underperforming
benchmark indices \citep{cover1991universal, gupta2024ai}.
Similarly, \textit{smart-grid load balancing} relies on known routing costs to optimize
efficiency while strictly adhering to stability constraints \citep{tang2016online}.
{Ultimately, these domains highlight the critical need for both theoretical foundations and practical solutions in principled safety-aware deployment.}

%
%
%
\vspace{-0.75em}
\section{Preliminaries and problem setup}
\vspace{-0.5em}
\label{sec:preliminaries}
\paragraph{Notation.}
We use standard $O$-asymptotic notation. Let $[n] := \{1,\ldots,n\}$ for $n \in \mathbb{N}$; $\langle \cdot, \cdot \rangle$ and $\|\cdot\|$ denote the Euclidean inner product and norm; $\lceil x \rceil$ is the ceiling of $x \in \mathbb{R}$; $e_i$ is the $i$-th standard basis vector of $\mathbb{R}^n$; and $\Delta_n := \{x \in \mathbb{R}^n_{\ge 0} : \sum_i x_i = 1\}$ is the $n$-simplex. We consider full-information online learning with $A \ge 2$ arms over horizon $T \ge 1$ on a probability space $(\Omega, \mathcal{G}, \mathbb{P})$; the game proceeds in discrete rounds $t = 1, \dots, T$. In addition to the standard framework, the player receives a \emph{safe comparator strategy} $\mu^c \in \Delta_A$ as input. The objective is twofold: to compete with the best arm in hindsight (optimality), while ensuring performance effectively at least as good as $\mu^c$ (safety).
\vspace{-0.5em}\paragraph{Filtration and Adversarial Models.}
To rigorously capture the information flow and the power of the adversary, we define the history and the associated filtration.
Let $H_{t-1} := (\mu^1, c^1, \dots, \mu^{t-1}, c^{t-1})$ denote the history of interactions up to the beginning of round $t$. We view $H_{t-1}$ as a random variable taking values in the history space $\mathsf{H}_{t-1} = (\Delta_A \times [0,1]^A)^{t-1}$.
For any sample path $\omega \in \Omega$, we denote the realized history as $h_{t-1} := H_{t-1}(\omega)$.

We define the natural filtration $\mathbb{F} = (\mathcal{F}_t)_{t \ge 0}$ where $\mathcal{F}_t := \sigma(\mu^1, c^1, \dots, \mu^t, c^t)$ represents the $\sigma$-algebra generated by all variables revealed up to the end of round $t$. Consequently, $\mathcal{F}_{t-1}$ encapsulates all information available to the learner before making the decision $\mu^t$. This formalism allows us to distinguish between different environmental assumptions standard in the literature \citep{borodin-1998, CesaBianchiLugosi2006}:
\begin{itemize}[leftmargin=0.2em, topsep=0pt, itemsep=0pt]
    \item \textbf{Oblivious Adversary:} The loss sequence $c^{1:T}$ is fixed in advance or depends only on extrinsic randomness, independent of the learner's actions. Formally, $c^t$ is independent of $\mathcal{F}_{t-1}$.
    \item \textbf{Adaptive Adversary:} The adversary may choose $c^t$ based on the realized history $h_{t-1}$. Here, $c^t$ is $\mathcal{F}_{t-1}$-measurable (or measurable with respect to $\mathcal{F}_{t-1}$ augmented by internal randomization).
\end{itemize}

In this work, we consider a general setting that encompasses both, characterized by the \emph{conditional mean loss}. We denote:
\[
m_t(a \mid h) := \mathbb{E}[c^t(a) \mid H_{t-1}=h].
\]
Throughout the paper, $\mathbb{E}[\cdot]$ denotes the expectation with respect to the internal randomness of both the algorithm and the environment. 

\subsection{Performance metrics}
In online learning, the primary objective is to minimize the gap between the learner’s cumulative loss and that of the best fixed strategy in hindsight, yielding the standard notion of \textit{Expected Regret}.
\begin{definition}[Expected Regret]
\label{def:expected-regret}
The expected regret of an algorithm is defined as the expectation of the realized regret against the best arm in hindsight:
\(
\mathcal{R}_{\text{exp}}(T) \triangleq \mathbb{E}\left[ \sum_{t=1}^T \langle c^t, \mu^t \rangle - \min_{\mu \in \Delta_A} \sum_{t=1}^T \langle c^t, \mu \rangle \right].
\)
\end{definition}
This metric captures the algorithm's performance relative to the optimal action for the \textit{specific realization} of the loss sequence.
In our setting, however, competing with the best expert is not enough; we must also respect a specific safety constraint.
\begin{definition}[Comparator Regret]
\label{def:comparator-regret}
The regret compared to a specific fixed strategy $\mu^c \in \Delta_A$ is:
\( \mathcal{R}(\mu^c) \triangleq \mathbb{E}\left[\sum_{t=1}^T \langle c^t, \mu^t - \mu^c \rangle\right]. \)
\end{definition}
Our goal is to design a single algorithm that simultaneously guarantees $\mathcal{R}(\mu^c) \le \tilde{\mathcal{O}}(1)$ (constant-like safety) while minimizing the worst-case regret against the best expert.
\vspace{-1em}
\paragraph{From Expected to Pseudo-Regret.}
Directly analyzing $\mathcal{R}_{\text{exp}}$ is often intractable due to the correlation between the algorithm's choices and the identity of the hindsight optimizer.  More concretely, the hindsight optimizer is also a random variable determined by each random history path.
To facilitate analysis, we utilize the notion of \textit{Pseudo-Regret}, which compares the learner's expected loss to the expected loss of any fixed comparator.

\begin{definition}[Pseudo-Regret]
\label{def:pseudo-regret}
The worst-case pseudo-regret is defined as:
\( \mathcal{R}_{\text{pseudo}}(T)\triangleq  \max_{\mu \in \Delta_A}\sum_{t=1}^T \mathbb{E}\left[\langle c^t, \mu^t - \mu \rangle\right].  \)
\end{definition}

\begin{remark}[Discussion on the Gap.]\it
$\mathcal{R}_{\text{exp}}$ and $\mathcal{R}_{\text{pseudo}}$ differ in the order of max and expectation: $\mathcal{R}_{\text{exp}}$ compares against the random empirical minimizer $\mathbb{E}[\min \dots]$, while $\mathcal{R}_{\text{pseudo}}$ compares against $\min \mathbb{E}[\dots]$. Jensen's inequality gives $\mathcal{R}_{\text{exp}} \ge \mathcal{R}_{\text{pseudo}}$, so pseudo-regret is a lower bound. Lemma~\ref{lem:pseudo-gap} shows the gap is $O(\sqrt{T})$ under adversarial environment with controlled adaptivity, and \(O(\frac{\log A}{\Delta})\) under stochastic environment with unique optimizer.
Thus, optimizing for pseudo-regret serves as a valid proxy for the true objective.
\end{remark}

\subsection{Model assumptions}
\subsubsection{Beyond oblivious adversaries}
\label{subsec:adv_assumptions}

While standard Best-of-Both-Worlds (BOBW) approaches typically assume an oblivious adversary, our framework extends to \emph{adaptive} environments. However, it is a known impossibility result that against a fully adaptive adversary—one with unlimited memory and sensitivity—policy regret is lower-bounded by $\Omega(T)$ \citep{arora2012online}. To bypass this barrier, the literature typically imposes structural constraints: limiting the adversary's memory \citep{arora2012online} or  smoothed adversaries \citep{haghtalab2024smoothed}.

We introduce a third, distinct perspective: instead of limiting \textit{what} the adversary remembers or \textit{how concentrated} their randomness is, we limit the \textit{magnitude} of their reaction.

\begin{assumptionAdv}[Controlled Adaptivity]
\label{A1}
For each $t \ge 1$ and $a \in [A]$, assume there exists a deterministic nonnegative sequence $\{\varepsilon_t\}_{t \ge 1}$ such that the conditional mean loss satisfies:
\begin{equation}
\sup_{h, h' \in \mathsf{H}_{t-1}} \big| m_t(a \mid h) - m_t(a \mid h') \big| \le \varepsilon_t. \tag{\textcolor{softblue}{A1}} \label{eq:A1}
\end{equation}
We define the cumulative adaptivity for any interval $I$ as $E_I := \sum_{t \in I} \varepsilon_t$. We assume $E_I \le C_E \sqrt{|I|}$ for some constant $C_E > 0$.
\end{assumptionAdv}

\begin{remark}\it
\eqref{eq:A1} positions our model uniquely within the landscape of adaptive adversaries:

\begin{itemize}[leftmargin=0pt, itemsep=1pt, topsep=1pt, parsep=0pt]
    \item[]\textbf{\ref{eq:A1}  vs. Memory-Bounded:} An $m$-memory bounded adversary \citep{arora2012online} is constrained to depend only on the $m$ most recent actions. While this limits the \textit{temporal scope} of dependency, it allows for arbitrary discontinuities ("jumps") based on recent history. In contrast, our model allows for \textbf{infinite memory} ($\infty$-memory bounded), provided the aggregate "drift" or sensitivity to history remains within the budget $O(\sqrt{T})$.
    \item[]\textbf{\ref{eq:A1}  vs. Smoothed Analysis:} Recent works like \citet{haghtalab2024smoothed} assume the adversary plays from distributions with bounded density (anti-concentration). 
    Our model does not require the adversary to be noisy or smooth; it allows for deterministic behavior (Dirac deltas), provided the deterministic choice is stable over histories.
\end{itemize}
\end{remark}
\begin{remark}\it
Notably, our framework is not strictly stronger than these alternatives but rather \emph{orthogonal} to them. For example, a switching-cost adversary satisfies the $m=1$ memory constraint (local dependence), yet may violate our stability condition by inducing large instantaneous jumps in the loss landscape. As a result, neither set of assumptions subsumes the other. 
\end{remark}
In essence, \eqref{eq:A1} restricts the environment from changing the ``rules of the game'' too aggressively in response to the learner's past actions, thereby capturing a \emph{drifting}—as opposed to vindictive—adversarial behavior. To facilitate concentration analysis, we also control the variance of the noise around the conditional mean.

\subsubsection{Optimal arm identifiability}
In the stochastic regime, the environment behaves probabilistically with a fixed underlying structure.
\begin{assumptionStoc}[i.i.d. Losses]
\label{Stoc1}
The losses $c^1,..., c^T$ are independent and identically distributed.
\end{assumptionStoc}
\begin{assumptionStoc}[Unique Optimizer and Positive Gap]
\label{Stoc2}
There exists a \textbf{unique} optimal arm
\begin{equation}
\tag{\textcolor{softblue}{S2}}\label{eq:S2}
a^\star \in \operatorname{argmin}_{a} \mathbb{E}[c^t(a)].
\end{equation}
Let $\Delta_a := \mathbb{E}[c^t(a) - c^t(a^\star)]$ denote the suboptimality gap. Consequently, we assume $\Delta := \min_{a \neq a^\star} \Delta_a > 0.$
\end{assumptionStoc}

\begin{remark}[On Identifiability] \it \eqref{eq:S2} is a standard yet non-trivial requirement in gap-dependent analysis. It ensures a strictly positive margin $\Delta$, allowing the algorithm such as the UCB algorithm \citep{auer2002finite}. to settle on $a^\star$ and achieve logarithmic instance-optimal regret $O(\frac{\log T}{\Delta})$. 
The uniqueness of the optimal arm ensures an identifiable gap structure under which \emph{pseudo-regret} faithfully tracks \emph{expected regret}.

In contrast, when multiple arms are optimal, \citet{mourtada2019optimality} show that the discrepancy between these metrics can remain asymptotically non-negligible (scaling as $\Omega(\sqrt{T})$) due to random fluctuations causing the empirical minimizer to oscillate among optimal arms. Consequently, addressing this regime would require techniques fundamentally different from the ``regret self-estimation'' framework developed here; thus, we highlight the relaxation of this assumption as a major open question for future research.
\end{remark}

A key consequence of the above model assumptions is a bound on the gap between
pseudo-regret and expected regret over arbitrary intervals. 

\begin{lemma}[Pseudo-regret versus expected regret]
\label{lem:pseudo-gap}
Let $I \subseteq \{1,\dots,T\}$ be any interval. Then,
\[
\mathcal{R}_{\mathrm{pseudo}}(I)
\;\le\;
\mathcal{R}_{\mathrm{exp}}(I)
\;\le\;
\mathcal{R}_{\mathrm{pseudo}}(I)
 + \mathbb{G}(I),
\]
where $\mathcal{R}_{\mathrm{exp}}(I)$ and $\mathcal{R}_{\mathrm{pseudo}}(I)$ denote,
respectively, the expected regret and the pseudo-regret accumulated over $I$.
Moreover, the remainder term $\mathbb{G}(I)$ satisfies
\[
\mathbb{G}(I) =
\Big\{
O\!\left(\sqrt{|I|}\right),
 \text{under (\hyperref[A1]{\textcolor{softblue}{A1}}\textendash adversarial regime)}
 \lor
O(\frac{\log A}{\Delta}),
 \text{under (\textcolor{softblue}{\hyperref[Stoc1]{\textcolor{softblue}{S1\&2}}}\textendash stochastic regime)}.
\Big\}
\]
\end{lemma}

\begin{remark} In words, \cref{lem:pseudo-gap} bounds the gap between pseudo-regret and expected regret.
In the adversarial setting, the difference is controlled via moderate history dependence of conditional mean loss, and bounded variance of martingale noises. This leads to an $O(\sqrt{|I|})$ gap between pseudo-regret and expected regret. In the stochastic setting, i.i.d. cost and unique optimal arm assumption imply that this gap is of optimal gap-dependent rate. This observation is the key technical turning point: it allows us to analyze \emph{expected} regret while still using data-dependent stopping rules in the algorithm, even against adaptive adversaries.
\end{remark}

\section{Our results}
\subsection{Algorithm description and intuition}
\begin{wrapfigure}{r}{0.45\textwidth}
\vspace{-5\baselineskip}
\begin{minipage}{\linewidth}
\begin{algorithm}[H]
\caption{\small\textsc{Compass-Hedge}: \textbf{C}omparator-\textbf{O}riented \textbf{M}ixing with \textbf{P}hase-\textbf{A}daptive \textbf{S}afeguard \textbf{S}caling Hedge.}
\label{alg:doubling_ftrl_comparator}
\scriptsize
\begin{algorithmic}[1]
\State \textbf{Input:} arms $A$, safe comparator $\mu^c$, budget $\widehat R_1 \gets 2$
\State \textbf{Definition:} $\mathcal{L}_{start}^{t}(\mu) := \sum_{j=\text{start}}^t \langle c^j, \mu \rangle$, $\widehat {\mathcal{L}}_{start}^{t} := \sum_{j=\text{start}}^t \langle c^j, \hat{\mu}^j \rangle$
\State \textbf{Init:} $\hat{\mu}^{\,1} \gets \text{Unif}(A)$; $\alpha \gets 1/\widehat R_1$; $\text{start}, k, s \gets 1$
\For{$t = 1,2,\ldots$}
    \State Play $\mu^t := \alpha\,\hat{\mu}^t + (1-\alpha)\,\mu^c$; observe $c^t$
    \State \textcolor{royalblue}{\textbf{// 1. Check Stage Condition}}
    \State $R_{\text{hedge}} \gets \max_{\mu} (\widehat{\mathcal{L}}_{start}^{t} - \mathcal{L}_{start}^t(\mu))$
    \If{$R_{\text{hedge}} \le \widehat R_s$}
        \State \textcolor{royalpurple}{\textbf{// 2. Check Phase Condition}}
        \State $R_{\text{comp}} \gets \max_{\mu} (\mathcal{L}_{start}^{t}(\mu^c) - \mathcal{L}_{start}^t(\mu))$
        \If{$R_{\text{comp}} > 2\widehat R_{s}$ and $\alpha < 1$}
            \State \textcolor{brickred}{\emph{// Aggressive Shift}}
            \State $k \gets k+1$; $\text{start} \gets t+1$
            \State $\alpha \gets \min(2^{k-1}/( \widehat R_s + 1), 1)$
            \State Reset $\hat{\mu}^{t+1}$ to uniform
        \Else
            \State \textcolor{forestgreen}{\emph{// No-Regret Update}}
            \State $\eta_t \gets 2\sqrt{\tfrac{\log A}{t - \text{start} + 1}}$
            \State $\hat{\mu}^{t+1} \!\in\! \arg\min_{\mu\in\Delta_A}\!\bigl\{\mathcal{L}_{\text{start}}^{t}(\mu) + \tfrac{\mathrm{Ent}(\mu)}{\eta_t}\bigr\}$
        \EndIf
    \Else
        \State \textcolor{royalblue}{\emph{// Stage Reset (Double Budget)}}
        \State $s \gets s+1$; $\text{start}\gets t+1$
        \State $\widehat R_s \gets \widehat R_{s-1}\!\cdot\! 2^{\lceil \log_2(R_{\text{hedge}}/\widehat R_{s-1})\rceil}$
        \State $\alpha \gets 1/\widehat R_s$; $k \gets 1$; $\hat{\mu}^{t+1} \gets \text{Unif}(A)$
    \EndIf
\EndFor
\end{algorithmic}
\end{algorithm}
\end{minipage}
\vspace{-2\baselineskip}
\end{wrapfigure}
 
We now introduce \textsc{Compass-Hedge} (\cref{alg:doubling_ftrl_comparator}), a meta-algorithm designed to safeguard a base learner (Hedge) using a conservative comparator $\mu^c$. The core intuition relies on a \textbf{hierarchical geometric regret estimation} that manages the trade-off between ``safety'' (adhering to $\mu^c$) and ``optimality'' (following the Hedge expert).
 
The algorithm operates on two time-scales:
\begin{enumerate}[leftmargin=8pt,topsep=0pt,itemsep=1pt]
    \item \textbf{Stages ($s$):} The outer loop controls the global regret budget. Each stage $s$ has pseudo-regret capacity $\widehat R_s$. If the Hedge expert's regret exceeds this capacity, the stage ends and we double the budget ($\widehat R_{s+1} \approx 2 \widehat R_s$), so the total number of stages is $O(\log T)$.
    \item \textbf{Phases ($k[s]$):} Within a stage, time is divided into phases. Each phase fixes a mixing coefficient $\alpha \in [0,1]$ and plays $\mu^t = \alpha \hat{\mu}^t + (1-\alpha)\mu^c$: \emph{(i)} start with small $\alpha$ (mostly safe); \emph{(ii)} exit the phase (and increase $\alpha$) only if $\mu^c$ is verified significantly suboptimal, i.e.\ the Hedge expert has ``earned'' the right to be played more aggressively; \emph{(iii)} this ``phase-budgeted'' approach never incurs large regret against $\mu^c$ unless we have already banked significant gains against the environment.
\end{enumerate}
 
\textbf{Intuition.} \textsc{Compass-Hedge} follows a ``probationary'' principle: it anchors to the safe comparator and shifts weight to the no-regret learner only when $\mu^c$ is empirically proven suboptimal. This ensures \emph{safety} (\texttt{lines 11, 14}): requiring ``proof of failure'' from the safe arm before abandoning it caps potential losses and guarantees constant per-stage regret against the comparator.
 
Calibrating the rejection threshold, however, poses a dilemma: identifying the precise magnitude of cumulative regret that distinguishes suboptimality from transient noise. A static worst-case barrier of order $\Theta(\sqrt{T})$ would suffice but is overly conservative in benign (``beyond worst-case'') environments, effectively destroying instance-optimality. To overcome this, we deploy a \emph{\textbf{hierarchical doubling trick estimation}} (\texttt{line 24}) that expands the pseudo-regret budget adaptively, scaling with the realized noise rather than the worst-case horizon. Consequently, the algorithm automatically interpolates between adversarial robustness and stochastic instance-optimality without requiring prior knowledge.
\subsection{Main results}
\label{sec:main_results}

With the structural assumptions in place, we present the theoretical analysis of \textsc{Compass-Hedge}. Our main contribution is formalized in \cref{thm:main}, which establishes \textsc{Compass-Hedge} as a robust, ``Best-of-Three-Worlds'' algorithm: it guarantees baseline safety while simultaneously achieving minimax-optimality in adversarial regimes and instance-optimality in stochastic ones.
\newpage
\begin{theorem}[Universal Regret Guarantees]
\label{thm:main}
Let $\mu^c \in \Delta_A$ be an arbitrary comparator policy. For any horizon $T$, \textsc{Compass-Hedge} guarantees:

\begin{itemize}[leftmargin=*, itemsep=0pt, topsep=2pt]
    \item \textbf{Safety.}
    Regardless of the environment, the comparator regret is logarithmic:
    \(
    \mathcal R(\mu^c) \le O(\log^2 T).
    \)    
    \item \textbf{Adversarial Optimality.}
    Under Assumption \hyperref[A1]{\textcolor{softblue}{A1}} (Controlled-adaptivity):
    \[
    \mathcal{R}_{\mathrm{pseudo}}(T) \le O\!\left(\sqrt{T \log A}\,\log^2 T\right) = \tilde{O}(\sqrt{T}).
    \] \vspace{-2em}   
    \item \textbf{Stochastic Optimality.}
    Under Assumptions \hyperref[Stoc1]{\textcolor{softblue}{S1}}--\hyperref[Stoc2]{\textcolor{softblue}{S2}} (Gap $\Delta > 0$):
    \[
    \mathcal{R}_{\mathrm{pseudo}}(T) \le O\!\left(\frac{\log A\,\log^2 T}{\Delta}\right) = \tilde{O}\left(\frac{1}{\Delta}\right).
    \]
\end{itemize}
\end{theorem}
\begin{remark}[The ``Best-of-Three-Worlds'' Architecture]
\label{rem:interpretation}
\cref{thm:main} bridges the long-standing gap between conservative safety and aggressive adaptivity. We highlight three key properties:

\begin{enumerate}[label=(\textbf{\roman*}), leftmargin=*, itemsep=-2pt]
    \item \textbf{Hard Safety Constraint.} 
    Unlike standard BoBW algorithms that may suffer $\Omega(\sqrt{T})$ regret if the environment is adversarial, our algorithm treats $\mu^c$ as a \textit{performance floor}. The $O(\log^2 T)$ bound ensures the learner never deviates significantly from the safe baseline unless it is winning.
    
    \item \textbf{Blind Adaptation.} 
    Crucially, these rates are achieved \textit{without} regime detection or prior knowledge of the gap $\Delta$. The algorithm naturally accelerates from $\tilde{O}(\sqrt{T})$ to $\tilde{O}(1/\Delta)$ solely based on the observed loss curvature.
    
    \item \textbf{The Price of Protection.} 
    The additional poly-logarithmic factors ($\log^2 T$) represent the overhead of the doubling trick and the phase-based safety checks. This is the structural cost of maintaining the $\mu^c$ guarantee while remaining open to aggressive exploration.
\end{enumerate}

\end{remark}

\subsection{Analysis and proof sketch}
\label{sec:analysis}

The proof of \cref{thm:main} relies on a modular decomposition of regret across \emph{stages} (doubling epochs) and \emph{phases} (safety checks). We provide a high-level roadmap of the four key steps. All detailed proofs are deferred to the supplement.

\paragraph{Step 1: Bridging Realized and Expected Regret.}
Our analysis begins by observing that the algorithm's control logic (\textbf{Lines 7 and 10}) relies on \emph{data-dependent} quantities: the realized regret of the mixture against the Hedge experts ($R_{\text{hedge}}$) and the comparator ($R_{\text{comp}}$). Since the environment or the learner's decisions may be randomized, we must control the deviation between these observed quantities and their expected counterparts.
\cref{lem:pseudo-gap} bridges this gap via standard martingale concentration arguments, ensuring that the stopping conditions triggered by realized losses validly translate to guarantees on the pseudo-regret.

\paragraph{Step 2: The Doubling Mechanism (Regret Control).}
With the statistical validity established, we focus on the algorithmic structure. We ensure that our proxy for total regret does not grow uncontrollably. The stage-based doubling trick acts as a coarse regularizer.
\begin{lemma}[Estimated Regret Bounds]
\label{lem:regret-scale}
For every stage $s$, the pseudo-regret capacity satisfies $\widehat R_s < \widehat R_{s+1} < 2 R_{\text{hedge}}$.
\end{lemma}
\textit{Implication:} This ensures that the number of stages remains logarithmic ($S \approx O(\log T)$), allowing us to sum per-stage regrets efficiently without incurring a linear penalty.

\paragraph{Step 3: The Safety Ratchet (Phase Dynamics).}
This step constitutes the core safety mechanism. Within a phase $k$ of any stage, the algorithm fixes the mixing weight $\alpha_{[k]}$. The following lemmas govern the behavior inside a phase:

\begin{lemma}[Regret within a Normal Phase]
\label{lem:normal-phase}
Consider a phase $k$ with mixing parameter $\alpha_{[k]} < 1$. The regret against the safety comparator $\mu^c$ is bounded by $\mathcal R^{k}(\mu^c) \le 2^{k-1}$.
\end{lemma}

\begin{lemma}[Comparator Gain upon Phase Exit]
\label{lem:phase-exit}
If the algorithm chooses to exit phase $k$ (triggering a stronger Hedge mix), it implies that the mixture has significantly outperformed the comparator. Specifically:$\mathcal R^k(\mu^c) \le -\,2^{k-1} + 2.$
\end{lemma}
\noindent\textit{Implication:} These two lemmas form a ``ratche'': we only escalate risk (increase $\alpha$) when we have ``banked'' sufficient negative regret (surplus) against the comparator. 
Crucially, by construction, the total number of phases cannot exceed $O(\log R_{\text{hedge}})$. This guarantees that the regret remains logarithmic even if the algorithm halts abruptly due to an unknown horizon.

\paragraph{Step 4: Aggregating to Stage Optimality.}
Finally, by aggregating the phases within a stage $s$, we derive bounds that hold regardless of the stopping time.

\begin{lemma}[Per-Stage Regret]
\label{lem:stage-regret}
For any fixed stage $s$, we have $\mathcal R^s(\mu^c) = O(\log T)$. Furthermore, the regret against the best expert scales as $O (\sqrt{T \log A} \log T)$ (Adv.) or $O (\frac{\log T \log A}{\Delta})$ (Stoch.).
\end{lemma}
Since the total regret against $\mu^c$ is the sum of $O(1)$ over at most $O(\log T)$ stages, we obtain the global logarithmic safety bound declared in \cref{thm:main}.
\section{Experimental Evaluation \& Future Work}
\label{sec:experiments}
\begin{figure}[h]
    \centering
    \vspace{-1.25em}
        \begin{subfigure}[t]{0.48\textwidth}
        \centering       \includegraphics[width=0.85\linewidth]{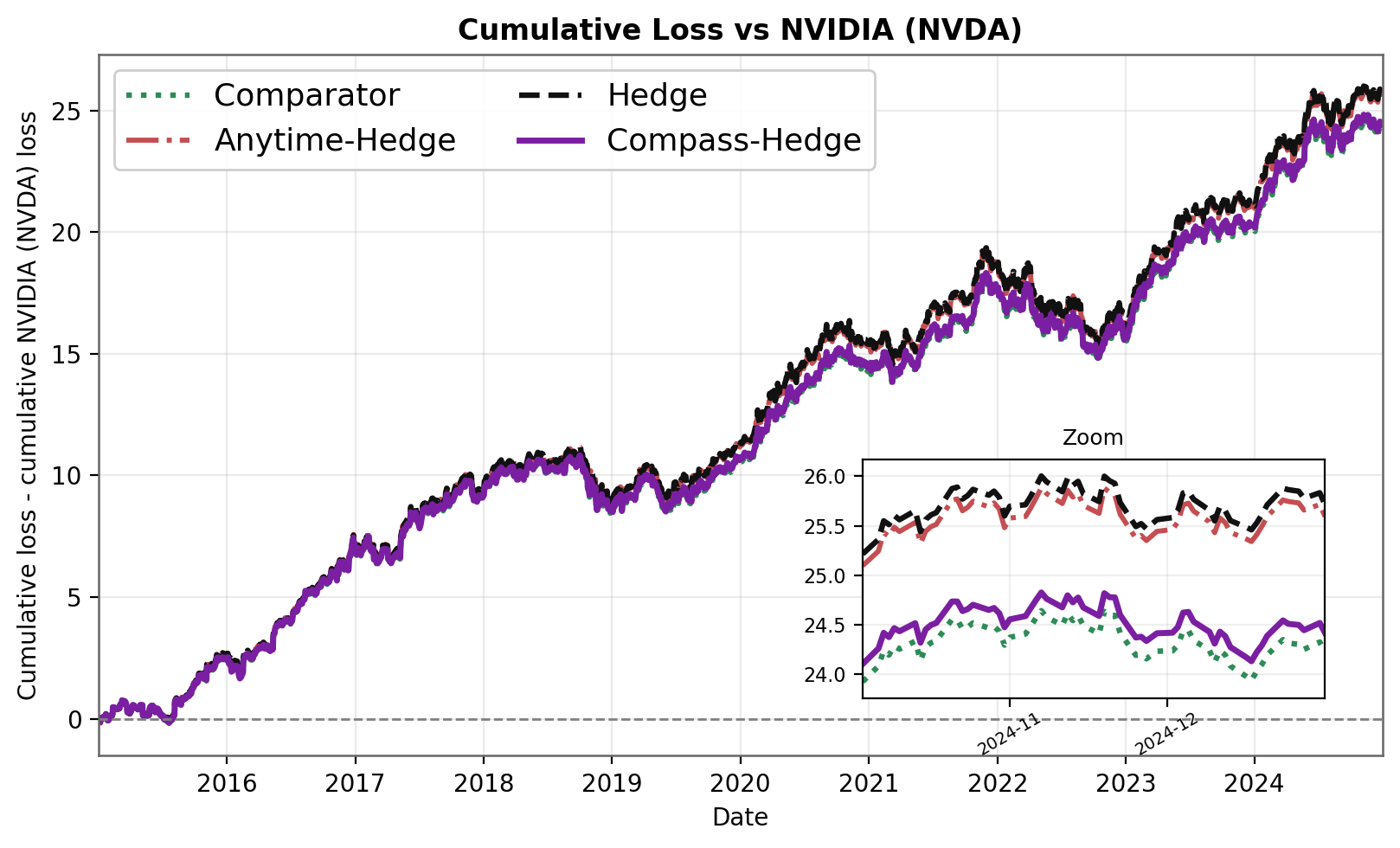}
        \caption{NVDA comparison.}
        \label{fig:stochastic-best}
    \end{subfigure}
    \hfill
    \begin{subfigure}[t]{0.48\textwidth}
        \centering
        \includegraphics[width =0.85\textwidth]{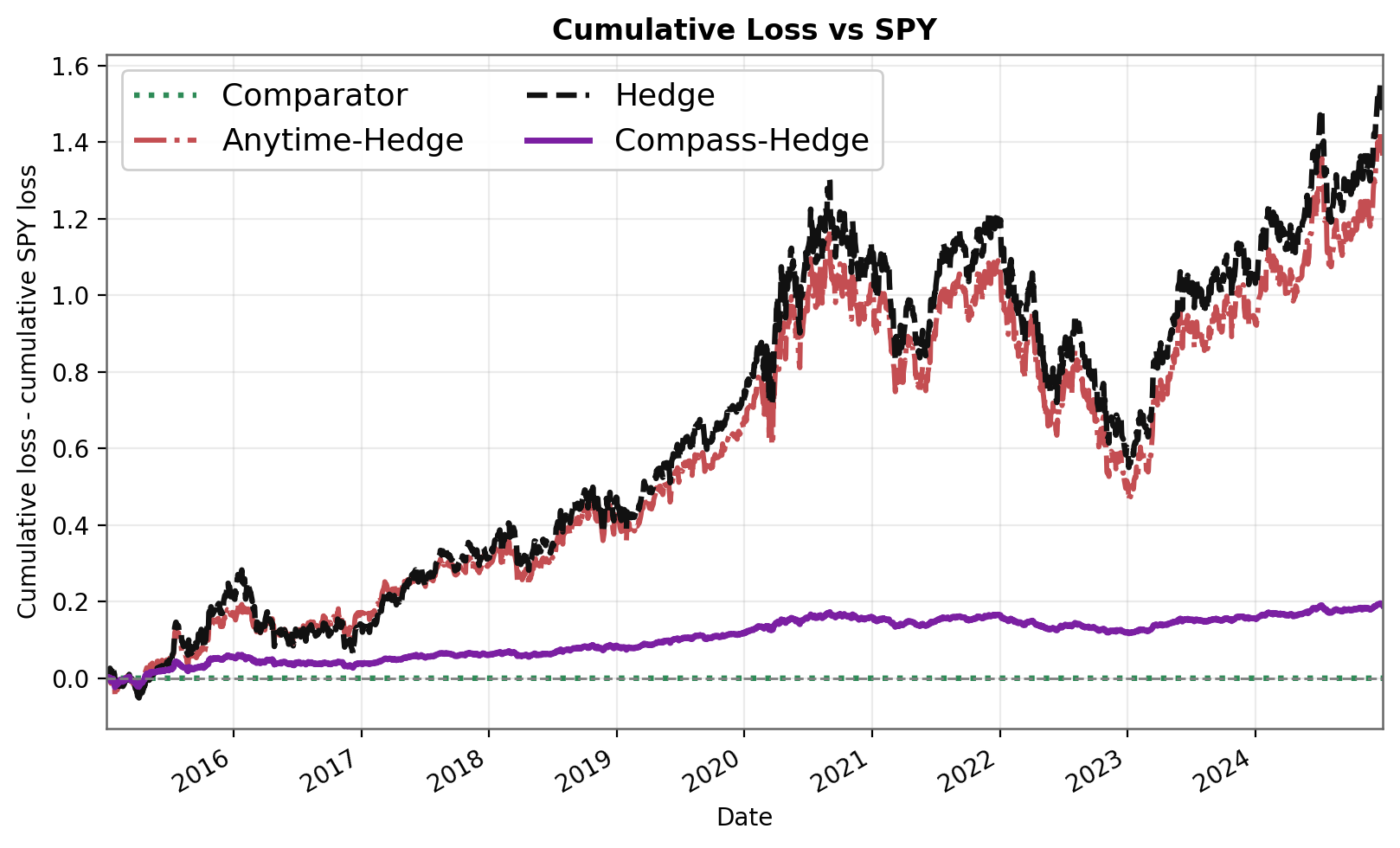}
        \caption{S\&P 500 comparison}
        \label{fig:stochastic-safe}
    \end{subfigure}
    \begin{subfigure}[t]{0.48\textwidth}
        \centering
        \includegraphics[width=0.85\linewidth]{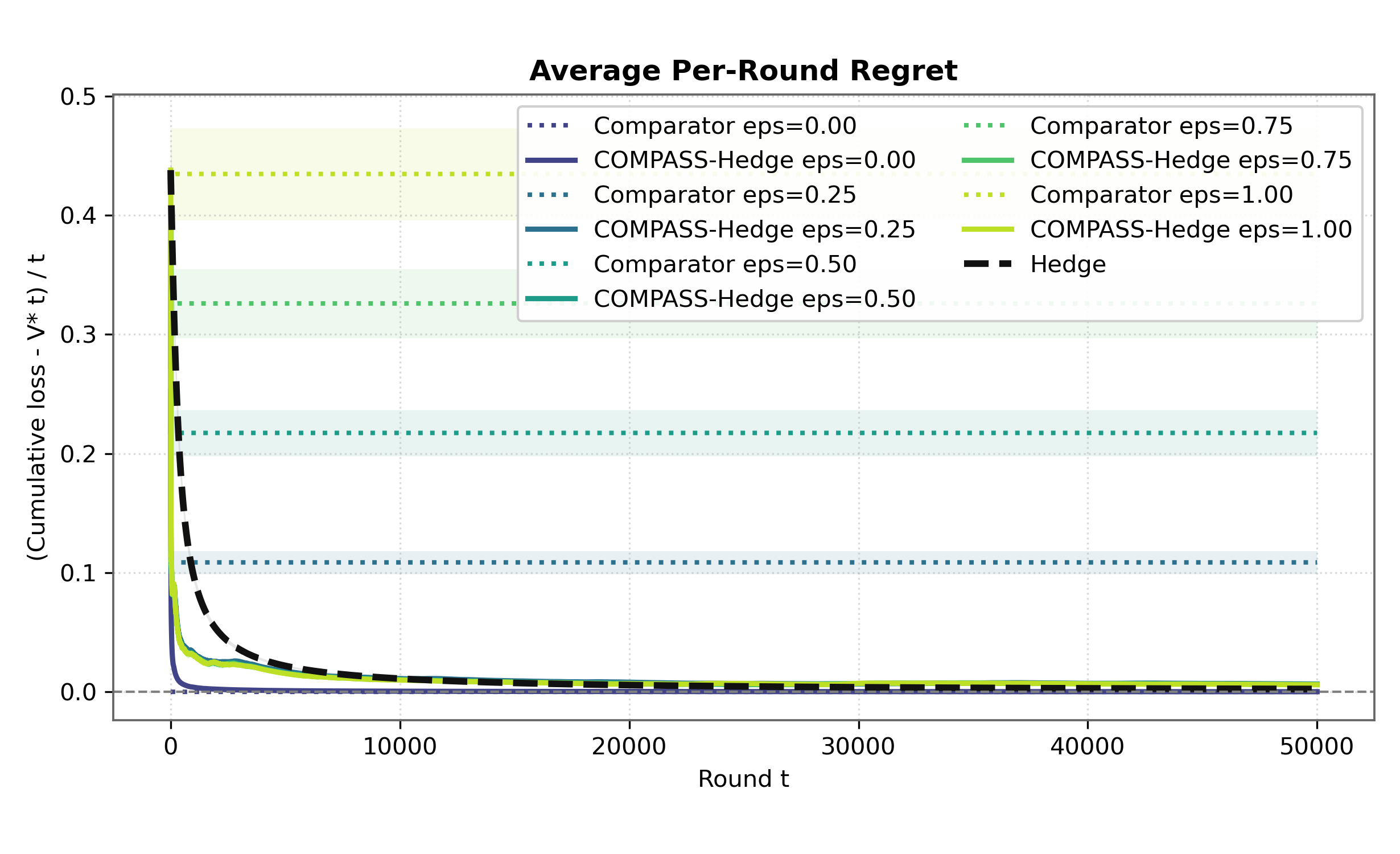}
        \caption{Average regret.}
        \label{fig:adv-avg}
    \end{subfigure}
    \hfill
    \begin{subfigure}[t]{0.48\textwidth}
        \centering
        \includegraphics[width=0.85\linewidth]{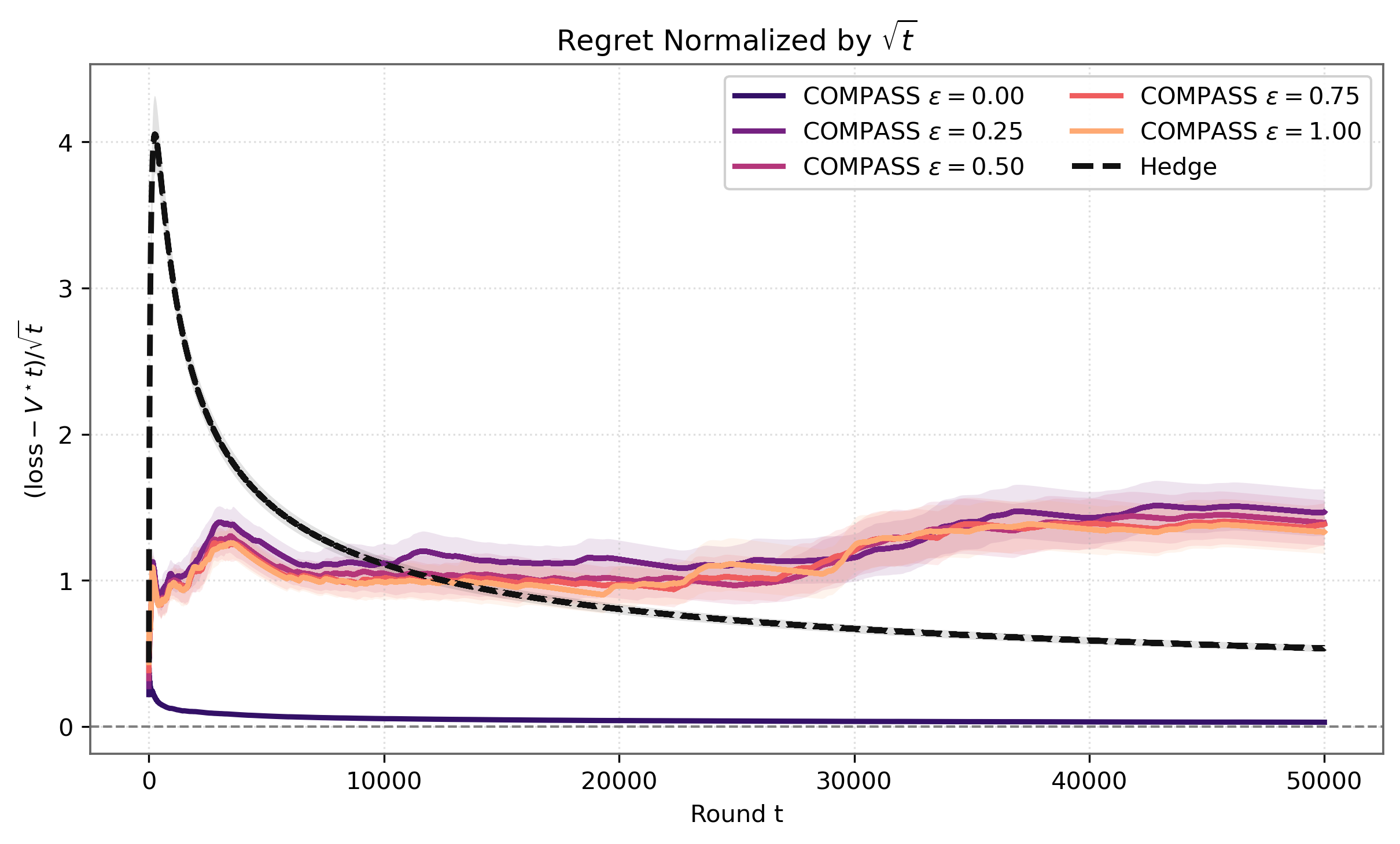}
        \caption{Regret normalized by $\sqrt{t}$.}
        \label{fig:adv-sqrt}
    \end{subfigure}

    \vspace{-0.5em}
    \caption{\textbf{Experimental evaluation.}
       \textbf{(a)--(b) Stochastic adaptivity on S\&P 500 data (2015--2024):}
    Panel~(a) reports regret against NVIDIA, the ex-post best single asset,
    while Panel~(b) reports regret against SPY, a market comparator.
    \textsc{Compass-Hedge} gives the best trade-off: it improves over Hedge variants against NVDA while remaining close to SPY.
    \textbf{(c)--(d) Adversarial robustness:}
    \textsc{Compass-Hedge} exhibits vanishing average regret and stable
    $R_t/\sqrt{t}$, consistent with $O(\sqrt{T})$ adversarial regret.
    Shaded regions denote  90\% CIs.
 }
    \label{fig:experiments}
\end{figure}
\vspace{-0.75em}

Finally, we validate our theoretical guarantees by evaluating whether
\textsc{Compass-Hedge} realizes the intended universality: adversarial robustness
without sacrificing adaptivity in benign regimes. We test two complementary
settings: a control adversarial zero-sum game (\cite{su2026exploitability,assos2024maximizing}), where a
strategic optimizer alternates between regimes to stress Hedge-type dynamics and daily S\&P 500 data from 2015--2024, a
real-data environment with both persistent stochastic structure and abrupt regime
shifts.
A detailed breakdown of the experimental setup and results can be found in App.~\ref{app:experiments} and ~\ref{sec:experiments}.
\paragraph{Adversarial robustness.}
We vary a parameter $\varepsilon$ controlling the quality of the baseline
comparator, with smaller $\varepsilon$ corresponding to a comparator closer to
the learner's optimal strategy. Panel~(c) shows that, across all comparator
qualities, the average regret of \textsc{Compass-Hedge} decays to zero, while
blindly following the baseline can incur linear regret. As $\varepsilon \to 0$,
the regret decreases faster, and in well-aligned cases \textsc{Compass-Hedge}
can outperform standard Hedge. Panel~(d) shows that $R_t/\sqrt{t}$ remains
approximately flat, consistent with $O(\sqrt{T})$ adversarial regret.
\vspace{-1em}
\paragraph{Stochastic adaptivity on market data.}
We next evaluate the algorithms on daily S\&P 500 data from 2015--2024, treating stocks as arms and converting realized returns to bounded losses, details are left in Appendix~\ref{app:data_setup}. We compare against two benchmarks:
NVIDIA (NVDA), the ex-post best single asset, and SPY, a diversified market
comparator. NVDA measures adaptivity to favorable structure, while SPY measures
baseline safety. Relative to NVDA, Hedge and
Anytime-Hedge accumulate large regret, whereas \textsc{Compass-Hedge} closely
tracks the best-arm benchmark. Relative to SPY, the gap is sharper: Hedge and
Anytime-Hedge incur substantial regret, while \textsc{Compass-Hedge} remains
near the market comparator throughout. Thus, \textsc{Compass-Hedge} avoids the
usual tradeoff between exploiting the best realized asset and preserving safety
against a stable baseline.

\paragraph{Future Work.}
We introduced \textsc{Compass-Hedge}, a single parameter-free algorithm unifying
\emph{baseline safety}, \emph{worst-case robustness}, and \emph{stochastic
adaptivity}. Its bounds are nearly optimal up to polylogarithmic factors, leaving
open whether these overheads are inherent or artifacts of our restart method.
We discuss this in Appendix~\ref{app:price-of-protection}; Appendix~\ref{app:adaptive-adversaries}
explains the extension to stronger adaptive adversaries. Extending
comparator-safe adaptivity beyond experts---to linear bandits, structured action
spaces, and reinforcement learning---is a promising direction for safer
sequential decision-making.

\newpage
%



\bibliography{example_paper}

@inproceedings{gaillard2014second,
  title={A second-order bound with excess losses},
  author={Gaillard, Pierre and Stoltz, Gilles and Van Erven, Tim},
  booktitle={Conference on Learning Theory},
  pages={176--196},
  year={2014},
  organization={PMLR}
}

@article{besson2018doubling,
  title={What doubling tricks can and can't do for multi-armed bandits},
  author={Besson, Lilian and Kaufmann, Emilie},
  journal={arXiv preprint arXiv:1803.06971},
  year={2018}
}

@article{fang2022online,
  title={Online mirror descent and dual averaging: keeping pace in the dynamic case},
  author={Fang, Huang and Harvey, Nicholas JA and Portella, Victor S and Friedlander, Michael P},
  journal={Journal of Machine Learning Research},
  volume={23},
  number={121},
  pages={1--38},
  year={2022}
}

@article{zimmert2024productive,
  title={Productive bandits: Importance weighting no more},
  author={Zimmert, Julian and Marinov, Teodor V},
  journal={Advances in Neural Information Processing Systems},
  volume={37},
  pages={85360--85388},
  year={2024}
}

@article{cesa2007improved,
  title={Improved second-order bounds for prediction with expert advice},
  author={Cesa-Bianchi, Nicolo and Mansour, Yishay and Stoltz, Gilles},
  journal={Machine Learning},
  volume={66},
  number={2},
  pages={321--352},
  year={2007},
  publisher={Springer}
}

@article{su2026exploitability,
  title={On the Exploitability of FTRL Dynamics},
  author={Su, Yiheng and Vlatakis-Gkaragkounis, Emmanouil-Vasileios},
  journal={arXiv e-prints},
  pages={arXiv--2604},
  year={2026}
}

@inproceedings{joulani2013online,
  title={Online learning under delayed feedback},
  author={Joulani, Pooria and Gyorgy, Andras and Szepesv{\'a}ri, Csaba},
  booktitle={International conference on machine learning},
  pages={1453--1461},
  year={2013},
  organization={PMLR}
}

@inproceedings{auer1995gambling,
  title={Gambling in a rigged casino: The adversarial multi-armed bandit problem},
  author={Auer, Peter and Cesa-Bianchi, Nicolo and Freund, Yoav and Schapire, Robert E},
  booktitle={Proceedings of IEEE 36th annual foundations of computer science},
  pages={322--331},
  year={1995},
  organization={IEEE}
}

@article{auer2010ucb,
  title={UCB revisited: Improved regret bounds for the stochastic multi-armed bandit problem},
  author={Auer, Peter and Ortner, Ronald},
  journal={Periodica Mathematica Hungarica},
  volume={61},
  number={1-2},
  pages={55--65},
  year={2010},
  publisher={Akad{\'e}miai Kiad{\'o}, co-published with Springer Science+ Business Media BV~…}
}

@article{arora2012online,
  title={Online bandit learning against an adaptive adversary: from regret to policy regret},
  author={Arora, Raman and Dekel, Ofer and Tewari, Ambuj},
  journal={arXiv preprint arXiv:1206.6400},
  year={2012}
}

@article{zimmert2021tsallis,
  title={Tsallis-inf: An optimal algorithm for stochastic and adversarial bandits},
  author={Zimmert, Julian and Seldin, Yevgeny},
  journal={Journal of Machine Learning Research},
  volume={22},
  number={28},
  pages={1--49},
  year={2021}
}

@article{haghtalab2024smoothed,
  title={Smoothed analysis with adaptive adversaries},
  author={Haghtalab, Nika and Roughgarden, Tim and Shetty, Abhishek},
  journal={Journal of the ACM},
  volume={71},
  number={3},
  pages={1--34},
  year={2024},
  publisher={ACM New York, NY}
}

@book{borodin-1998,
  title={Online computation and competitive analysis},
  author={Borodin, Allan and El-Yaniv, Ran},
  year={2005},
  publisher={cambridge university press}
}

@article{shalev2012online,
  title={Online Learning and Online Convex Optimization},
  author={Shalev-Shwartz, Shai},
  journal={Foundations and Trends{\textregistered} in Machine Learning},
  volume={4},
  number={2},
  pages={107--194},
  year={2012},
  publisher={Now Publishers, Inc.}
}

@article{hazan2016introduction,
  title={Introduction to Online Convex Optimization},
  author={Hazan, Elad},
  journal={Foundations and Trends{\textregistered} in Optimization},
  volume={2},
  number={3-4},
  pages={157--325},
  year={2016},
  publisher={Now Publishers, Inc.}
}

@article{tang2016online,
  title={Online charging scheduling algorithms of electric vehicles in smart grid: An overview},
  author={Tang, Wanrong and Bi, Suzhi and Zhang, Ying Jun},
  journal={IEEE communications Magazine},
  volume={54},
  number={12},
  pages={76--83},
  year={2016},
  publisher={IEEE}
}

@inproceedings{sani2014exploiting,
  title={Exploiting easy data in online optimization},
  author={Sani, Amir and Neu, Gergely and Lazaric, Alessandro},
  booktitle={Advances in Neural Information Processing Systems (NeurIPS)},
  volume={27},
  year={2014}
}

@article{cover1991universal,
  title={Universal portfolios},
  author={Cover, Thomas M},
  journal={Mathematical finance},
  volume={1},
  number={1},
  pages={1--29},
  year={1991},
  publisher={Wiley Online Library}
}

@article{villar2015multi,
  title={Multi-armed bandit models for the optimal design of clinical trials: benefits and challenges},
  author={Villar, Sof{\'\i}a S and Bowden, Jack and Wason, James},
  journal={Statistical science: a review journal of the Institute of Mathematical Statistics},
  volume={30},
  number={2},
  pages={199},
  year={2015}
}

@article{LittlestoneWarmuth1994,
  title={The weighted majority algorithm},
  author={Littlestone, Nick and Warmuth, Manfred K},
  journal={Information and computation},
  volume={108},
  number={2},
  pages={212--261},
  year={1994},
  publisher={Elsevier}
}

@article{FreundSchapire1997,
  title={A decision-theoretic generalization of on-line learning and an application to boosting},
  author={Freund, Yoav and Schapire, Robert E},
  journal={Journal of computer and system sciences},
  volume={55},
  number={1},
  pages={119--139},
  year={1997},
  publisher={Elsevier}
}

@book{CesaBianchiLugosi2006,
  title={Prediction, learning, and games},
  author={Cesa-Bianchi, Nicolo and Lugosi, G{\'a}bor},
  year={2006},
  publisher={Cambridge university press}
}

@article{auer2002finite,
  title={Finite-time analysis of the multiarmed bandit problem},
  author={Auer, Peter and Cesa-Bianchi, Nicolo and Fischer, Paul},
  journal={Machine learning},
  volume={47},
  number={2},
  pages={235--256},
  year={2002},
  publisher={Springer}
}

@article{auer-2002,
  title={Adaptive and self-confident on-line learning algorithms},
  author={Auer, Peter and Cesa-Bianchi, Nicolo and Gentile, Claudio},
  journal={Journal of Computer and System Sciences},
  volume={64},
  number={1},
  pages={48--75},
  year={2002},
  publisher={Elsevier}
}

@article{deRooijEtAl2014,
  title={Follow the leader if you can, hedge if you must},
  author={De Rooij, Steven and Van Erven, Tim and Gr{\"u}nwald, Peter D and Koolen, Wouter M},
  journal={The Journal of Machine Learning Research},
  volume={15},
  number={1},
  pages={1281--1316},
  year={2014},
  publisher={JMLR. org}
}

@inproceedings{koolen2015second,
  title={Second-order quantile methods for experts and combinatorial games},
  author={Koolen, Wouter M and Van Erven, Tim},
  booktitle={Conference on Learning Theory},
  pages={1155--1175},
  year={2015},
  organization={PMLR}
}

@article{even2008regret,
  title={Regret to the best vs. regret to the average},
  author={Even-Dar, Eyal and Kearns, Michael and Mansour, Yishay and Wortman, Jennifer},
  journal={Machine Learning},
  volume={72},
  number={1},
  pages={21--37},
  year={2008},
  publisher={Springer}
}

@inproceedings{gupta2024ai,
  title={AI in Financial Decision-Making: Revolutionizing investment strategies and risk management},
  author={Gupta, Aaryan and Puri, Mayank and Keshan, Mayank and Tiwari, Varun},
  booktitle={This paper was presented at Global Forum for Financial Consumers (GFFC)},
  year={2024}
}

@article{li2025efficient,
  title={Efficient Best-of-Both-Worlds Algorithms for Contextual Combinatorial Semi-Bandits},
  author={Li, Mengmeng and Schneider, Philipp and Aleksi{\'c}, Jelisaveta and Kuhn, Daniel},
  journal={arXiv preprint arXiv:2508.18768},
  year={2025}
}

@article{orabona2016coin,
  title={Coin betting and parameter-free online learning},
  author={Orabona, Francesco and P{\'a}l, D{\'a}vid},
  journal={Advances in Neural Information Processing Systems},
  volume={29},
  year={2016}
}

@article{MullerEtAl2025,
  title={Best of Both Worlds: Regret Minimization versus Minimax Play},
  author={M{\"u}ller, Adrian and Schneider, Jon and Skoulakis, Stratis and Viano, Luca and Cevher, Volkan},
  journal={arXiv preprint arXiv:2502.11673},
  year={2025}
}

@article{mourtada2019optimality,
  title={On the optimality of the Hedge algorithm in the stochastic regime},
  author={Mourtada, Jaouad and Ga{\"\i}ffas, St{\'e}phane},
  journal={Journal of Machine Learning Research},
  volume={20},
  number={83},
  pages={1--28},
  year={2019}
}

@inproceedings{bubeck2012best,
  title={The best of both worlds: Stochastic and adversarial bandits},
  author={Bubeck, S{\'e}bastien and Slivkins, Aleksandrs},
  booktitle={Conference on Learning Theory},
  pages={42--1},
  year={2012},
  organization={JMLR Workshop and Conference Proceedings}
}

@article{ganzfried2015safe,
  title={Safe opponent exploitation},
  author={Ganzfried, Sam and Sandholm, Tuomas},
  journal={ACM Transactions on Economics and Computation (TEAC)},
  volume={3},
  number={2},
  pages={1--28},
  year={2015},
  publisher={ACM New York, NY, USA}
}

@inproceedings{damer2017safely,
  title={Safely using predictions in general-sum normal form games},
  author={Damer, Steven and Gini, Maria},
  booktitle={Proceedings of the 16th conference on autonomous agents and multiagent systems},
  pages={924--932},
  year={2017}
}

@article{liu2022safe,
  title={Safe opponent-exploitation subgame refinement},
  author={Liu, Mingyang and Wu, Chengjie and Liu, Qihan and Jing, Yansen and Yang, Jun and Tang, Pingzhong and Zhang, Chongjie},
  journal={Advances in Neural Information Processing Systems},
  volume={35},
  pages={27610--27622},
  year={2022}
}

@article{hutter2005adaptive,
  title={Adaptive online prediction by following the perturbed leader},
  author={Hutter, Marcus and Poland, Jan},
  year={2005},
  publisher={Journal of Machine Learning Research}
}

@article{kapralov2011prediction,
  title={Prediction strategies without loss},
  author={Kapralov, Michael and Panigrahy, Rina},
  journal={Advances in Neural Information Processing Systems},
  volume={24},
  year={2011}
}

@article{koolen2013pareto,
  title={The pareto regret frontier},
  author={Koolen, Wouter M},
  journal={Advances in Neural Information Processing Systems},
  volume={26},
  year={2013}
}

@inproceedings{cutkosky2018black,
  title={Black-box reductions for parameter-free online learning in banach spaces},
  author={Cutkosky, Ashok and Orabona, Francesco},
  booktitle={Conference On Learning Theory},
  pages={1493--1529},
  year={2018},
  organization={PMLR}
}

@article{orabona2019modern,
  title={A modern introduction to online learning},
  author={Orabona, Francesco},
  journal={arXiv preprint arXiv:1912.13213},
  year={2019}
}

@inproceedings{wu2016conservative,
  title={Conservative bandits},
  author={Wu, Yifan and Shariff, Roshan and Lattimore, Tor and Szepesv{\'a}ri, Csaba},
  booktitle={International Conference on Machine Learning},
  pages={1254--1262},
  year={2016},
  organization={PMLR}
}

@inproceedings{garcelon2020conservative,
  title={Conservative exploration in reinforcement learning},
  author={Garcelon, Evrard and Ghavamzadeh, Mohammad and Lazaric, Alessandro and Pirotta, Matteo},
  booktitle={International conference on artificial intelligence and statistics},
  pages={1431--1441},
  year={2020},
  organization={PMLR}
}

@article{badanidiyuru2018bandits,
  title={Bandits with knapsacks},
  author={Badanidiyuru, Ashwinkumar and Kleinberg, Robert and Slivkins, Aleksandrs},
  journal={Journal of the ACM (JACM)},
  volume={65},
  number={3},
  pages={1--55},
  year={2018},
  publisher={ACM New York, NY, USA}
}

@article{efroni2020exploration,
  title={Exploration-exploitation in constrained mdps},
  author={Efroni, Yonathan and Mannor, Shie and Pirotta, Matteo},
  journal={arXiv preprint arXiv:2003.02189},
  year={2020}
}

@article{liu2021learning,
  title={Learning policies with zero or bounded constraint violation for constrained mdps},
  author={Liu, Tao and Zhou, Ruida and Kalathil, Dileep and Kumar, Panganamala and Tian, Chao},
  journal={Advances in Neural Information Processing Systems},
  volume={34},
  pages={17183--17193},
  year={2021}
}

@article{lattimore2015pareto,
  title={The pareto regret frontier for bandits},
  author={Lattimore, Tor},
  journal={Advances in Neural Information Processing Systems},
  volume={28},
  year={2015}
}

@article{van2020comparator,
  title={Comparator-adaptive convex bandits},
  author={van der Hoeven, Dirk and Cutkosky, Ashok and Luo, Haipeng},
  journal={Advances in Neural Information Processing Systems},
  volume={33},
  pages={19795--19804},
  year={2020}
}

@article{assos2024maximizing,
  title={Maximizing utility in multi-agent environments by anticipating the behavior of other learners},
  author={Assos, Angelos and Dagan, Yuval and Daskalakis, Constantinos},
  journal={Advances in Neural Information Processing Systems},
  volume={37},
  pages={38769--38798},
  year={2024}
}

@inproceedings{braverman2018selling,
  title={Selling to a no-regret buyer},
  author={Braverman, Mark and Mao, Jieming and Schneider, Jon and Weinberg, Matt},
  booktitle={Proceedings of the 2018 ACM Conference on Economics and Computation},
  pages={523--538},
  year={2018}
}

@article{deng2019strategizing,
  title={Strategizing against no-regret learners},
  author={Deng, Yuan and Schneider, Jon and Sivan, Balasubramanian},
  journal={Advances in neural information processing systems},
  volume={32},
  year={2019}
}

@article{deng2019prior,
  title={Prior-free dynamic auctions with low regret buyers},
  author={Deng, Yuan and Schneider, Jon and Sivan, Balasubramanian},
  journal={Advances in Neural Information Processing Systems},
  volume={32},
  year={2019}
}

@article{kolumbus2024contracting,
  title={Contracting with a learning agent},
  author={Kolumbus, Yoav and Schneider, Jon and Talgam-Cohen, Inbal and Vlatakis-Gkaragkounis, Emmanouil-Vasileios and Wang, Joshua R and Weinberg, SM},
  journal={Advances in Neural Information Processing Systems},
  volume={37},
  pages={77366--77408},
  year={2024}
}

@inproceedings{mansour2022strategizing,
  title={Strategizing against learners in bayesian games},
  author={Mansour, Yishay and Mohri, Mehryar and Schneider, Jon and Sivan, Balasubramanian},
  booktitle={Conference on Learning Theory},
  pages={5221--5252},
  year={2022},
  organization={PMLR}
}
\bibliographystyle{plainnat}
\clearpage
\tableofcontents
\appendix
\clearpage
\begin{figure}[t]
\centering
\begin{tikzpicture}[
    every node/.style={font=\small},
    lbl/.style={font=\scriptsize\itshape, text=slate, align=center},scale=1.4
]

\begin{scope}[opacity=0.10]
  \fill[brickred]    (90:1.55)  circle (2.0cm);
  \fill[royalblue]   (210:1.55) circle (2.0cm);
  \fill[forestgreen]  (330:1.55) circle (2.0cm);
\end{scope}
\draw[brickred,   thick] (90:1.55)  circle (2.0cm);
\draw[royalblue,  thick] (210:1.55) circle (2.0cm);
\draw[forestgreen, thick] (330:1.55) circle (2.0cm);

\node[font=\normalsize\bfseries, text=brickred, align=center]
  at (90:2.5) {Baseline\\[-2pt]Safety};
\node[font=\footnotesize, text=brickred!80!black] at (90:1.6) {$\mathcal{R}(\mu^c)\!=\!O(\log T)$};

\node[font=\normalsize\bfseries, text=royalblue, align=center]
  at (220:2.5) {Adversarial\\[-2pt]Robustness};
\node[font=\footnotesize, text=royalblue!80!black] at (218:1.65) {$\tilde{O}\!\bigl(\!\sqrt{T}\bigr)$};

\node[font=\normalsize\bfseries, text=forestgreen!85!black, align=center]
  at (320:2.55) {Stochastic\\[-2pt]Adaptivity};
\node[font=\footnotesize, text=forestgreen!70!black] at (322:1.65) {$\tilde{O}(1/\Delta)$};

\node[lbl] at (150:1.30) {\textnormal{Safe}\\[-2pt]\textnormal{Adv.}};
\node[lbl] at (30:1.30)  {\textnormal{Safe}\\[-2pt]\textnormal{Stoc.}};
\node[lbl] at (270:0.65) {\textnormal{BOBW}};

\node[fill=white, draw=black!60, semithick, rounded corners=3pt,
      inner xsep=6pt, inner ysep=4pt,
      blur shadow={shadow blur steps=5, shadow xshift=0.4pt, shadow yshift=-0.4pt}]
  (center) at (0, 0.15) {\footnotesize\bfseries\textsc{Compass-Hedge}};

\coordinate (extSA) at (160:3.85);
\coordinate (extSS) at (20:3.85);
\coordinate (extBW) at (270:3.25);

\draw[gray!50, thin] (150:1.55) -- (extSA);
\node[font=\tiny, text=slate, align=center, anchor=south] at (extSA)
  {\citet{even2008regret}\\[-1pt]Phased Aggression};

\draw[gray!50, thin] (30:1.55) -- (extSS);
\node[font=\tiny, text=slate, align=center, anchor=south] at (extSS)
  {\citet{MullerEtAl2025}\\[-1pt](Appendix)};

\draw[gray!50, thin] (270:0.95) -- (extBW);
\node[font=\tiny, text=slate, align=center, anchor=north] at (extBW)
  {Hedge / FTRL\\[-1pt]\citep{mourtada2019optimality}};

\end{tikzpicture}
\caption{\textbf{The Best-of-Three-Worlds landscape.} Each circle represents one desideratum. Prior algorithms (annotations) cover at most two regions simultaneously. \textsc{Compass-Hedge} (center) is the first to unify all three in a single, parameter-free algorithm.}
\label{fig:three_worlds}
\end{figure}
\section{Extended related work and motivating scenarios}
\label{app:related_work}

The literature on online learning is vast, spanning several decades of research in information theory, statistics, and computer science. We provide here a more detailed contextualization of \textsc{COMPASS-Hedge} within the broader landscape of sequential decision-making.

\subsection{Exploitation vs.\ safety trade-offs}
Safe exploitation methods typically deviate from minimax play to target weak opponents, often using parameters to balance safety and aggression~\cite{damer2017safely, liu2022safe}. This approach forces a compromise: algorithms either lack theoretical exploitation guarantees~\cite{ganzfried2015safe} or risk linear regret against strong adversaries~\cite{damer2017safely, liu2022safe}. \textsc{Compass-Hedge} avoids these trade-offs by defining safety as achieving $\tilde{\mathcal{O}}(1)$ regret against a baseline, ensuring robustness without sacrificing optimal adversarial rates.
\subsection{Comparator adaptivity in full information}
Safety, defined as ensuring low regret relative to a pre-specified baseline policy $\mu^c$, is a cornerstone of ``safe AI.'' In the full-information setting, it is established that algorithms can achieve constant regret against a specific comparator strategy while maintaining optimal worst-case regret guarantees. Notable examples include the Phased Aggression framework~\cite{even2008regret} and various adaptive strategies~\cite{hutter2005adaptive, kapralov2011prediction, koolen2013pareto, sani2014exploiting}. Similarly, parameter-free methods adapt to unknown comparators~\cite{orabona2016coin, cutkosky2018black, orabona2019modern}. These works successfully resolve the ``Best-of-Both-Worlds'' challenge regarding safety and adversarial robustness. However, they typically do not account for the stochastic nature of the environment, potentially missing the opportunity for gap-dependent logarithmic regret. Our work extends this line of inquiry by unifying these objectives.

\subsection{Hardness of safe bandit learning}
Under bandit feedback, achieving constant regret against a deterministic comparator implies linear worst-case regret against other actions~\cite{lattimore2015pareto}. This result indicates that the $\tilde{\mathcal{O}}(1)$ safety guarantees are impossible in the bandit setting without relaxing assumptions on the comparator, such as requiring it to be full-support~\cite{MullerEtAl2025}. Although recent advances in bandit convex optimization have explored comparator adaptivity~\cite{van2020comparator}, they generally cannot replicate the safety guarantee under full information. By operating in the full-information regime, \textsc{Compass-Hedge} avoids these hardness results.

\subsection{Conservative bandits and safe RL}
Our framework's definition of safety is distinct from the paradigms often studied in conservative bandits~\cite{wu2016conservative} and conservative reinforcement learning (RL)~\cite{garcelon2020conservative}. In those contexts, algorithms are designed to satisfy a constraint—often defined as securing at least a $(1-\alpha)$-fraction of a baseline's return. While this ensures performance relative to a baseline, it allows for linear regret $\mathcal{O}(\alpha T)$ relative to the optimal policy in the worst case. Furthermore, in constrained settings such as Constrained MDPs~\cite{badanidiyuru2018bandits, efroni2020exploration}, achieving constant regret against a safe baseline is frequently unobtainable under adversarial environemnt ~\cite{liu2021learning}. In contrast, \textsc{Compass-Hedge} aims for constant $\tilde{\mathcal{O}}(1)$ regret relative to the baseline while simultaneously minimizing regret against the best strategy in hindsight.

\subsection{Foundations of adversarial and adaptive learning}
The adversarial online learning framework, pioneered by \citet{LittlestoneWarmuth1994} and \citet{FreundSchapire1997}, focuses on minimizing regret against an arbitrary sequence of losses. The Hedge and Weighted Majority algorithms established the minimax-optimal $O(\sqrt{T \log A})$ rate, later generalized through the lens of Online Mirror Descent (OMD) and Follow-the-Regularized-Leader (FTRL) \citep{shalev2012online, hazan2016introduction}. 

Subsequent research sought to move beyond worst-case bounds. \emph{First-order} bounds \citep{auer-2002} scale with the loss of the best expert, while \emph{second-order} (variance-dependent) bounds—achieved by algorithms like AdaHedge \citep{deRooijEtAl2014} and Squint \citep{koolen2015second}—exploit cases where experts exhibit low variance or large suboptimality gaps. \citet{mourtada2019optimality} provided a landmark analysis showing that even vanilla Hedge with a $1/\sqrt{t}$ learning rate can be instance-optimal in stochastic settings while remaining robust, though it lacks explicit safety guarantees relative to a specific baseline.

\subsection{Adaptation via doubling tricks}
The strategy of parameter adaptation via geometric restarts, commonly referred to as the ``Doubling Trick,'' is a foundational technique in online learning used to convert horizon-dependent algorithms into anytime strategies. The concept can be traced back to \citet{auer1995gambling} in the context of adversarial bandits. In its most ubiquitous form, popularized by \citet{CesaBianchiLugosi2006} for prediction with expert advice, the trick involves partitioning the time horizon into epochs of length $T_i = 2^i$. At the start of each epoch, the base algorithm is restarted with the new horizon guess. This approach has been widely adapted across the bandit literature; for instance, \citet{auer2010ucb} utilized a sparser doubling schedule with $T_i = 2^{2^i}$ to render the UCB-R algorithm anytime for stochastic environments. In the specific setting of delayed feedback, \citet{joulani2013online} employed this mechanism to adapt to the unknown total delay $D$ by restarting the algorithm whenever the accumulated delay exceeded the current estimate. 

However, while effective for minimax regret, the doubling trick is not without caveats. As analyzed by \citet{besson2018doubling}, the ``oblivious'' nature of the trick—where the algorithm's internal state is reset at every epoch—can be detrimental when preserving finer, instance-dependent guarantees. Specifically, in our safety-constrained setting, a standard doubling trick would periodically reset the negative regret accumulated against the comparator. This necessitates a more intricate, coupled schedule (as employed in \textsc{Compass}-Hedge) that adapts without discarding the statistical evidence required to maintain the $O(1)$ safety bound.

\subsection{The best-of-both-worlds paradigm}
The Best-of-Both-Worlds (BOBW) literature seeks algorithms that perform optimally in both stochastic and adversarial environments without prior knowledge of the regime. While earlier work focused on full-information feedback, the field has recently prioritized the multi-armed bandit (MAB) setting \citep{bubeck2012best, zimmert2021tsallis} due to the challenge of exploration under uncertainty. However, the full-information setting remains technically distinct and practically vital. Modern BOBW methods often rely on adaptive regularization (e.g., Tsallis entropy) to interpolate between the two regimes \citep{zimmert2021tsallis}. \textsc{COMPASS-Hedge} contributes to this line by showing that the "third world" of safety can be integrated into this duality through an empirical doubling-trick on the aggression schedule.

\subsection{Manipulation of learning agents}
While our work focuses on designing learners that are robust to adversarial environments, a growing body of literature investigates the inverse problem: how strategic optimizers can exploit the predictable dynamics of learning agents~\cite{braverman2018selling, deng2019strategizing}. This line of work demonstrates that non-myopic opponents can steer standard mean-based learners (such as MWU) toward outcomes that significantly favor the manipulator. It is characterized across repeated auctions~\cite{braverman2018selling, deng2019prior}, contract design~\cite{kolumbus2024contracting}, and Bayesian games~\cite{mansour2022strategizing}. Most recently, Assos et al.~\cite{assos2024maximizing} demonstrated that while optimal manipulation strategies can be efficiently generated for zero-sum settings , the same task becomes computationally intractable when applied to general-sum game.
\subsection{Detailed motivating scenarios for full-information safety}
We elaborate on why the full-information feedback model is the appropriate abstraction for many safety-critical applications, where counterfactual losses are observable post-hoc:

\begin{itemize}[leftmargin=1.5em, itemsep=4pt]
    \item \textbf{Clinical Decision Support and Titration:} In personalized medicine, clinicians use algorithms to suggest medication dosages. Once a patient is treated, their physiological markers (e.g., blood pressure, glucose levels) are recorded. Biomechanical models or historical patient data allow clinicians to estimate what the response \emph{would have been} for alternative dosages within a local range (the counterfactuals). The algorithm must outperform the standard-of-care ($\mu^c$) if the patient responds predictably (stochastic), but must never exceed a risk threshold if the patient's reaction is anomalous (adversarial) \citep{villar2015multi}.
    
    \item \textbf{Financial Hedging and Index Tracking:} In "Active Indexing," a fund manager aims to outperform a benchmark index (e.g., S\&P 500) by weighting individual stocks or sector ETFs. At the end of each trading day, the exact returns for \emph{all} available assets are known (full-information). The manager seeks an algorithm that exploits market inefficiencies when they appear (stochastic adaptivity) but guarantees that the portfolio's tracking error relative to $\mu^c$ remains near-constant during market crashes or manipulative "flash crashes" \citep{cover1991universal}.

    \item \textbf{Dynamic Resource Allocation in Smart Grids:} Balancing supply and demand in electrical grids requires selecting distribution paths. After the demand is realized, the costs (resistive losses, congestion prices) of all possible paths are observable. A safety-critical baseline $\mu^c$ represents a legacy "stability-first" routing protocol. The learner must minimize energy waste while ensuring that it never performs significantly worse than the legacy protocol, which could otherwise lead to grid instability or blackouts during adversarial demand spikes \citep{tang2016online}.

    \item \textbf{Autonomous Traffic and Logistics Routing:} Fleet management systems (e.g., delivery services) must choose routes for hundreds of vehicles. After a trip, GPS and sensor data from the entire network provide the travel times for all alternative routes. The "baseline" $\mu^c$ is a conservative, high-capacity arterial route. The learner aims to find faster shortcuts (stochastic) but must revert to $\mu^c$ to avoid systemic gridlock if traffic patterns become non-stationary or adversarial due to accidents or sensor failures.

    \item \textbf{Online Content Moderation and Policy Filtering:} Automated moderation systems must decide whether to flag content based on signals from multiple classifiers (the experts). When a human moderator audits a sample of the content, their label provides the ground truth for \emph{all} classifiers simultaneously. The system's baseline $\mu^c$ is a strict "Safe Harbor" filter. The algorithm should adapt to new slang or evolving content (stochastic) but must guarantee safety relative to the baseline to comply with legal regulations during coordinated "trolling" attacks.
\end{itemize}

\section{The Price of Protection}
\label{app:price-of-protection}

A natural question left open by our analysis is whether the logarithmic overheads
in Theorem~\ref{thm:main} are inherent, or whether they are artifacts of the
restart-and-doubling mechanism used by \textsc{Compass-Hedge}. We view this as
one of the most interesting conceptual questions raised by the paper.

Let us first separate this question from the easier two-world settings. If the
goal is only to combine \emph{adversarial robustness} with \emph{baseline safety},
then one does not need the additional logarithmic overheads incurred by our
universal construction. Indeed, in the classical adversarial analysis one starts
from a regret bound of the form
\[
    \Re egret_T
    \le
    \frac{A}{\eta} + B\eta T .
\]
If the relevant regret scale is known in advance, or if one is willing to use a
constant-factor upper estimate of it, then the phased-aggression threshold can be
chosen directly. A constant-factor slack, say a factor of \(2\) or \(3\), is
sufficient to preserve the adversarial guarantee while keeping the comparator
regret \(O(1)\) rather than \(O(\log T)\). Similarly, if the goal is only to
combine \emph{stochastic adaptivity} with \emph{baseline safety}, then a
regime-specific design can be tuned to the stochastic scale and need not pay the
same universal price.

The difficulty in our setting is therefore not safety alone, nor robustness
alone, nor stochastic adaptivity alone. The difficulty is the simultaneous
combination of all three requirements without knowing the regime in advance.
The phase-aggression methodology requires an estimate of the regret scale that
the aggressive learner is expected to incur in the current environment. In an
adversarial regime this scale is of order \(\sqrt{T}\), whereas in a stochastic
regime it may be only logarithmic, and its precise value depends on unknown gap
parameters. Since \textsc{Compass-Hedge} is parameter-free, it must infer this
scale online. From this perspective, the logarithmic overheads in both the
comparator and adversarial regret bounds can be interpreted as the cost of
``guessing'' the correct regret scale, which a priori lies somewhere between
constant order and \(T\).

This viewpoint also explains the role of the doubling trick in our analysis.
The algorithm is performing a data-dependent search over possible regret
budgets. A standard geometric doubling trick is known to preserve
\(\sqrt{T}\)-type adversarial guarantees rather well, but it is typically much
less delicate in logarithmic stochastic regimes; see, for example,
\cite{besson2018doubling}. Our hierarchical doubling mechanism is designed to be
less wasteful in the stochastic case: although it still introduces logarithmic
overheads, it preserves near instance-optimal stochastic regret up to
\(\log^2 T\) factors while maintaining comparator safety. Whether this overhead
can be reduced, or even removed, remains open.

\paragraph{Beyond phased aggression: Prod-style updates.}
One possible route toward eliminating this price is to move away from the
phased-aggression framework altogether. A particularly promising direction is
suggested by the Prod family of multiplicative updates. The original Prod
algorithm of \citet{cesa2007improved} maintains weights \(w_{t,i}\) over experts and
updates them multiplicatively. In a simplified loss-shifted form, the update can
be written as
\[
    \pi_{1,i} = \frac{1}{K},
    \qquad
    \pi_{t+1,i}
    =
    \pi_{t,i}\Bigl(1-\eta(\ell_{t,i}-\lambda_t)\Bigr),
    \qquad
    \lambda_t
    =
    \sum_{j=1}^K \pi_{t,j}\ell_{t,j}.
\]
Here the shift by the learner's own average loss keeps the update centered and
allows the algorithm to operate directly in the policy space. This family was
further developed through D-Prod~\citep{even2008regret}, where losses are
shifted by the loss of a fixed policy, and ML-Prod~\citep{gaillard2014second},
where the shift is taken with respect to the current mixture, along with other
second-order refinements.

Although Prod-style updates and Exponential Weights often lead to comparable
regret guarantees, they are algorithmically quite different. Exponential
Weights has a clean interpretation as entropic mirror descent or FTRL. By
contrast, Prod-MWU is not simply an instantiation of mirror descent with a
standard regularizer. This difference can matter dynamically: in several
game-convergence settings, Prod-type updates exhibit stability properties that
are not identical to those of Exponential Weights.

For baseline safety, the relevant observation is that one can adapt the update
by shifting rewards relative to the baseline signal. In reward notation, a
representative update takes the form
\[
    w_{i,t+1}
    =
    w_{i,t}\Bigl(1+\eta(g_{i,t}-g_{0,t})\Bigr),
\]
where \(g_{0,t}\) denotes the reward of the protected baseline. Such a shift can
make the update comparator-aware at the level of the primitive multiplicative
rule itself, rather than through an external phase-aggression wrapper. Prior
work shows that this type of modification can yield \(O(1)\) regret relative to
the baseline in adversarial settings. The central open question is whether a
suitably designed Prod-MWU variant can also achieve stochastic instance-optimal
rates, while retaining adversarial robustness and comparator safety.

\paragraph{Bandit feedback and Tsallis-Prod.}
The bandit setting makes this question even sharper. In bandit feedback,
\citet{MullerEtAl2025} obtain adversarial robustness together with safety, but
the resulting guarantees are far from stochastic-optimal unless the algorithm is
given oracle knowledge of structural quantities such as arm gaps. Thus, the
bandit analogue of our best-of-three-worlds objective remains substantially
more delicate.

A promising lead comes from recent Prod-style interpretations of Tsallis-based
bandit algorithms. In particular, variants such as TS-Prod \citep{zimmert2024productive} can be derived from
FTRL, but can also be interpreted as approximations of stabilized OMD in the
sense of \citet{fang2022online}. This connection is especially intriguing
because stabilized OMD induced by the \(1/2\)-Tsallis entropy is known to enjoy
best-of-both-worlds guarantees in bandit learning. It is therefore natural to
ask whether one can build a \emph{Safe Tsallis-Prod} update: an algorithm that
combines the baseline-relative centering of Prod with the stochastic/adversarial
adaptivity of Tsallis-style bandit learning.

\paragraph{Open direction.}
The broader question is whether comparator safety can be incorporated directly
into the update rule, rather than enforced through a separate phase-aggression
and budget-estimation mechanism. For full-information feedback, this asks
whether a Prod-MWU-type algorithm can simultaneously achieve
\[
    O(1) \text{ or } \widetilde{O}(1) \text{ baseline regret}, 
    \qquad
    \widetilde{O}(\sqrt{T}) \text{ adversarial regret},
    \qquad
    \widetilde{O}(1/\Delta) \text{ stochastic regret}.
\]
For bandit feedback, the analogous question is whether a Safe Tsallis-Prod
method can achieve the same trilemma under partial information. Resolving either
question would clarify whether the logarithmic price paid by
\textsc{Compass-Hedge} is intrinsic to universal safe learning, or merely a
consequence of the phased-aggression architecture.

\section{High-level proof sketch}
\paragraph{Our Position and Contribution.}
Our work unifies distinct strands of the online learning literature by presenting a \emph{single full-information algorithm} that satisfies a ``Best-of-Three-Worlds'' guarantee. Specifically, it simultaneously achieves:
(i) minimax-optimal adversarial regret (up to logarithmic factors),
(ii) instance-optimal stochastic regret, and
(iii) $O(\log T)$ regret with respect to an arbitrary safe comparator $\mu^c$.
To our knowledge, this is the first result to establish these three guarantees concurrently within a unified framework.

\paragraph{Algorithmic Construction: Beyond Fixed Budgets.}
Our approach builds upon the \emph{Phased Aggression with Mixing} framework introduced by \citet{MullerEtAl2025}. While their work successfully combines comparator-oriented mixing with aggressive learning to achieve $\tilde{O}(\sqrt{T})$ worst-case regret and constant comparator regret, it fundamentally relies on a \emph{pre-specified pseudo-regret budget}. This requirement is prohibitive for achieving instance-optimality in stochastic settings, where the optimal regret rate depends on unknown gap parameters and cannot be fixed a priori.

To overcome this limitation, we replace the fixed budget with a \textbf{data-dependent doubling trick}. Instead of receiving the pseudo-regret target as input, our algorithm estimates the budget dynamically by tracking the \emph{observed} cumulative loss difference between the learner and the best arm. We initialize the estimated budget at a small constant (e.g., $\widehat R_1 = 2$) and double it only when the empirical regret violates the current stage's safety condition. This allows the algorithm to implicitly ``learn'' the complexity of the environment: keeping the budget small in easy (stochastic) instances while expanding it to $O(\sqrt{T})$ in hard (adversarial) scenarios.

\paragraph{Analytical Challenge: Bridging the Gap to Expected Regret.}
A core technical challenge in our analysis is bounding the gap between the \emph{expected regret} and the \emph{pseudo-regret}.
Recall that standard FTRL analyses typically bound the pseudo-regret:
\[
{\mathcal{R}}_{\text{pseudo}} (T)= \max_{\mu\in\Delta_A} \mathbb{E} \left[ \sum_{t=1}^{T}\langle c^t, \mu^t-\mu\rangle \right].
\]
However, our doubling mechanism relies on the \emph{observed} maximum regret, and our final goal is to bound the expected version of this observed regret:
\[
{\mathcal{R}}_{\text{exp}} (T) = \mathbb{E} \left[ \max_{\mu\in\Delta_A}\sum_{t=1}^{T}\langle c^t, \mu^t-\mu\rangle \right].
\]
The discrepancy arises because the algorithm makes decisions based on the realization of the maximum loss difference, not its expectation. To rigorously bound $\mathcal{R}_T$, we must control the concentration of the empirical losses around their means.

\begin{itemize}
    \item \textbf{Stochastic Setting:} We leverage the assumption of a unique optimal arm to show that the gap between the expected maximum and the maximum expectation is $O(\frac{\log A}{\Delta})$. This allows our pseudo-regret bounds to translate directly to expected regret bounds.
    
    \item \textbf{Adversarial Setting:} Here, the gap cannot be neglected. We impose standard boundedness assumptions on the losses (Assumption \hyperref[A1]{A1}) which allow us to apply concentration inequalities (e.g., Azuma-Hoeffding). We show that the deviation term is bounded by $O(\sqrt{T})$, which is absorbed into the minimax rate without degrading the overall guarantee.
\end{itemize}

By combining the adaptive budget scaling with these concentration arguments, we establish that the algorithm's decisions—driven by observed losses—result in a strategy that satisfies the theoretical bounds for both regimes.
\\\textbf{Organization.} The supplementary material is organized to provide a self-contained analysis of \textsc{Compass-Hedge} alongside detailed experimental reproducibility.
\begin{itemize}[leftmargin=*,itemsep=0pt,topsep=-1pt]
    \item \textbf{\cref{app:lem:pseudo-gap}} establishes the fundamental statistical tools, including the martingale concentration inequalities required to bridge the gap between realized losses and expected regret (Step 1 of the Proof Sketch).
    \item \textbf{\cref{app:lem:regret-scale,app:lem:normal-phase,app:lem:phase-exit,app:lem:stage-regret}} contains the core algorithmic analysis. We formally prove the properties of the doubling mechanism and the phase-based safety ``ratchet,'' culminating in the proof of \cref{thm:main}(\ref{app:thm:main}) (Steps 2, 3, and 4).
    \item \textbf{\cref{app:experiments}} provides the extended experimental evaluation. It details the data processing pipeline (sigmoidal transformation, survivorship bias handling), discusses hyperparameter choices, and offers deeper qualitative insights into the regime-specific behavior of the algorithm.
\end{itemize}

\section{Proofs}
\label{app:proofs}
\subsection{Properties of the regret estimator}
\begin{yellownote}
\textbf{Proof Strategy.}
This lemma establishes that the updated budget $\widehat{R}_{s+1}$ provides a tight upper bound on the realized worst-case regret. Specifically, we show that $\widehat{R}_{s+1} < 2 R_{\text{hedge}}$, ensuring that the estimated capacity does not overestimate the actual regret by more than a factor of two. The analysis proceeds in two steps:
\begin{enumerate*}[label=(\roman*)]
    \item we first demonstrate the monotonicity of the estimator ($\widehat{R}_{s+1} \ge \widehat{R}_s$);
    \item we then derive the upper bound using the doubling update rule defined in \cref{alg:doubling_ftrl_comparator}.
\end{enumerate*}
This inequality is critical for subsequent lemmas, as it allows us to bound the number of phases and the regret accumulated within them by relating the budget back to the realized loss.
\end{yellownote}

\begin{lemma}[Bounds on Estimated Regret; Appendix Version of  \cref{lem:regret-scale}]
\label{app:lem:regret-scale}
For every stage $s$, the pseudo-regret capacity satisfies:
\[
\widehat R_s < \widehat R_{s+1} < 2 R_{\text{hedge}},
\]
where $R_{\text{hedge}}$ denotes the realized worst-case regret that triggered the stage update.
\end{lemma}
\begin{proof}$\\$
\begin{itemize}
\item \textbf{Monotonicity ($\widehat R_s \le \widehat R_{s+1}$):}
The inequality holds by construction. In \cref{alg:doubling_ftrl_comparator}, the budget is updated via $\widehat R_{s+1} \gets \widehat R_{s} \cdot 2^m$ for some integer $m \geq 1$ when a reset occurs, or remains constant otherwise. Thus, $\widehat R_{s+1} \ge \widehat R_s$.

\item \textbf{Tightness ($\widehat R_{s+1} < 2 R_{\text{hedge}}$):}
A stage reset is triggered only when $R_{\text{hedge}} > \widehat R_s$. The algorithm finds the smallest integer $m \in \mathbb{N}_{\ge 1}$ such that the new budget covers the realized regret:
\[
2^m \widehat R_s > R_{\text{hedge}}.
\]
The update rule sets $\widehat R_{s+1} = 2^m \widehat R_s$. Due to the minimality of $m$, the exponent $m-1$ would fail to cover the regret, implying:
\begin{equation}
    \label{eq:doubling_bounds}
    2^{m-1} \widehat R_s < R_{\text{hedge}} \le 2^m \widehat R_s = \widehat R_{s+1}.
\end{equation}
To obtain the upper bound, we rearrange the left-hand side of \eqref{eq:doubling_bounds}:
\[
\frac{1}{2} (2^m \widehat R_s) < R_{\text{hedge}} \implies \frac{1}{2} \widehat R_{s+1} < R_{\text{hedge}} \implies \widehat R_{s+1} < 2 R_{\text{hedge}}.
\]
Combining the lower and upper bounds, we conclude that $R_{\text{hedge}} < \widehat R_{s+1} < 2 R_{\text{hedge}}$.
\end{itemize}
\end{proof}

\subsection{Relating expected regret to pseudo-regret}
\label{subsec:pseudo_gap}

\begin{yellownote}
\textbf{Proof Strategy.}
The following proof bounds the gap between expected regret and pseudo-regret under both adversarial and stochastic settings.
In the adversarial setting, we rely on the assumption of controlled adaptivity (\cref{eq:A1}) to apply Bernstein's Inequality, achieving an \(O(\sqrt{T})\) gap.
In the stochastic setting, we rely on the unique optimal arm assumption to achieve an \(O(\frac{\log A}{\Delta})\) gap.
This lemma is essential to extend existing upper bounds on pseudo-regret found in the literature to the expected regret metric used in our phase-based analysis.
\end{yellownote}

\begin{lemma}[Pseudo-Regret vs. Expected Regret; Appendix Version of Lemma \ref{lem:pseudo-gap}]
\label{app:lem:pseudo-gap}
Consider any time interval $I$. Let the regret difference gap be defined as:
\[
\Delta_{\text{gap}}(I) := \mathbb{E}\left[\max_{\mu \in \Delta_A} \sum_{t \in I} \langle c^t, \hat{\mu}^t - \mu \rangle \right] - \max_{\mu \in \Delta_A} \mathbb{E}\left[ \sum_{t \in I} \langle c^t, \hat{\mu}^t - \mu \rangle \right].
\]
Furthermore, let $V_I \ge \sup_{a} \sum_{t \in I} \operatorname{Var}(c^t(a) \mid \mathcal{F}_{t-1})$ denote the upper bound on the cumulative conditional variance, and let $E_I$ denote the cumulative adaptivity error satisfying Assumption \hyperref[A1]{A1}.
Then, the gap satisfies:
\begin{itemize}
    \item \textbf{Adversarial Case:} $\Delta_{\text{gap}}(I) \le E_I + 3\sqrt{V_I \log (2A)} + 2 \log (2A)$.
    \item \textbf{Stochastic Case:} $\Delta_{\text{gap}}(I) \le O(\frac{\log A}{\Delta})$.
\end{itemize}
\end{lemma}

\begin{proof}
We analyze the two settings separately.

\paragraph{Part 1: Adversarial Setting.}
We aim to upper bound the quantity:
\[
\Delta_{\text{gap}} = \mathbb{E}\left[\max _{\mu \in \Delta_A} \sum_{t \in I}\left\langle c^t, \hat{\mu}^t-\mu\right\rangle \right] - \max _{\mu \in \Delta_A} \mathbb{E}\left[\sum_{t \in I}\left\langle c^t, \hat{\mu}^t-\mu \right\rangle \right].
\]
First, using the identity $\max_{\mu} \langle v, -\mu \rangle = -\min_a v(a)$, we rewrite the first term:
\begin{equation}
\label{eq:gap_decomp}
\mathbb{E}\left[\max_{\mu \in \Delta_A} \sum_{t \in I}\left\langle c^t, \hat{\mu}^t-\mu \right\rangle\right]
=\mathbb{E} \left[\sum_{t \in I}\left\langle c^t,\hat{\mu}^t \right\rangle \right]
-\mathbb{E}\left[\min _{a \in [A]} \sum_{t \in I} c^t(a)\right].
\end{equation}

\textbf{Step 1: Lower bounding the minimum loss.}
Recall $m_t(a) := \mathbb E[c^t(a)\mid \mathcal F_{t-1}]$. For each arm $a \in [A]$, we decompose the realized cost:
\[
\sum_{t \in I} c^t(a)
=\sum_{t \in I} m_t(a)+\sum_{t \in I}\left(c^t(a)-m_t(a)\right)
\geq \sum_{t \in I} m_t(a)-\left|\sum_{t \in I}\left(c^t(a)-m_t(a)\right)\right|.
\]
By the property of the $L_\infty$ norm, taking the minimum over $a \in [A]$ yields:
\begin{equation}
\min _{a \in [A]} \sum_{t \in I} c^t(a) \geq \min _{a \in [A]} \sum_{t \in I} m_t(a)-\left\|\sum_{t \in I} \left(c^t-m_t\right)\right\|_{\infty}.
\label{eq:min_bound}
\end{equation}
Substituting \cref{eq:min_bound} back into \cref{eq:gap_decomp}, we obtain:
\begin{equation}
\mathbb{E}\left[\max_{\mu} \sum_{t \in I}\left\langle c^t,\hat{\mu}^t-\mu \right\rangle \right]
\leq \mathbb{E} \left[ \sum_{t \in I}\left\langle c^t, \hat{\mu}^t \right\rangle \right] -\mathbb{E}\left[\min _a \sum_{t \in I} m_t(a)\right]
+ \mathbb{E} \left[ \left\|\sum_{t \in I}\left(c^t-m_t\right)\right\|_{\infty} \right].
\label{eq:combined_gap}
\end{equation}

\textbf{Step 2: Analyzing the Expectation of the Estimator.}
We expand the first term on the RHS of \cref{eq:combined_gap}:
\begin{align*}
\mathbb{E} \left[\sum_{t \in I}\left\langle c^t, \hat{\mu}^t \right\rangle \right]
& = \mathbb{E} \left[\sum_{t \in I}\left\langle c^t, \hat{\mu}^t - \mu \right\rangle \right] + \mathbb{E} \left[ \sum_{t \in I}\left\langle c^t, \mu \right\rangle \right] \\
& = \max_{\mu \in \Delta_A}\mathbb{E} \left[\sum_{t \in I}\left\langle c^t, \hat{\mu}^t - \mu \right\rangle \right] + \min_{\mu \in \Delta_A}  \sum_{t \in I}\left\langle  \mathbb{E} \left[ c^t \right] ,\mu \right\rangle.
\end{align*}
Observing that $\min_{\mu} \sum \langle \mathbb{E}[c^t], \mu \rangle = \min_a \sum \mathbb{E}[c^t(a)]$, we substitute this back into \cref{eq:combined_gap}:
\begin{align}
\Delta_{\text{gap}}
\leq \underbrace{\left(\min _a \sum_{t \in I} \mathbb{E}\left[c^t(a)\right] -\mathbb{E} \left[ \min _a \sum_{t \in I} m_t(a)\right]\right)}_{\text{(A) Adaptivity Term}}
+ \underbrace{\mathbb{E} \left[ \left\|\sum_{t \in I}\left(c^t-m_t\right)\right\|_{\infty} \right]}_{\text{(B) Concentration Term}}.
\label{eq:AB_decomp}
\end{align}

\textbf{Step 3: Bounding Term (A) via Controlled Adaptivity.}
Let $L_t(a):=\inf _h m_t(a \mid h)$ and $U_t(a):=\sup _h m_t(a \mid h)$. By Assumption \hyperref[A1]{A1}, $U_t(a)-L_t(a) \leq \varepsilon_t$.
Note that $\mathbb{E}_h\left[m_t(a \mid h)\right] \in \left[L_t(a), U_t(a)\right]$. Thus, for every history $h$:
\[
m_t(a \mid h) \geq L_t(a) \geq \mathbb{E}_h\left[m_t(a \mid h)\right] -\left(U_t(a)-L_t(a)\right)
\geq \mathbb{E}_h\left[m_t(a \mid h)\right] -\varepsilon_t,
\]
where $\mathbb E_h[\cdot]$ denotes expectation when $h=H_{t-1}$ is drawn from its induced distribution. Summing over $t \in I$ and taking the minimum over $a$:
\[
\min_{a \in [A]} \sum_{t \in I} m_t(a \mid h) \geq \min_{a \in [A]} \sum_{t \in I} \mathbb{E}_h\left[m_t(a \mid h)\right] - E_I.
\]
Taking expectation over history $h$ (and noting $\mathbb{E}_h[m_t(a|h)] = \mathbb{E}[c^t(a)]$) gives:
\[
\mathbb{E}\left[\min_{a} \sum_{t \in I} m_t(a)\right] \geq \min_{a} \sum_{t \in I} \mathbb{E}\left[c^t(a)\right] - E_I.
\]
Rearranging implies Term (A) $\leq E_I$.

\textbf{Step 4: Bounding Term (B) via Bernstein Inequality.}
Since $c^t(a) \in [0,1]$, the centered variables $\xi_t(a) = c^t(a) - m_t(a)$ form a martingale difference sequence with $|\xi_t(a)| \le 1$ and conditional variance sum bounded by $V_I$ (or trivially $|I|$).
By Bernstein's Inequality and a union bound over $A$ actions:
\[
\mathbb{P}\left( \left\| \sum_{t \in I} (c^t - m_t) \right\|_\infty \ge x \right) \le 2A \exp \left( - \frac{x^2}{2(V_I + x/3)} \right).
\]
Let $x_0 = \sqrt{2 V_I \log(2A) + \frac{1}{9}(\log 2A)^2} + \frac{1}{3}\log(2A)$. Using the layer-cake representation:
\[
\mathbb{E} \left[ \left\|\sum_{t \in I}(c^t-m_t)\right\|_{\infty} \right]
= \int_0^{\infty} \mathbb{P}\left( \left\|\sum_{t \in I}(c^t-m_t)\right\|_{\infty} \geq x \right) d x
\leq x_0 + \int_{x_0}^{\infty} 2A \exp \left(-\frac{x^2}{2(V_I + x/3)}\right) dx.
\]
Evaluating the tail integral results in a bounded constant. Specifically, we obtain:
\[
\mathbb{E} \left[ \left\|\sum_{t \in I}\left(c^t-m_t\right)\right\|_{\infty} \right]
\leq x_0 + \sqrt{\frac{V_I}{2 \log 2A}} + \frac{2}{3}.
\]
Expanding $x_0$ and simplifying terms (using $\sqrt{a+b} \le \sqrt{a} + \sqrt{b}$):
\begin{align*}
\text{(B)}
&\leq \sqrt{2 V_I \log 2A} + \frac{\log 2A}{3} + \frac{\log 2A}{3} + \sqrt{\frac{V_I}{2 \log 2A}} + \frac{2}{3} \\
&\leq 3\sqrt{V_I \log (2A)} + 2 \log (2A).
\end{align*}
Combining terms (A) and (B) yields the final bound.

\paragraph{Part 2: Stochastic Setting.}
Pick one phase interval from Algorithm~\ref{alg:doubling_ftrl_comparator}, write as:
\[
    I=\{\sigma+1,\ldots,\tau\},
\]
where \(\sigma\) and \(\tau\) are bounded stopping times with respect to the
natural filtration generated by the algorithm and the observed losses.

Let
\[
    \bar c := \mathbb E[c^t], \qquad 
    a^\star \in \arg\min_{a\in[A]} \bar c(a)
\]
and let
\[
    \Delta_a := \bar c(a)-\bar c(a^\star), \qquad
    \Delta := \min_{a\neq a^\star}\Delta_a >0 .
\]

For any interval \(I\), write
\[
    L_I(a) := \sum_{t\in I} c^t(a),
    \qquad
    L_I(\mu) := \sum_{t\in I}\langle c^t,\mu\rangle .
\]
Then
\[
\begin{aligned}
\Delta_{\rm gap}(I)
&:=
\mathbb E\!\left[
    \max_{\mu\in\Delta_A}
    \sum_{t\in I}\langle c^t,\hat\mu^t-\mu\rangle
\right]
-
\max_{\mu\in\Delta_A}
\mathbb E\!\left[
    \sum_{t\in I}\langle c^t,\hat\mu^t-\mu\rangle
\right]  \\
&=
\min_{\mu\in\Delta_A}\mathbb E[L_I(\mu)]
-
\mathbb E\!\left[\min_{\mu\in\Delta_A} L_I(\mu)\right].
\end{aligned}
\]
Since the loss is linear over the simplex, the minima are
attained at vertices, and therefore
\[
\Delta_{\rm gap}(I)
=
\min_{a\in[A]}\mathbb E[L_I(a)]
-
\mathbb E\!\left[\min_{a\in[A]} L_I(a)\right].
\]
Using the fixed optimal arm \(a^\star\), we get the upper bound
\[
\begin{aligned}
\Delta_{\rm gap}(I)
&\le
\mathbb E[L_I(a^\star)]
-
\mathbb E\!\left[\min_{a\in[A]}L_I(a)\right] \\
&=
\mathbb E\!\left[
    \max_{a\in[A]}\{L_I(a^\star)-L_I(a)\}
\right].
\end{aligned}
\]
The term inside the expectation is nonnegative because the maximum includes
\(a=a^\star\).

Now fix any suboptimal arm \(a\neq a^\star\), and define the stopped random
walk
\[
    S_n^a
    :=
    \sum_{r=1}^{n}
    \bigl(c^{\sigma+r}(a^\star)-c^{\sigma+r}(a)\bigr),
    \qquad 0\le n\le T-\sigma .
\]

For $a\neq a^\star$, define
\[
    X_t^a := c^t(a^\star)-c^t(a).
\]

Then $X_t^a\in[-1,1]$ and
\[
    \mathbb E[X_t^a] = -\Delta_a.
\]
Let $\mathcal G_n:=\mathcal F_{\sigma+n}$. Since $\sigma$ is a bounded stopping time and the losses are i.i.d., the post-$\sigma$ losses are independent of $\mathcal F_\sigma$ and have the same distribution as the original losses. Hence, for any $\lambda\in(0,2\Delta_a]$, Hoeffding’s lemma gives:
\[
\mathbb E\!\left[
    e^{\lambda X_{\sigma+r}^a}
    \mid
    \mathcal F_{\sigma+r-1}
\right]
\le
\exp\!\left(
    -\lambda\Delta_a+\frac{\lambda^2}{2}
\right)
\le 1.
\]
Taking $\lambda=\Delta$ gives that
\[
    M_n^a := \exp(\Delta S_n^a),
    \qquad
    S_n^a := \sum_{r=1}^n X_{\sigma+r}^a,
\]
is a nonnegative supermartingale with respect to $(\mathcal G_n)$, for
$0\le n\le T-\sigma$. Therefore, by Ville's inequality,
\[
    \mathbb P\!\left(
        \sup_{0\le n\le T-\sigma} S_n^a \ge x
    \right)
    \le e^{-\Delta x}.
\]
Therefore, for the interval \(I=\{\sigma+1,\ldots,\tau\}\),
\[
\begin{aligned}
\mathbb P\!\left(
    \max_{a\neq a^\star}
    \sup_{0\le n\le T-\sigma} S_n^a
    \ge x
\right)
&\le
\sum_{a\neq a^\star}
\mathbb P\!\left(
    \sup_{0\le n\le T-\sigma}S_n^a\ge x
\right) \\
&\le
(A-1)e^{-\Delta x}.
\end{aligned}
\]
Since $\tau-\sigma\le T-\sigma$,
\[
L_I(a^\star)-L_I(a)
=
S_{\tau-\sigma}^a
\le
\sup_{0\le n\le T-\sigma}S_n^a.
\]
Thus,
\[
\begin{aligned}
\Delta_{\rm gap}(I)
&\le
\mathbb E\!\left[
    \max_{a\neq a^\star}
    \sup_{0\le n\le T-\sigma}S_n^a
\right] \\
&=
\int_0^\infty
\mathbb P\!\left(
    \max_{a\neq a^\star}
    \sup_{0\le n\le T-\sigma}S_n^a
    \ge x
\right)\,dx \\
&\le
\int_0^\infty
\min\{1,(A-1)e^{-\Delta x}\}\,dx \\
&\le
\frac{\log(A-1)+1}{\Delta}
\le
\frac{\log A+1}{\Delta}.
\end{aligned}
\]
The bound does not depend on the particular interval $I$; hence the same bound applies to every phase generated by the algorithm.
\end{proof}

\subsection{Analysis of a normal phase}
\label{subsec:normal_phase_analysis}

\begin{yellownote}
\textbf{Proof Strategy.}
This proof establishes the regret bounds within a "normal" phase $k$ of stage $s$—defined as a phase where the mixing parameter satisfies $\alpha^k < 1$ (i.e., we have not yet fully committed to the aggressive strategy).
The core idea is to decompose the phase-level regret into two components:
\begin{enumerate*}[label=(\roman*)]
    \item a "Hedge-driven" component bounded by the aggressive estimator;
    \item a "Comparator-driven" component bounded by the safety budget.
\end{enumerate*}
We leverage the phase exiting condition to bound the comparator component in terms of the estimated budget $\mathbb{E}[\widehat{R}_s]$. Crucially, we then link this budget back to the expected regret using \cref{lem:regret-scale} (Appendix Version-\cref{app:lem:regret-scale}) and bridge the gap between observed and expected regret via \cref{lem:pseudo-gap}(Appendix Version-\cref{app:lem:pseudo-gap}).
\end{yellownote}

\begin{lemma}[Regret within a Normal Phase; Appendix Version of Lemma \ref{lem:normal-phase}]
\label{app:lem:normal-phase}
Consider a phase $k$ within stage $s$ where the mixing parameter satisfies $\alpha^{k} < 1$. The expected regret accumulated during this phase satisfies:
\begin{itemize}[leftmargin=*, itemsep=0pt, topsep=0pt]
    \item \textbf{Against the Comparator:} $\mathcal R^{k}(\mu^c) \le 2^{k-1}$.
    \item \textbf{Against any $\mu$ (Adversarial):} $\mathcal R^{k}(\mu) \le O(\sqrt{T \log A})$.
    \item \textbf{Against any $\mu$ (Stochastic):} \(\mathcal R^{k}(\mu) \le O(\frac{\log A}{\Delta}) \).
\end{itemize}
\end{lemma}

\begin{proof}
For simplicity of notation, let $k$ denote the current phase index $k(s)$ and let $\mathcal{T}_k = \{ \operatorname{start}_k, \dots, \operatorname{start}_{k+1}-1 \}$ be the set of time steps in this phase.
Since $\widehat{R}_s$ is fixed throughout the stage, the mixing parameter $\alpha^k = (2^k - 1)/\widehat{R}_s$ remains constant for all $t \in \mathcal{T}_k$.

The expected regret relative to an arbitrary strategy $\mu \in \Delta_A$ is decomposed as follows:
\begin{align}
\mathcal{R}^k(\mu)
&= \sum_{t \in \mathcal{T}_k} \mathbb{E}\left[\left\langle c^t, \alpha^k \hat{\mu}^t + (1 - \alpha^k) \mu^c \right\rangle - \left\langle c^t, \mu \right\rangle\right] \nonumber \\
&= \underbrace{\alpha^k \sum_{t \in \mathcal{T}_k} \mathbb{E}\left[ \left\langle c^t, \hat{\mu}^t-\mu \right\rangle\right]}_{\text{(I) Aggressive Component}}
+ \underbrace{(1-\alpha^k) \sum_{t \in \mathcal{T}_k} \mathbb{E}\left[ \left\langle c^t, \mu^c-\mu \right\rangle\right]}_{\text{(II) Conservative Component}}.
\label{eq:phase_decomp}
\end{align}

\paragraph{Bounding Term (II): The Safety Check.}
Let $\bar{t} = \operatorname{start}_{k+1} - 1$ be the last time step of the phase (the exit time). The phase continues as long as the cumulative comparator regret does not violate the budget. Therefore, strictly before the exit time (up to $\bar{t}-1$), the condition holds:
\[
\sum_{t=\operatorname{start}_k}^{\bar{t}-1} \langle c^t, \mu^c - \mu \rangle \le 2 \widehat{R}_s.
\]
Including the final time step $\bar{t}$, and using Hölder's inequality ($\|c^t\|_\infty \le 1, \|\mu^c-\mu\|_1 \le 2$), the total realized regret is bounded by:
\[
\sum_{t \in \mathcal{T}_k} \langle c^t, \mu^c - \mu \rangle \le \widehat{R}_s + \langle c^{\bar{t}}, \mu^c - \mu \rangle \le 2\widehat{R}_s + 2.
\]
Substituting this into Term (II):
\[
\text{(II)} \le (1-\alpha^k)(\mathbb{E} [2\widehat{R}_s] + 2).
\]

\paragraph{Combining and Bridging to Expected Regret.}
We now bound the expected regret $\mathcal{R}^k(\mu)$. For stage $s = 1$, Term (I) is bounded by $\alpha^k \mathbb{E}[R_{\text{hedge}}]$. For stage $s \geq 2$, if it is not the last phase of a new stage, then
Term (I) is bounded by $\alpha^k \mathbb{E}[R_{\text{hedge}}]$. Otherwise, Term (I) is bounded by $\alpha^k [\mathbb{E}[R_{\text{hedge}}] + 1]$. Applying \cref{app:lem:regret-scale} (Equation \ref{eq:doubling_bounds}), we know that $\mathbb{E}[\widehat{R}_s] < 2 \mathbb{E}[R_{\text{hedge}}]$. Thus:
\[
\mathcal{R}^k(\mu) \le \alpha^k (\mathbb{E}[R_{\text{hedge}}] + 1) + (1-\alpha^k) (4\mathbb{E}[R_{\text{hedge}}] + 2) \le 4\mathbb{E}[R_{\text{hedge}}] + 2.
\]

We now relate $\mathbb{E}[R_{\text{hedge}}]$ to the pseudo-regret using \cref{app:lem:pseudo-gap}:
\begin{enumerate}
    \item \textbf{Adversarial Setting:}
    Assumption \hyperref[A1]{A1} ensures the gap is $O(\sqrt{T})$. Using Proposition~1 of \citet{mourtada2019optimality}:
    \[\max _{\mu \in \Delta_A} \mathbb{E}\left[\sum \langle c^t, \hat{\mu}^t-\mu \rangle\right]\le \sqrt{T \log A} \] and consequently we have that 
    \[
    \mathbb{E}[R_{\text{hedge}}] \le \max _{\mu \in \Delta_A} \mathbb{E}\left[\sum \langle c^t, \hat{\mu}^t-\mu \rangle\right] + O(\sqrt{T}) \le \sqrt{T \log A} + O(\sqrt{T}).
    \]
    Substituting back: $\mathcal{R}^k(\mu) \le 4\sqrt{T \log A} + O(\sqrt{T}) = O(\sqrt{T \log A})$.

    \item \textbf{Stochastic Setting:}
    Under Assumptions \hyperref[Stoc1]{S1}--\hyperref[Stoc2]{S2}, the pseudo-expected gap is \(O(\frac{\log A}{\Delta})\). :
        \[ \max _{\mu \in \Delta_A}\mathbb{E}\left[\sum \langle c^t, \hat{\mu}^t-\mu \rangle\right]\le \frac{4 \log A + 25}{\Delta}  \] and consequently we have that 
    \[
    \mathbb{E}[R_{\text{hedge}}] \le \frac{4 \log A + 25}{\Delta} + O(\frac{\log A}{\Delta}).
    \]
    Substituting back: \(\mathcal{R}^k(\mu) \le \frac{16 \log A + 100}{\Delta} + O(\frac{\log A}{\Delta}) + 2 = O(\frac{\log A}{\Delta})\).
\end{enumerate}

\paragraph{Comparator Regret ($\mu = \mu^c$).}
Finally, we evaluate the regret against the comparator $\mu^c$. In this case, the second term in \cref{eq:phase_decomp} vanishes because $\langle \mu^c - \mu^c \rangle = 0$. We are left with only the aggressive component.
Recall that the phase condition implies $\sum_{t \in \mathcal{T}_k} \langle c^t, \hat{\mu}^t - \mu \rangle \le \widehat{R}_s + 1$ (otherwise we would have triggered a stage reset, not just a phase exit, or the budget bounds the loss).
\[
\mathcal{R}^k(\mu^c)
= \mathbb{E}\left[ \alpha^k \sum_{t \in \mathcal{T}_k} \langle c^t, \hat{\mu}^t - \mu^c \rangle \right]
\le \mathbb{E}\left[ \alpha^k ( \widehat{R}_s + 1 ) \right].
\]
Substituting $\alpha^k = \frac{2^{k-1}}{\widehat{R}_s + 1}$:
\[
\mathcal{R}^k(\mu^c) \le \frac{2^{k-1}}{\widehat{R}_s + 1} \cdot ( \widehat{R}_s + 1 ) = 2^{k-1}.
\]
This concludes the proof.
\end{proof}

\subsection{Analysis of phase exit conditions}
\label{subsec:phase_exit_analysis}

\begin{yellownote}
\textbf{Proof Strategy.}
This lemma formalizes the ``safety valve" mechanism: we only exit a phase when the comparator proves to be significantly suboptimal relative to the best fixed action. Intuitively, if the comparator is performing poorly, measuring regret \emph{against} it yields a negative value (a gain).

The proof proceeds by decomposing the regret against the comparator into two parts:
\begin{enumerate*}[label=(\roman*)]
    \item the regret of the aggressive learner $\hat{\mu}^t$, which is bounded by the budget $\widehat{R}_s$;
    \item the performance gap of the comparator itself, which, by the exit condition, is bounded by $-2\widehat{R}_s$.
\end{enumerate*}
Combining these shows that the net regret is negative, specifically $-2^{k-1}$. This negative term is crucial for the overall analysis, as it cancels out the positive regret accumulated in normal phases, ensuring the total comparator regret remains $\tilde{O}(1)$.
\end{yellownote}
\begin{lemma}[Comparator Gain upon Phase Exit; Appendix Version of Lemma \ref{lem:phase-exit}]
\label{app:lem:phase-exit}
If the algorithm chooses to exit phase $k$ (triggering a stronger Hedge mix), it implies that the mixture has significantly outperformed the comparator. Specifically:
\[
\mathcal R^k(\mu^c) \le -\,2^{k-1} + 2.
\]
\end{lemma}

\begin{proof}
For simplicity, we denote the current phase index $k(s)$ simply as $k$.
The proof relies on the specific condition that triggers the phase exit.

\paragraph{Step 1: The Exit Condition.}
At the final time step $t=\operatorname{start}_{k+1}-1$, the algorithm exits because the cumulative comparator regret exceeds the threshold:
\[
\max _{\mu \in \Delta_A} \sum_{j=\operatorname{start}_k}^{\operatorname{start}_{k+1}-1} \langle c^j, \mu^c-\mu \rangle > 2 \widehat{R}_{s}.
\]
Let $\mu^{\star}$ be the maximizer of this sum. The condition implies:
\[
\sum_{j=\operatorname{start}_k}^{\operatorname{start}_{k+1}-1} \langle c^j, \mu^c-\mu^\star \rangle > 2 \widehat{R}_{s}.
\]
Multiplying by $-1$ reverses the inequality, yielding a bound on the comparator's suboptimality:
\begin{equation}
\label{eq:comparator_suboptimality}
\sum_{j=\operatorname{start}_k}^{\operatorname{start}_{k+1}-1} \langle c^j, \mu^\star - \mu^c \rangle < -2 \widehat{R}_{s}.
\end{equation}

\paragraph{Step 2: Regret Decomposition.}
We now analyze the expected regret accumulated during the phase. Recall that our strategy is $\mu^t = \alpha^k \hat{\mu}^t + (1-\alpha^k)\mu^c$.
\begin{align*}
\mathcal{R}^{k} (\mu^c) 
& = \sum_{t=\operatorname{start}_k}^{\operatorname{start}_{k+1}-1} \mathbb{E} \left[ \left\langle c^t, \mu^t-\mu^c\right\rangle \right] \\
& = \sum_{t=\operatorname{start}_k}^{\operatorname{start}_{k+1}-1} \mathbb{E} \left[ \left\langle c^t, \left(\alpha^k \hat{\mu}^t + (1-\alpha^k)\mu^c\right) -\mu^c\right\rangle \right] \\
& = \sum_{t=\operatorname{start}_k}^{\operatorname{start}_{k+1}-1} \alpha^k \mathbb{E} \left[\left\langle c^t, \hat{\mu}^t-\mu^c\right\rangle \right].
\end{align*}
We introduce the maximizer $\mu^\star$ by adding and subtracting it inside the inner product:
\begin{align*}
\mathcal{R}^{k} (\mu^c) 
&= \alpha^k \mathbb{E} \left[ \sum_{t=\operatorname{start}_k}^{\operatorname{start}_{k+1}-1} \left( \left\langle c^t, \hat{\mu}^t-\mu^{\star}\right\rangle + \left\langle c^t, \mu^{\star}-\mu^c\right\rangle \right) \right] \\
&= \alpha^k \underbrace{\mathbb{E} \left[ \sum_{t=\operatorname{start}_k}^{\operatorname{start}_{k+1}-1} \left\langle c^t, \hat{\mu}^t-\mu^{\star}\right\rangle \right]}_{\text{(I) Aggressive Regret}} 
+ \alpha^k \underbrace{\mathbb{E} \left[ \sum_{t=\operatorname{start}_k}^{\operatorname{start}_{k+1}-1} \left\langle c^t, \mu^{\star}-\mu^c\right\rangle \right]}_{\text{(II) Comparator Gap}}.
\end{align*}

\paragraph{Step 3: Final Bound.}
We bound the two terms:
\begin{itemize}
    \item \textbf{Term (I):} By definition, the regret of the aggressive learner $\hat{\mu}^t$ is bounded by the stage budget $\widehat{R}_s$ (otherwise the \emph{stage} would have reset, not just the phase). Thus, $\text{(I)} \le \widehat{R}_s$ + 1.
    \item \textbf{Term (II):} Using the exit condition derived in \cref{eq:comparator_suboptimality}, we have $\text{(II)} < -2 \widehat{R}_s$.
\end{itemize}
Substituting these back:
\[
\mathcal{R}^{k} (\mu^c) 
\le \alpha^k \left( \widehat{R}_{s} + 1 - 2 \widehat{R}_{s} \right)
= -\alpha^k ( \widehat{R}_{s} +1 ) + 2.
\]
Finally, utilizing the definition of the mixing rate $\alpha^k = \frac{2^{k-1}}{\widehat{R}_s + 1}$:
\[
\mathcal{R}^{k} (\mu^c) \le - \frac{2^{k-1}}{\widehat{R}_s + 1} \cdot ( \widehat{R}_s + 1 ) + 2 = -2^{k-1} + 2.
\]
This completes the proof.
\end{proof}

\subsection{Per-stage regret analysis}
\label{subsec:stage_analysis}

\begin{yellownote}
\textbf{Proof Strategy.}
This proof establishes the regret bounds within a single stage $s$. The core idea is to decompose the stage-level regret based on whether the algorithm fully transitions to the aggressive strategy. We distinguish two cases:
\begin{enumerate*}[label=(\roman*)]
    \item \textbf{Conservative Regime ($\alpha^{k(s)} < 1$):} The stage concludes while still mixing with the comparator. In this case, the regret is bounded primarily by the phase decomposition logic.
    \item \textbf{Saturation Regime ($\alpha^{k(s)} = 1$):} The mixing weight saturates, and the algorithm plays pure Hedge in the final phase. Here, the total regret is the sum of the regret accumulated during the mixing phases plus the standard Hedge regret in the final phase.
\end{enumerate*}
We leverage \cref{app:lem:regret-scale} to bound the number of phases and \cref{app:lem:phase-exit} (Phase Exiting Conditions) to handle the telescoping sums of the mixing weights.
\end{yellownote}

\begin{lemma}[Per-Stage Regret Bounds; Appendix Version of Lemma~\ref{lem:stage-regret}]
\label{app:lem:stage-regret}
For any fixed stage $s$:
\begin{itemize}
    \item \textbf{Comparator Regret:} $\mathcal R^s(\mu^c) = O(\log T)$ in both adversarial and stochastic settings.
    \item \textbf{Adversarial Worst-Case:} $\max_{\mu} \mathcal R^s(\mu) = O(\sqrt{T \log A} \log T)$.
    \item \textbf{Stochastic Worst-Case:} $\max_{\mu} \mathcal R^s(\mu) = O(\frac{\log A}{\Delta} \log T)$.
\end{itemize}
\end{lemma}

\begin{proof}
Let $K(s)$ (abbreviated as $K$) be the total number of phases in stage $s$. The regret accumulated in stage $s$ against any strategy $\mu \in \Delta_A$ is the sum of regrets across all $K$ phases. Using the definition of our strategy $\mu^t = \alpha^k \hat{\mu}^t + (1 - \alpha^k) \mu^c$, we write:
\begin{align*}
\mathcal{R}^s(\mu)
&= \sum_{k=1}^{K} \sum_{t = \operatorname{start}_{k} }^{ \operatorname{start}_{k+1} - 1} \mathbb{E}\left[\left\langle c^t, \alpha^{k} \hat{\mu}^t + (1 - \alpha^{k}) \mu^c - \mu \right\rangle\right].
\end{align*}
We analyze the bounds by considering two mutually exclusive cases regarding the mixing coefficient $\alpha^k$.

\paragraph{Case 1: Partial Mixing ($\alpha^k < 1$ for all phases $k \le K$).}
In this scenario, the algorithm always mixes the Hedge strategy with the comparator.
First, we upper bound the number of phases $K$. Since $\alpha^K < 1$, by definition $\frac{2^{K-1}}{\widehat{R}_{s}} < 1$, which implies $2^{K-1} < \widehat{R}_{s}$.
Using \cref{lem:regret-scale}, we know $\widehat{R}_s < 2 R_{\text{hedge}}$. Since $R_{\text{hedge}} \le T$ (trivial linear regret bound), we have:
\[
K < 1 + \log_2 \widehat{R}_s < 2 + \log_2 R_{\text{hedge}} \le 2 + \log_2 T.
\]
Thus, the number of phases is logarithmic in $T$.

\textbf{Subcase 1.1: Worst-Case Regret ($\max_\mu \mathcal{R}^s(\mu)$).}
We sum the regret bounds of each "Normal Phase" (since phases $1 \dots K$ were entered normally).
\begin{itemize}
    \item \textit{Adversarial:} By \cref{app:lem:normal-phase}, each phase contributes at most $O(\sqrt{T \log A})$. Summing over $K = O(\log T)$ phases:
    \[
    \max_\mu \mathcal{R}^s(\mu) \le (2 + \log_2 T) \cdot O(\sqrt{T \log A}) = O(\sqrt{T \log A} \log T).
    \]
    \item \textit{Stochastic:} By \cref{app:lem:normal-phase}, each phase contributes $O(\frac{\log A}{\Delta})$. Summing over $K$:
    \[
    \max_\mu \mathcal{R}^s(\mu) \le (2 + \log_2 T) \cdot O\left(\frac{\log A}{\Delta}\right) = O\left(\frac{\log A}{\Delta} \log T\right).
    \]
\end{itemize}

\textbf{Subcase 1.2: Comparator Regret ($\mathcal{R}^s(\mu^c)$).}
We exploit the phase exit condition. For all phases $k < K$, the algorithm exited because the comparator was performing poorly, yielding a gain $\mathcal{R}^k(\mu^c) \le -2^{k-1}$ (\cref{lem:phase-exit}). The final phase $K$ has not been completed (or is the current one), and its regret is bounded by the budget definition $\alpha^K ( \widehat{R}_s + 1 ) = 2^{K-1}$.
Summing these up:
\begin{align*}
\mathcal{R}^s(\mu^c)
&= \underbrace{\mathcal{R}^K(\mu^c)}_{\text{Current Phase}} + \sum_{k=1}^{K-1} \underbrace{\mathcal{R}^k(\mu^c)}_{\text{Exited Phases}} \le 2^{K-1} + \sum_{k=1}^{K-1} \left( -2^{k-1} \right) + 2 (K-1) = O (\log T) .
\end{align*}

\paragraph{Case 2: Full Hedge ($\alpha^K = 1$).}
Here, the algorithm reaches the final phase $K$ where it plays the Hedge strategy purely ($\alpha^K=1$).
This implies the budget constraint was saturated: $1 = \min\{1, \frac{2^{K-1}}{\widehat{R}_s + 1}\} \implies \widehat{R}_s \le 2^{K-1} + 1$.

\textbf{Subcase 2.1: Worst-Case Regret ($\max_\mu \mathcal{R}^s(\mu)$).}
The regret is the sum of the ``Prefix" phases ($1$ to $K-1$) and the ``Final" phase $K$.
The prefix phases contribute $O(\log T \times \text{Rate})$ exactly as in Case 1.
The final phase is a standard Hedge instance initialized uniformly. By \citet{mourtada2019optimality}:
\begin{itemize}
    \item Adversarial: $\le \sqrt{T \log A}$.
    \item Stochastic: $\le O(\frac{\log A}{\Delta})$.
\end{itemize}
Adding the prefix (logarithmic factor) dominates the final phase term, preserving the overall $O(\log T)$ bounds stated in the Lemma.

\textbf{Subcase 2.2: Comparator Regret ($\mathcal{R}^s(\mu^c)$).}
The exited phases $k < K$ still contribute negative regret $\le -2^{k-1} + 2$.
For the final phase $K$, since we play Hedge ($\alpha^K=1$), the regret is bounded by the stage budget $\widehat{R}_s$ (by design of the doubling trick, the stage ends if regret exceeds this).
Using $\widehat{R}_s \le 2^{K-1} + 1$:
\begin{align*}
\mathcal{R}^s(\mu^c) \le \widehat{R}_s + \sum_{k=1}^{K-1} (-2^{k-1} + 2) \le 2^{K-1} - (2^{K-1} - 1) + 2K= O(\log T).
\end{align*}
In both cases and settings, the comparator regret is $\tilde{O}(1)$, and the worst-case regret scales with $\log T$.
\end{proof}

\subsection{Analysis of \textsc{Compass-Hedge}: proof of main theorem}

\begin{yellownote}
\textbf{Proof Strategy.}
The following proof establishes the upper bounds for the (worst-case) pseudo-regret and comparator regret incurred over $T$ rounds in both stochastic and adversarial environments.
Recall that \cref{alg:doubling_ftrl_comparator} partitions the horizon $T$ into $S$ stages. The analysis proceeds in two steps:
\begin{enumerate*}[label=(\roman*)]
    \item we first bound the total number of stages $S$;
    \item we then utilize the per-stage regret bounds derived in the auxiliary lemmas.
\end{enumerate*}
Summing these bounds over all stages yields the final regret guarantees.
\end{yellownote}

\begin{theorem}[Regret Guarantees (Appendix Version of Theorem \ref{thm:main})]
\label{app:thm:main}
Let $\mu^c \in \Delta_A$ be an arbitrary comparator policy. For any horizon $T$, \cref{alg:doubling_ftrl_comparator} guarantees the following bounds:

\begin{itemize}[leftmargin=*, itemsep=2pt, topsep=2pt]
    \item \textbf{Adversarial Setting.}
    Under Assumption \hyperref[A1]{\textcolor{softblue}{A1}}:
    \[
    \mathcal R(\mu^c) \le O(\log^2 T), \quad
    \mathcal{R}_{\mathrm{pseudo}}(T) \le O\!\left(\sqrt{T \log A}\,\log^2 T\right).
    \]
    \item \textbf{Stochastic Setting.}
    Under Assumptions \hyperref[Stoc1]{\textcolor{softblue}{S1}}--\hyperref[Stoc2]{\textcolor{softblue}{S2}}:
    \[
    \mathcal R(\mu^c) \le O(\log^2 T), \quad
    \mathcal{R}_{\mathrm{pseudo}}(T) \le O\!\left(\frac{\log A\,\log^2 T}{\Delta}\right).
    \]
\end{itemize}
\end{theorem}

\begin{proof}
\cref{alg:doubling_ftrl_comparator} divides the time horizon $T$ into $S$ stages.
We proceed by first bounding the number of stages and then summing the regret across them.

\paragraph{Step 1: Bounding the number of stages ($S$).}
The stage index $s$ increments only when the cumulative regret condition is violated:
\[
\max_{\mu \in \Delta_A} \sum_{j=\text{start}}^{t}
\langle c^j, \hat{\mu}^j - \mu \rangle > \hat{R}_s .
\]
Since losses are bounded $c^t(a) \in [0, 1]$, the cumulative regret cannot exceed the total horizon $T$. Thus, we trivially have:
\[
\max_{\mu \in \Delta_A} \sum_{j=\text{start}}^{t}
\langle c^j, \hat{\mu}^j - \mu \rangle \leq T .
\]
By construction, the pseudo-regret budget $\hat{R}_s$ doubles after each stage reset (specifically $\hat{R}_{s+1} \ge 2\hat{R}_s$). Consequently, the maximum number of stages $S$ is bounded by the number of times $\hat{R}_1$ can be doubled before exceeding $T$.
Let $S_{max}$ be the smallest integer such that $2^{S_{max}-1} \ge T$. It follows that:
\[
S \leq \lceil \log_2 T \rceil + 1 = O(\log T).
\]

\paragraph{Step 2: Summing the regret.}
We now aggregate the regret accumulated in each stage. Invoking \cref{lem:stage-regret}, the regret bounds within any single stage $s$ are:

\begin{itemize}[leftmargin=*]
    \item \textbf{In the Adversarial Setting: }$
    \mathcal{R}^s (\mu^c) \leq O(\log T) \quad \text{and} \quad
    \max _{\mu \in \Delta_A} \mathcal{R}^s (\mu) \leq O(\sqrt{T \log A} \log T).
    $
    
    \item \textbf{In the Stochastic Setting: }$
    \mathcal{R}^s (\mu^c) \leq O(\log T) \quad \text{and} \quad
    \max _{\mu \in \Delta_A} \mathcal{R}^s (\mu) \leq O\left( \frac{\log A \cdot \log T}{\Delta} \right).
    $\end{itemize}

Since there are at most $S = O(\log T)$ stages, the total regret over $T$ rounds is obtained by summing these per-stage bounds.
For the comparator regret in both settings:
\[
\mathcal{R}(\mu^c) = \sum_{s=1}^S \mathcal{R}^s(\mu^c) \le S \cdot O(\log T) = O(\log^2 T).
\]
For the worst-case pseudo-regret in the adversarial setting:
\[
\mathcal{R}_{\mathrm{pseudo}}(T) \le \sum_{s=1}^S O(\sqrt{T \log A} \log T) = O(\sqrt{T \log A} \log^2 T).
\]
Similarly, for the stochastic setting:
\[
\mathcal{R}_{\mathrm{pseudo}}(T) \le \sum_{s=1}^S O\left( \frac{\log A \cdot \log T}{\Delta} \right) = O\left( \frac{\log A \cdot \log^2 T}{\Delta} \right).
\]
This concludes the proof.
\end{proof}

\newpage
\section{Extension to More General Adaptive Adversaries}
\label{app:adaptive-adversaries}

In this appendix we record a simple extension of Assumption~\ref{eq:A1}.
Recall that Assumption~\ref{eq:A1} controls the cumulative adaptivity of the
environment over every interval $I$ through a bound of the form
$E_I \le C_E \sqrt{|I|}$. More generally, one can ask what happens if the
environment is allowed to react more strongly to the learner's past actions,
namely if
\[
    E_I \le C_E |I|^{\gamma},
    \qquad \gamma \in (0,1).
\]
The analysis changes only through the pseudo-regret/expected-regret gap in
Lemma~\ref{lem:pseudo-gap}. Under the above condition, Lemma~\ref{lem:pseudo-gap}
contributes an additional adaptivity term of order $O(|I|^\gamma)$.
Consequently, the normal-phase bound in Lemma~\ref{app:lem:normal-phase} becomes
\[
    \mathcal{R}^k(\mu)
    \le
    \sqrt{T\log A}
    +
    O\!\left(T^\gamma+\sqrt{T}\right).
\]
This yields two regimes.

First, if $\gamma \in (0,1/2]$, then the adaptivity term is dominated by the
usual martingale fluctuation term:
\[
    T^\gamma \le \sqrt{T}.
\]
Hence the normal-phase bound remains
\[
    \mathcal{R}^k(\mu)
    \le
    \sqrt{T\log A}
    +
    O(\sqrt{T}),
\]
and the subsequent stage-level and global regret bounds are unchanged.

Second, if $\gamma \in (1/2,1)$, then the adaptivity term dominates the
$\sqrt{T}$ fluctuation term. In this case Lemma~\ref{lem:stage-regret} becomes
\[
    \max_{\mu}\mathcal{R}^s(\mu)
    =
    O\! \left(\left(T^\gamma + \sqrt{T\log A}\right)\log T \right)
\]
in the adversarial regime. Summing over the $O(\log T)$ stages gives the
corresponding global pseudo-regret bound
\[
    \mathcal{R}_{\mathrm{pseudo}}(T)
    \le
    O\!\left(\left(T^\gamma + \sqrt{T\log A} \right)\log^2 T \right).
\]
Thus, Theorem~\ref{thm:main} remains unchanged for
$\gamma\in(0,1/2]$, while for $\gamma\in(1/2,1)$ the adversarial guarantee
degrades gracefully from the minimax $\widetilde{O}(\sqrt{T})$ scale to the
sublinear scale $\widetilde{O}(T^\gamma)$.

This extension shows that \textsc{Compass-Hedge} is not tied to the exact
square-root adaptivity condition: it remains no-regret for any
$\gamma<1$. However, the threshold $\gamma=1/2$ marks the point at which the
environment's cumulative adaptivity begins to dominate the intrinsic
martingale fluctuations. Going beyond this regime---for instance, obtaining
minimax-rate guarantees under stronger forms of adaptive response, or
identifying alternative stability conditions that permit fully adaptive
adversaries without sacrificing safety---is an interesting direction for future
work.
\newpage
\section{Experimental Evaluation}
\label{sec:experiments}

We complement our theoretical guarantees with experiments designed to test whether
\textsc{Compass-Hedge} achieves the intended universality: robustness in adversarial
environments while retaining strong performance in benign or stochastic regimes. Our
evaluation has two parts. First, we consider a clean adversarial benchmark based on a
zero-sum game construction in which a strategic optimizer induces large regret against
standard Hedge-type dynamics. Second, we evaluate the algorithms on S\&P 500 market
data from 2015--2024, a setting that combines long periods of statistical regularity
with abrupt regime shifts.

\paragraph{Algorithms.}
We compare four methods: the fixed comparator baseline, standard Hedge,
Anytime-Hedge, and \textsc{Compass-Hedge}. In the adversarial experiment, we vary the
comparator-quality parameter $\epsilon$, which interpolates between the exact minimax comparator and the uniform comparator. In the financial experiment, the benchmark is evaluated against two
qualitatively different comparators: the ex-post best single asset and the broad-market
SPY comparator.

\subsection{Adversarial Robustness}

We first test the behavior of \textsc{Compass-Hedge} in an adversarial zero-sum game
environment. The loss sequence is generated by a strategic optimizer that alternates
between carefully chosen regimes, in the spirit of adversarial constructions that force
large regret against standard learning dynamics. This setting is intentionally hostile:
the learner cannot rely on stationarity or a persistent stochastic gap, and therefore
any meaningful guarantee must come from adversarial robustness rather than statistical
regularity.

Figure~\ref{fig:adversarial} reports the resulting regret behavior. Safe comparator is \(
q_\varepsilon \;=\; (1-\varepsilon)\, q_{\text{eq}} \;+\; \varepsilon\, u\), where $u$ is the uniform distribution. At $\varepsilon=0$ the comparator is the true minimax strategy; at $\varepsilon=1$ it is uninformative.  Panel~(a) plots
average per-round regret over time. Across all values of $\varepsilon$, the regret of
\textsc{Compass-Hedge} decays toward zero, suggesting asymptotic no-regret behavior.
This contrasts with bad fixed comparators, whose average regret remains bounded away
from zero. Thus, even though the adversarial construction is designed to exploit
predictable learning behavior, \textsc{Compass-Hedge} avoids linear regret.

Panel~(b) provides a finer view of the rate by plotting cumulative regret normalized
by $\sqrt{t}$. The curves remain approximately stable rather than increasing with
time, supporting the interpretation that cumulative regret scales on the order of
$\sqrt{t}$. Importantly, the price paid for robustness appears mainly as a constant
factor: changing $\varepsilon$ affects the vertical level of the curves, but not the
qualitative $\sqrt{t}$ scaling. In the well-specified comparator case
$(\varepsilon=0)$, the regret is substantially smaller and decays more rapidly,
consistent with the intuition that the algorithm adapts quickly when the comparator is
aligned with the best response structure.

\begin{figure}[t]
    \centering
    \begin{subfigure}[t]{0.49\textwidth}
        \centering
        \includegraphics[width=\linewidth]{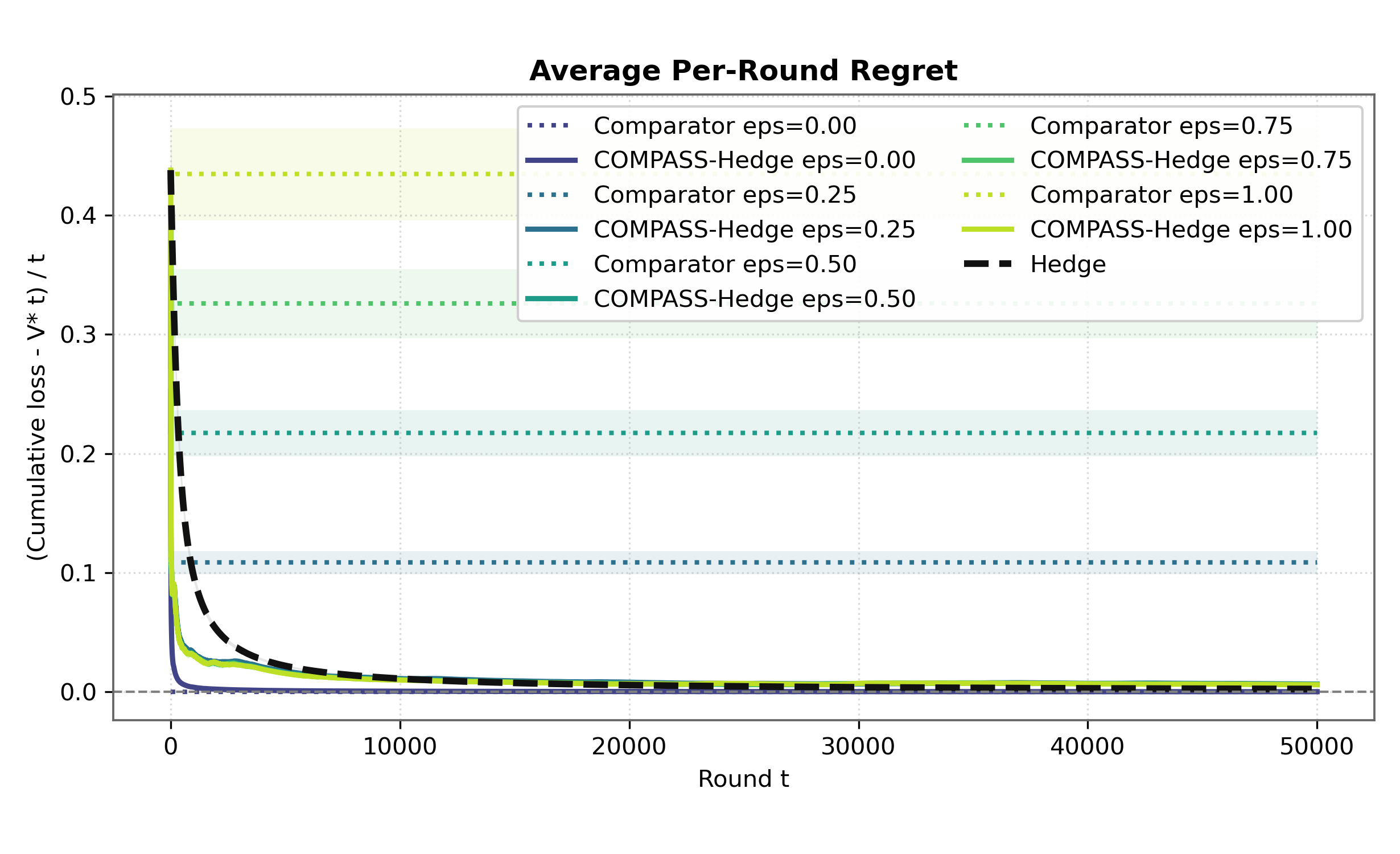}
        \caption{Average per-round regret over time.}
    \end{subfigure}
    \hfill
    \begin{subfigure}[t]{0.49\textwidth}
        \centering
        \includegraphics[width=\linewidth]{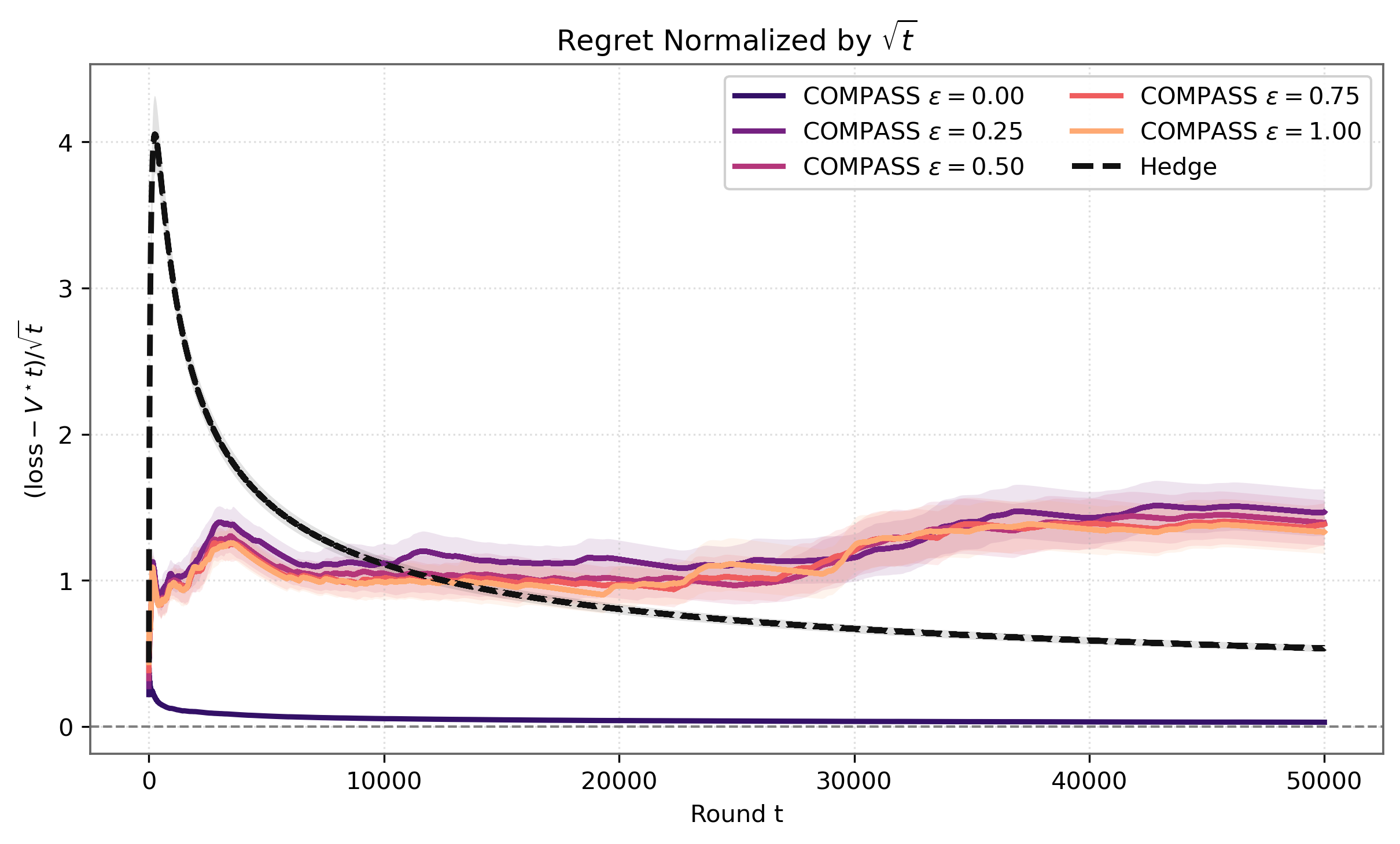}
        \caption{Regret normalized by $\sqrt{t}$ over time.}
    \end{subfigure}
    \caption{\textbf{Adversarial robustness.}
    Panel~(a) shows that the average regret of \textsc{Compass-Hedge} vanishes over
    time, while bad fixed comparators suffer constant average regret. Panel~(b) shows
    that regret normalized by $\sqrt{t}$ remains stable, indicating adversarial
    $O(\sqrt{T})$ regret up to constant factors. Shaded regions indicate pointwise
    90\% confidence intervals across trials.}
    \label{fig:adversarial}
\end{figure}

\subsection{Stochastic Adaptivity on Market Data}

We next evaluate whether the same method remains competitive in a more benign,
data-driven environment. We use daily S\&P 500 stock data from 2015--2024. This
setting is not purely stochastic: market data exhibit persistent trends, volatility
clusters, and occasional adversarial-looking regime shifts. At the same time, it
contains exploitable structure, making it a natural testbed for adaptivity. Following
the online portfolio-learning interpretation, each asset is treated as an arm, and the
losses are induced by realized returns.

The key question is whether a single algorithm can perform well relative to two very
different benchmarks. The first benchmark is the ex-post best single asset under the transformed loss sequence, which in
this period is NVIDIA (NVDA).  Competing with this comparator requires aggressive
adaptivity, since the best arm is identified only in hindsight and dominates over a
long horizon. The second benchmark is SPY, a broad-market comparator. Competing with
SPY requires safety: an algorithm that chases the best individual asset too strongly
can easily underperform the diversified market benchmark during unstable periods.

Figure~\ref{fig:stochastic} shows that \textsc{Compass-Hedge} is the only adaptive method that remains close to the SPY comparator while improving over the unprotected Hedge variants against NVDA. In Panel~(a), standard Hedge
and Anytime-Hedge accumulate large regret relative to NVDA, while
\textsc{Compass-Hedge} tracks the best-arm benchmark much more closely. The zoomed
inset highlights that this improvement persists even near the end of the horizon,
rather than being driven only by early transient behavior. In Panel~(b), the difference
is more pronounced: relative to SPY, Hedge and Anytime-Hedge incur substantial
positive regret, whereas \textsc{Compass-Hedge} remains close to the market comparator
throughout the entire period.

This experiment illustrates the central empirical message of the paper. Algorithms
optimized only for adversarial regret may be too conservative or may fail to exploit
persistent stochastic structure. Conversely, algorithms tuned to a favorable
comparator may lose robustness when the environment shifts. \textsc{Compass-Hedge}
balances these objectives: it remains close to the strong ex-post asset comparator
while simultaneously preserving safety against the market baseline.

\begin{figure}[t]
    \centering

    \begin{subfigure}{0.48\textwidth}
        \centering
        \includegraphics[width=\textwidth]{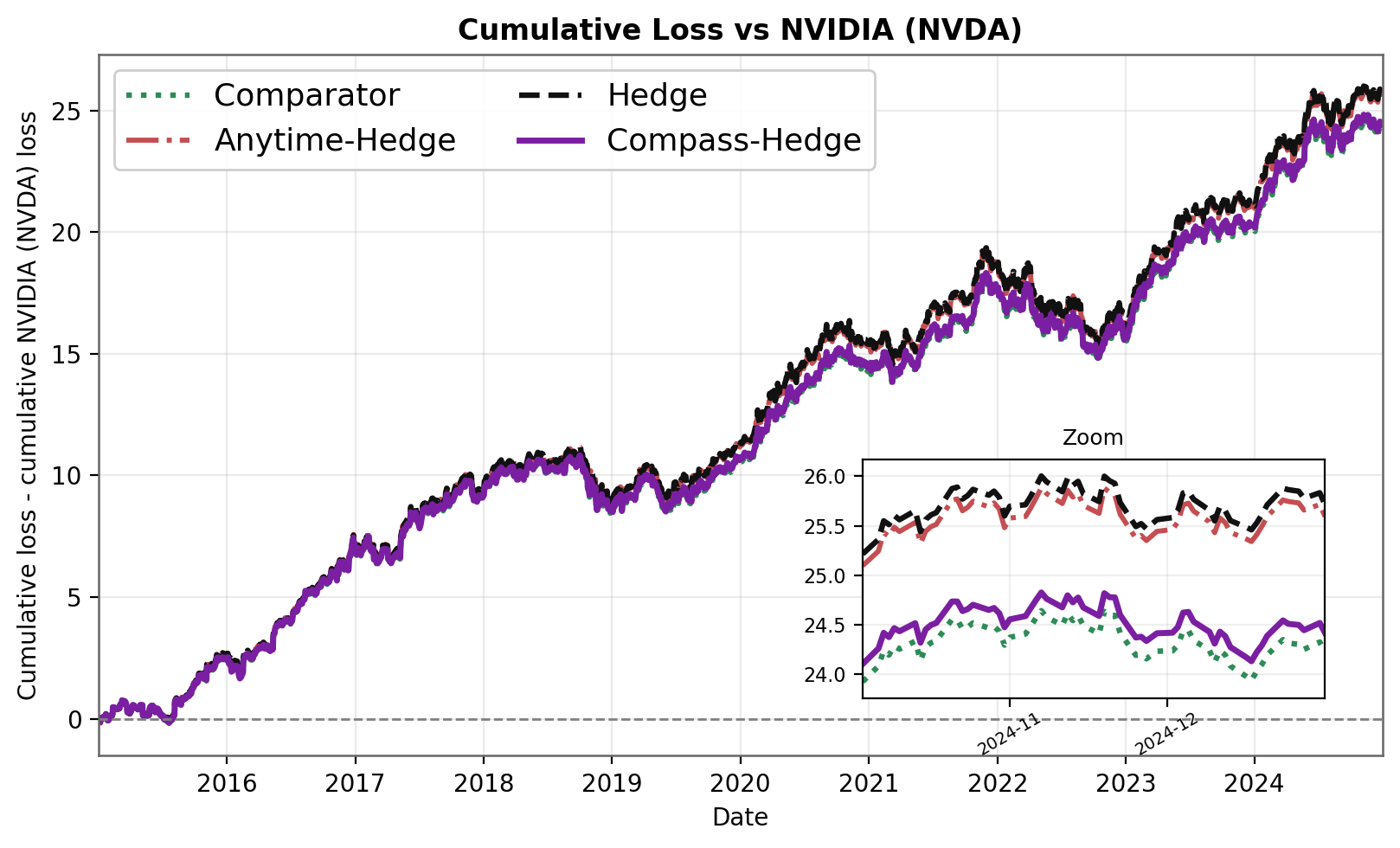}
        \caption{}
    \end{subfigure}
    \hfill
    \begin{subfigure}{0.48\textwidth}
        \centering
        \includegraphics[width=\textwidth]{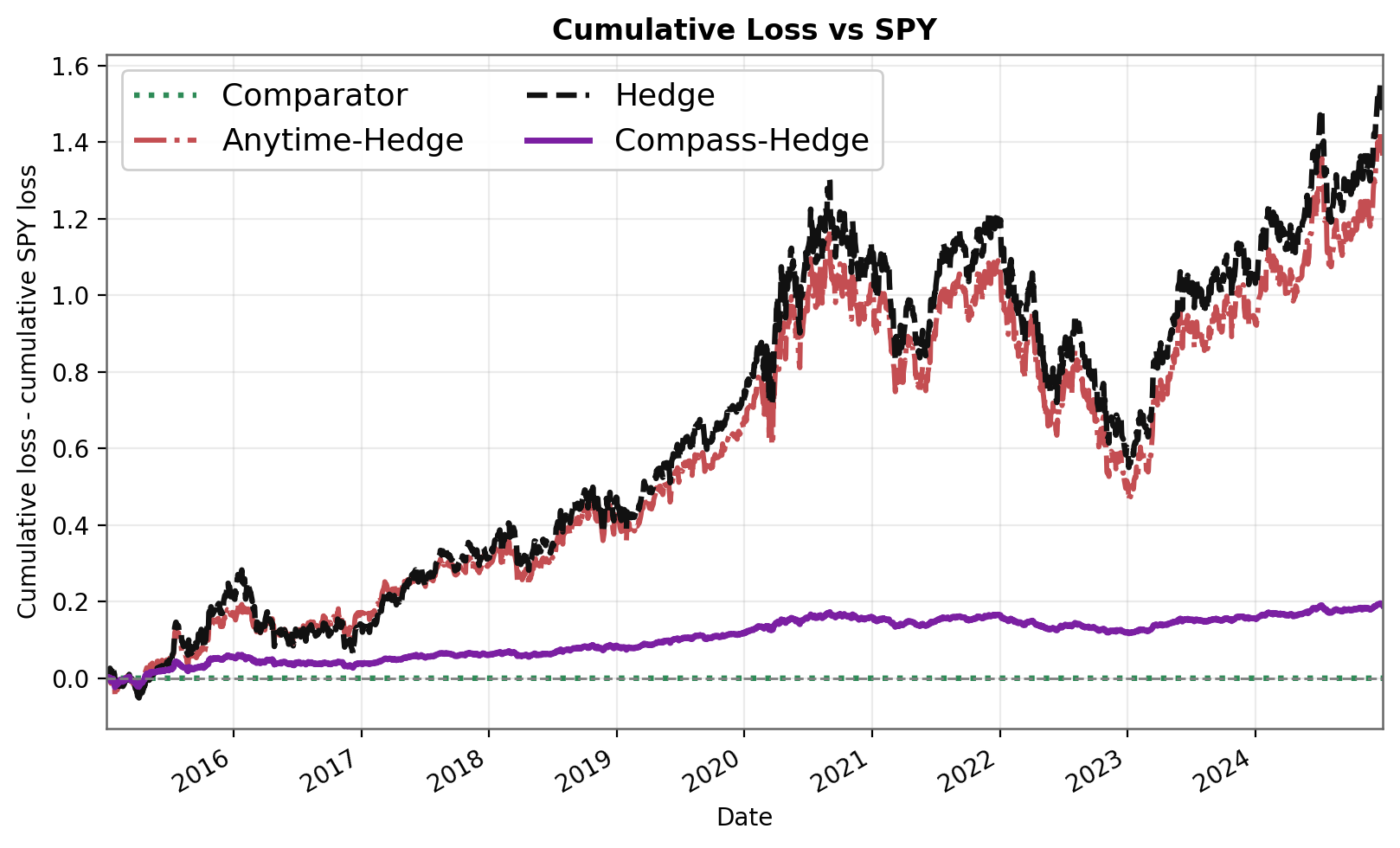}
        \caption{}
    \end{subfigure}

    \caption{\textbf{S\&P 500 Hedge family comparison, 2015--2024.}
    Panel~(a) plots cumulative loss minus cumulative NVIDIA loss, i.e.,
    regret against the \emph{ex-post} best single arm. Panel~(b) plots cumulative
    loss minus cumulative SPY loss, i.e., regret against the market comparator.
    \textsc{Compass-Hedge} is the only method that remains competitive with both the
    aggressive best-arm benchmark and the safer market benchmark.}
    \label{fig:stochastic}
\end{figure}

\paragraph{Takeaway.}
Together, the two experiments support the universality claim behind
\textsc{Compass-Hedge}. In the adversarial construction, the algorithm retains
sublinear regret with the expected $O(\sqrt{T})$ scaling. On real financial data, it
also adapts to favorable structure and avoids the failure modes of methods that are
competitive with only one benchmark. Thus, the empirical results mirror the theoretical
motivation: \textsc{Compass-Hedge} reconciles adversarial robustness, comparator
safety, and stochastic adaptivity within a single algorithmic design.
\newpage
\section{Extended experimental evaluation}
\label{app:experiments}

This appendix substantiates the empirical claims of Section~\ref{sec:experiments} by walking through the full experimental pipeline: data, preprocessing, simulation protocol, and a sequence of increasingly stringent diagnostic experiments. Our aim is twofold. First, we want the reader to be able to reproduce every number and every figure. Second, we want to show that the behavior of \textsc{Compass-Hedge} is not an artifact of a particular dataset or a particular adversary, but emerges consistently across heterogeneous environments---from real equity markets to synthetic zero-sum games designed to probe the algorithm's failure modes.

\subsection{A synthetic adversary: minimax environments as a controlled stress test}
\label{app:minimax_setup}

Real data is informative but noisy; it cannot by itself rule out the possibility that \textsc{Compass-Hedge} is getting lucky on average while quietly misbehaving on worst-case instances. To close this loophole, we complement the equity experiments with a synthetic adversary drawn from the zero-sum game pipeline, where the ``right answer'' is known in closed form. This synthetic data set is used for Figure~\ref{fig:adversarial} and following experiments.

\paragraph{Construction.}
In the current implementation, we fix the random seed to $0$ and draw a payoff matrix
\[
M \in [-1,1]^{m \times n}, \qquad M_{ij} \sim \mathrm{Unif}[-1,1],
\]
with $m = \MinimaxRows$ row actions and $n = \MinimaxArms$ learner arms. Solving the row maxmin and column minmax linear programs yields an equilibrium pair $(p_{\text{eq}}, q_{\text{eq}})$ and the game value $V^\star$; the code uses `scipy.optimize.linprog` with the HiGHS backend.

The row player does not play $p_{\text{eq}}$ directly. Instead, letting
\[
\ell = \arg\max_i (p_{\text{eq}})_i,
\]
we construct two neighboring distributions
\[
x' = \bigl(1-(p_{\text{eq}})_\ell\bigr)p_{\text{eq}} + (p_{\text{eq}})_\ell e_\ell,
\qquad
x'' = \bigl(1+(p_{\text{eq}})_\ell\bigr)p_{\text{eq}} - (p_{\text{eq}})_\ell e_\ell,
\]
where $e_\ell$ is the $\ell$th standard basis vector. Thus the perturbation magnitude is determined by the equilibrium mass $(p_{\text{eq}})_\ell$ itself: there is no separate perturbation hyperparameter and no explicit projection step.

The environment then alternates deterministically between these two distributions,
\[
x_t =
\begin{cases}
x', & t \text{ odd},\\
x'', & t \text{ even},
\end{cases}
\]
for $T = \MinimaxHorizon$ rounds. The learner therefore sees a structured non-stationary loss stream
\[
c_t = M^\top x_t, \qquad x_t \in \{x', x''\},
\]
with no additional loss rescaling; equivalently, the learner's score vector is updated using $-c_t$. This construction is deliberately harsh: the adversary is never stationary, yet the optimal strategy and benchmark value are both analytically available. For the synthetic experiment, we set the coefficient in front of $\widehat{R}_s$ in Algorihtm~\ref{alg:doubling_ftrl_comparator} (\texttt{line 11}) to 0.1.

\paragraph{A family of comparators.}
To isolate the role of prior quality, we sweep a one-parameter family interpolating between the exact equilibrium and total agnosticism:
\[
q_\varepsilon \;=\; (1-\varepsilon)\, q_{\text{eq}} \;+\; \varepsilon\, u, \qquad \varepsilon \in \{\MinimaxEpsilons\},
\]
where $u$ is the uniform distribution. At $\varepsilon=0$ the comparator is the true minimax strategy; at $\varepsilon=1$ it is uninformative. Every curve reported below aggregates $ \MinimaxTrials$ independent trials with pointwise 90\% confidence bands.

\begin{figure}[t]
    \centering
    \begin{subfigure}[t]{0.49\textwidth}
        \centering
        \includegraphics[width=\linewidth]{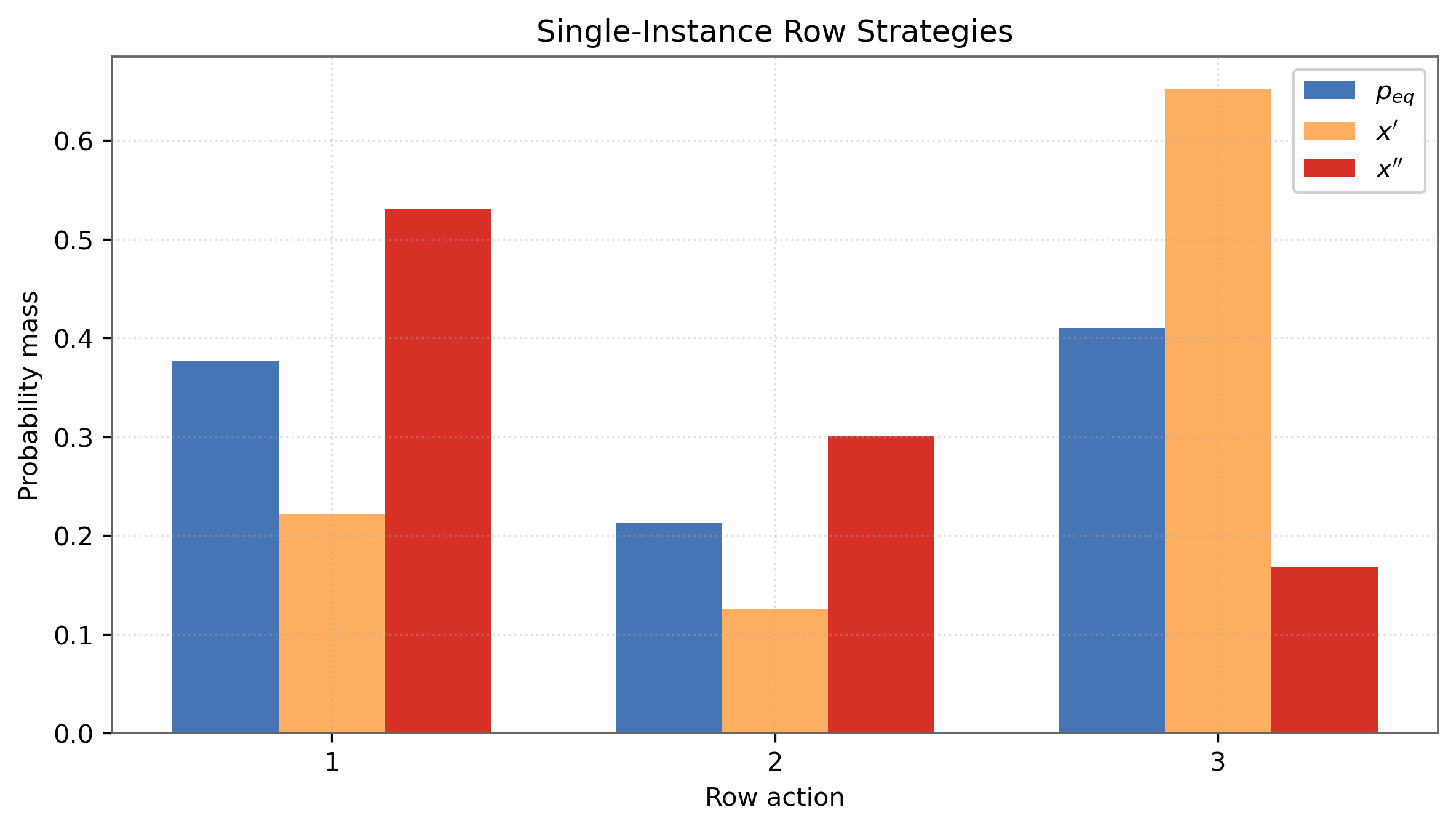}
        \caption{Equilibrium row strategy and its two perturbations}
    \end{subfigure}
    \hfill
    \begin{subfigure}[t]{0.49\textwidth}
        \centering
        \includegraphics[width=\linewidth]{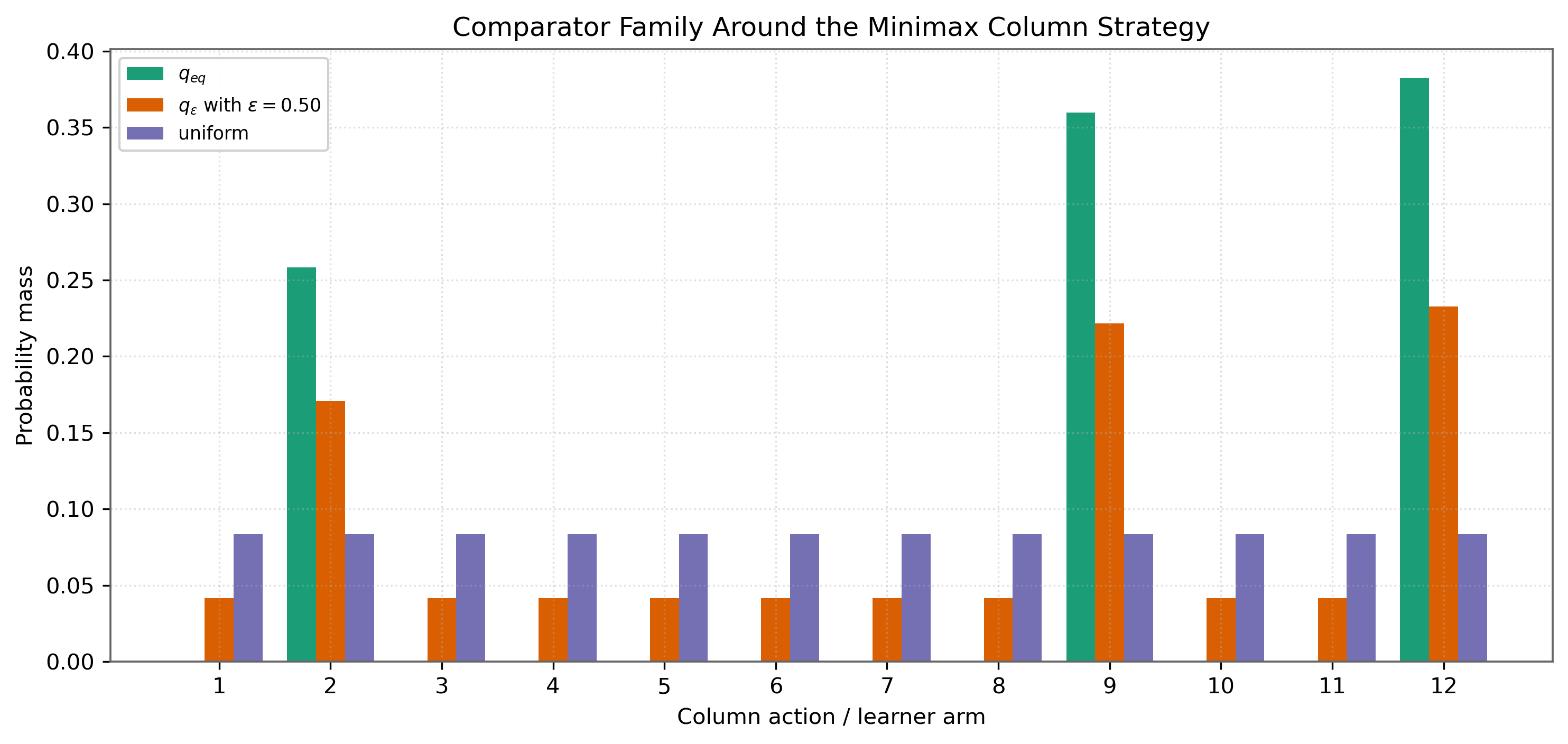}
        \caption{Comparator family, interpolating $q_{\text{eq}} \to u$}
    \end{subfigure}
    \vspace{0.75em}
    \begin{subfigure}[t]{0.8\textwidth}
        \centering
        \includegraphics[width=\linewidth]{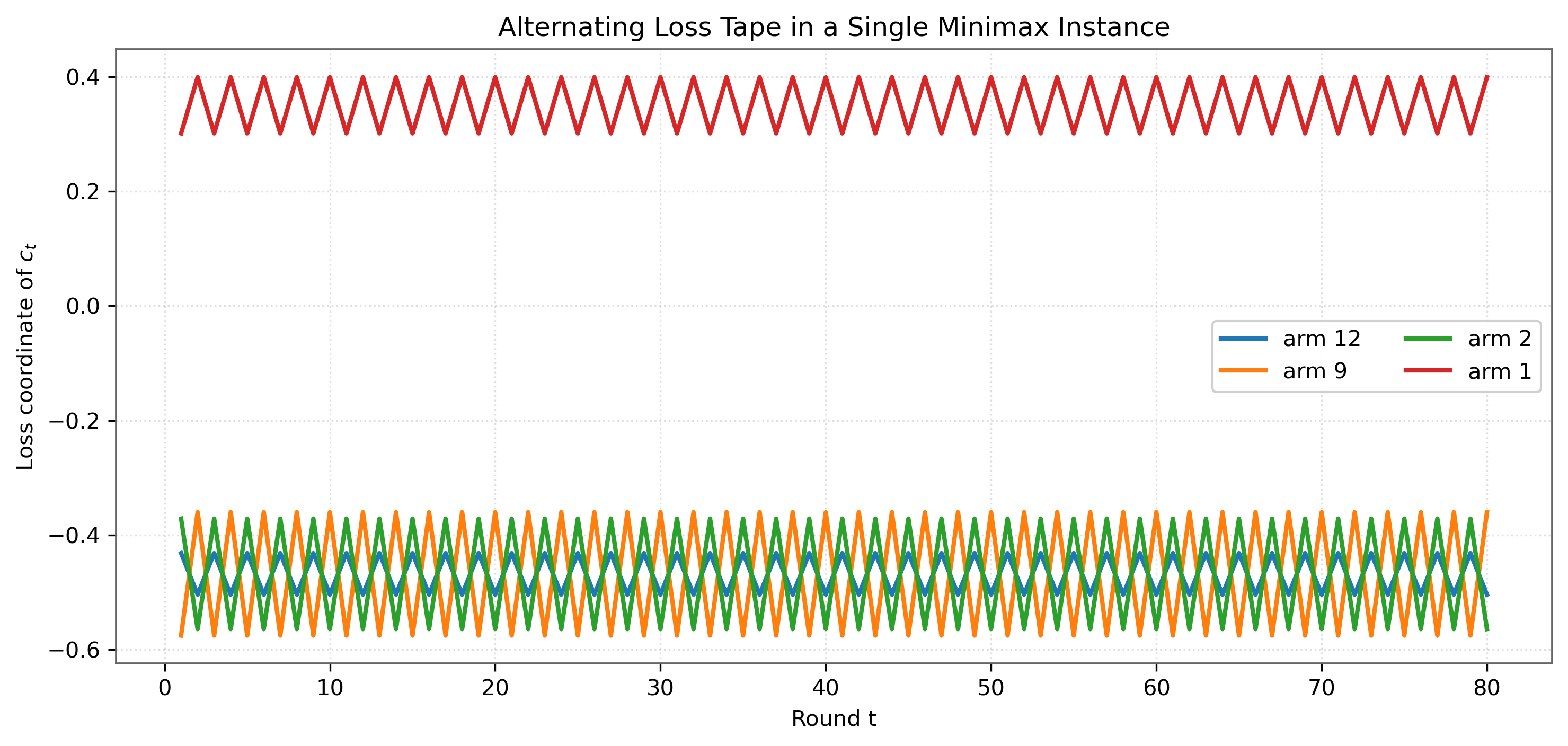}
        \caption{A single-instance view of the alternating loss tape}
    \end{subfigure}
    \caption{\textbf{Anatomy of the synthetic adversary.} The alternating pair $(x', x'')$ produces a structured non-stationary stream anchored to the equilibrium geometry of the game. The sweep over $\varepsilon$ lets us dial comparator quality from ``perfect'' to ``uninformative'' while keeping the underlying environment fixed.}
    \label{fig:minimax-setup}
\end{figure}

Figure~\ref{fig:minimax-setup} makes the construction concrete. The key point is that the safe side information is not arbitrary: it interpolates between the exact minimax strategy and a fully agnostic uniform baseline, so every degradation we observe below can be attributed unambiguously to comparator quality.

\subsubsection{Experiment I: how gracefully does safety degrade?}
\label{app:exp1}

The first experiment asks the most basic question one can ask about a safeguarded learner: \emph{what happens as the safeguard becomes unreliable?} Because $q_{\text{eq}}$ is known exactly, comparator quality is the only moving part.

\begin{figure}[t]
    \centering
    \begin{subfigure}[t]{0.49\textwidth}
        \centering
        \includegraphics[width=\linewidth]{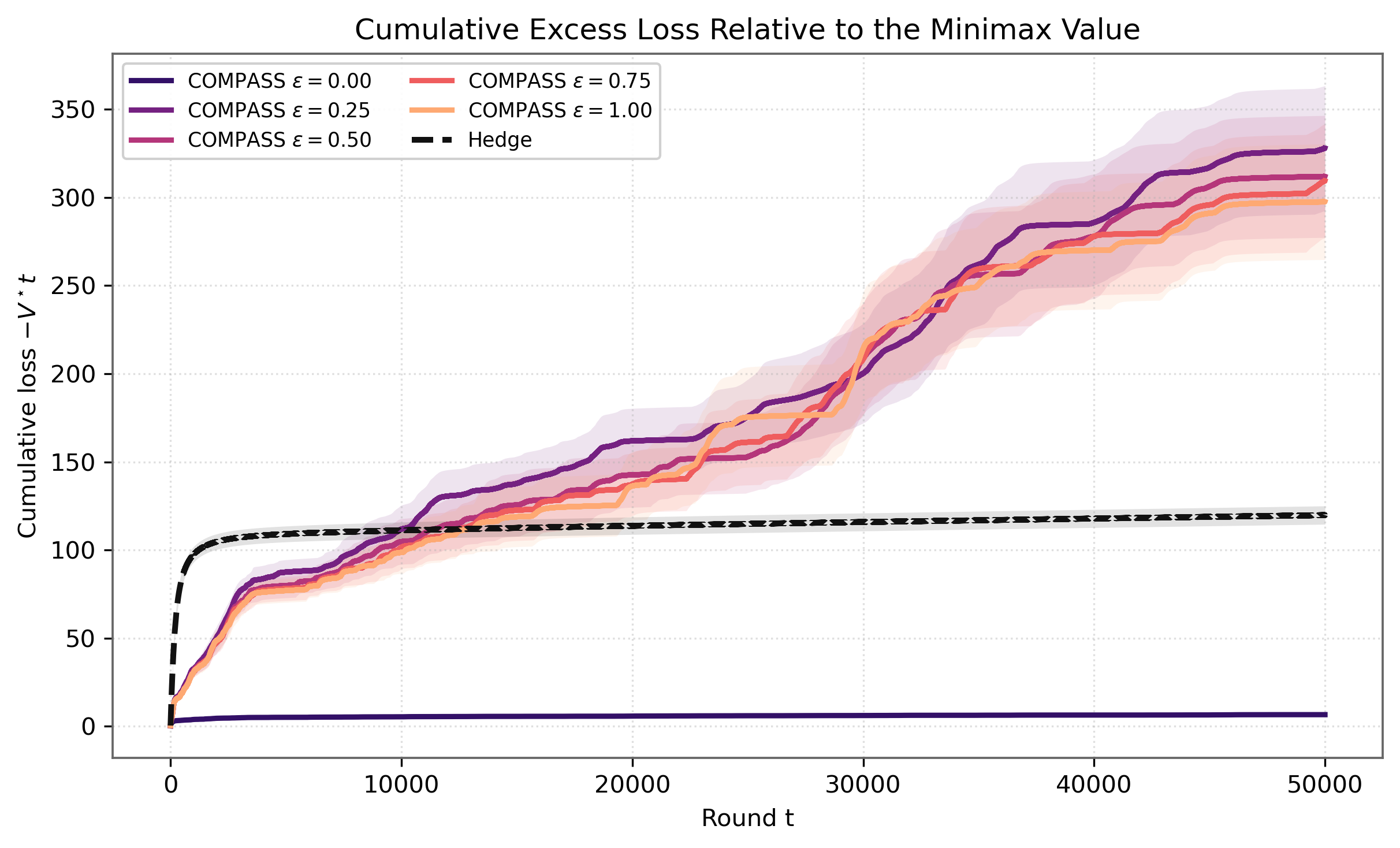}
        \caption{Cumulative excess loss relative to $V^\star t$}
    \end{subfigure}
    \hfill
    \begin{subfigure}[t]{0.49\textwidth}
        \centering
        \includegraphics[width=\linewidth]{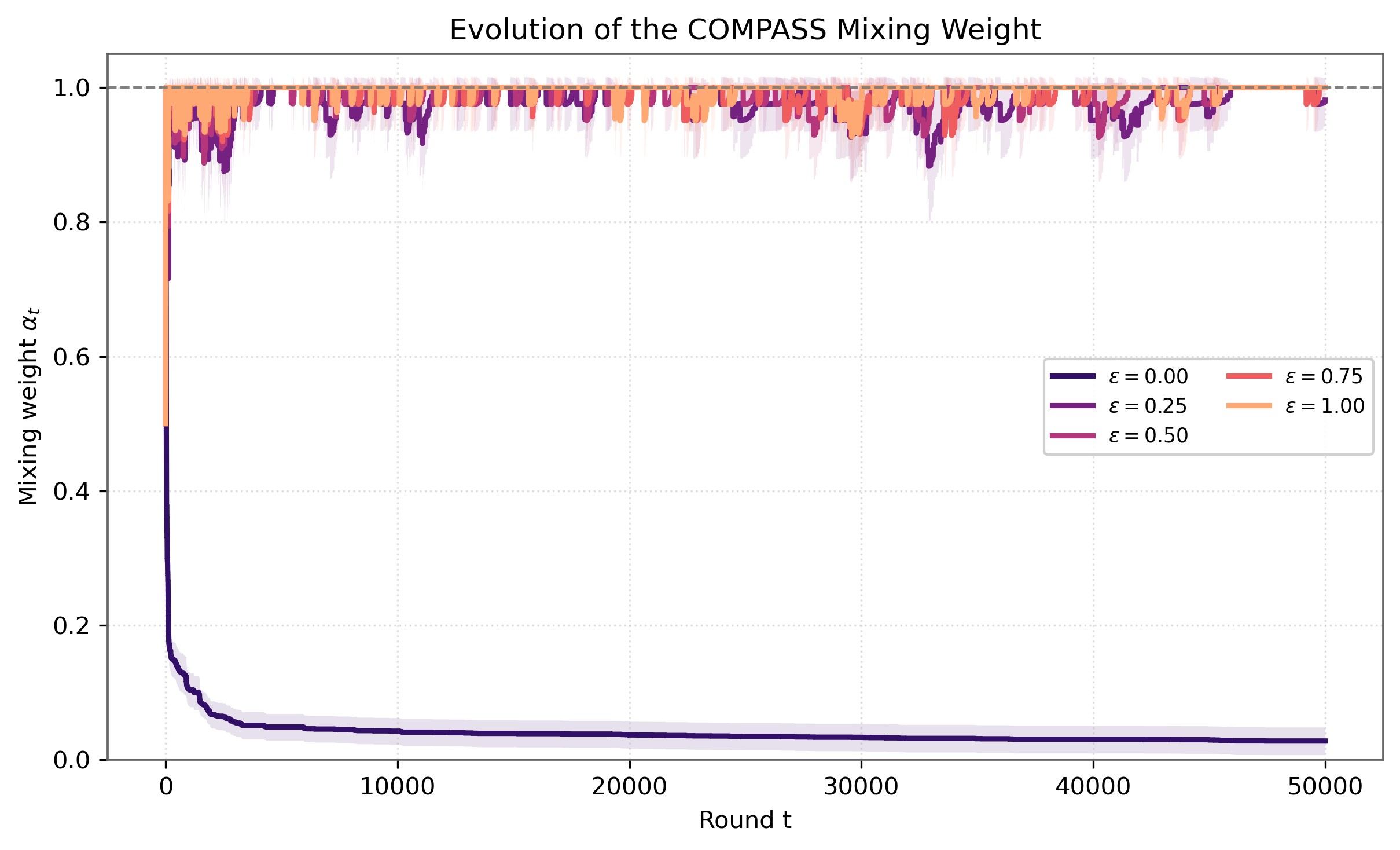}
        \caption{Time-varying mixing weight $\alpha_t$}
    \end{subfigure}
    \caption{Time-series behavior in the comparator-quality sweep.}
    \label{fig:minimax-sweep-timeseries}
\end{figure}

\begin{figure}[t]
    \centering
    \begin{subfigure}[t]{0.49\textwidth}
        \centering
        \includegraphics[width=\linewidth]{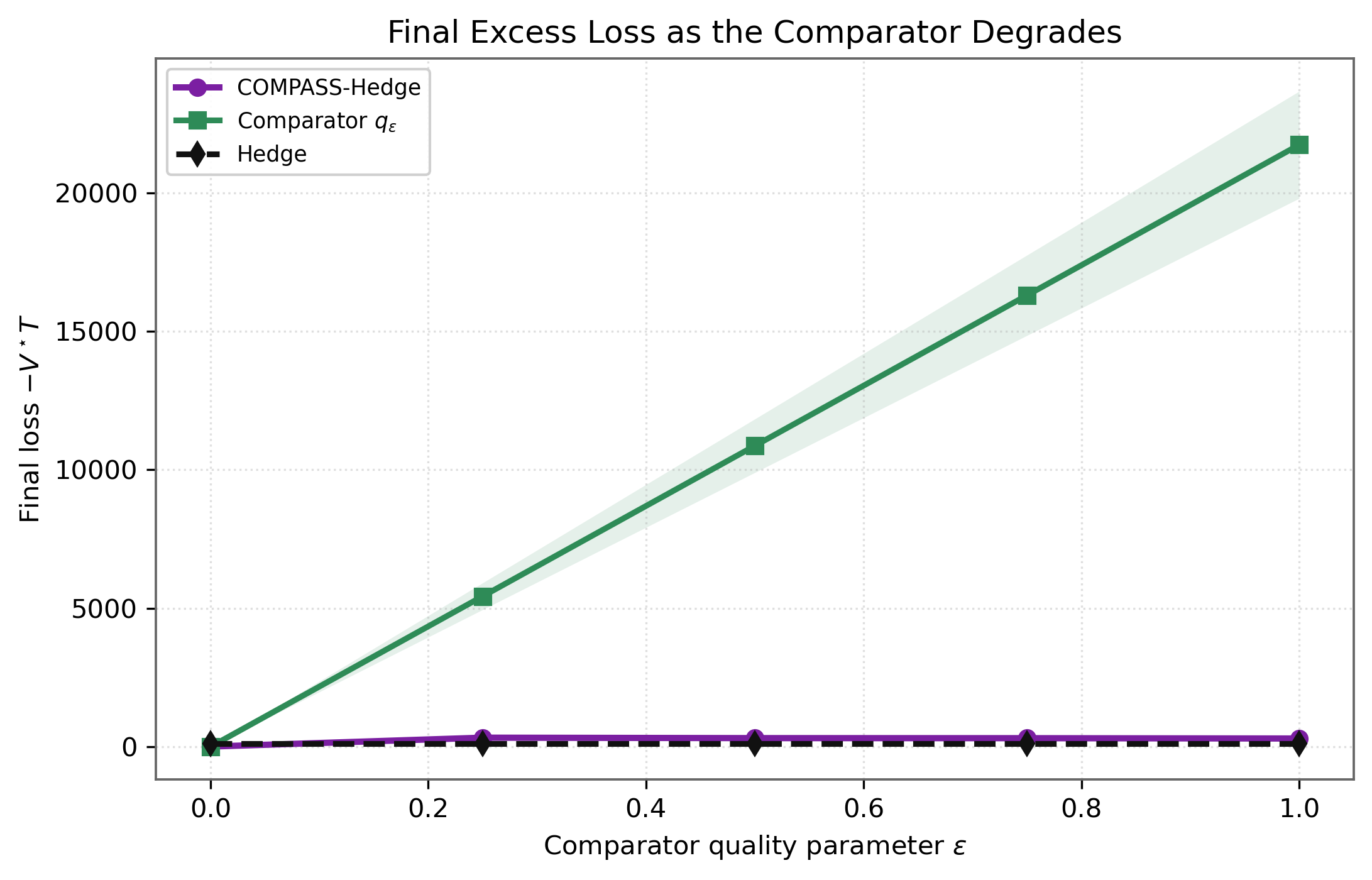}
        \caption{Final excess loss as a function of $\varepsilon$}
    \end{subfigure}
    \hfill
    \begin{subfigure}[t]{0.49\textwidth}
        \centering
        \includegraphics[width=\linewidth]{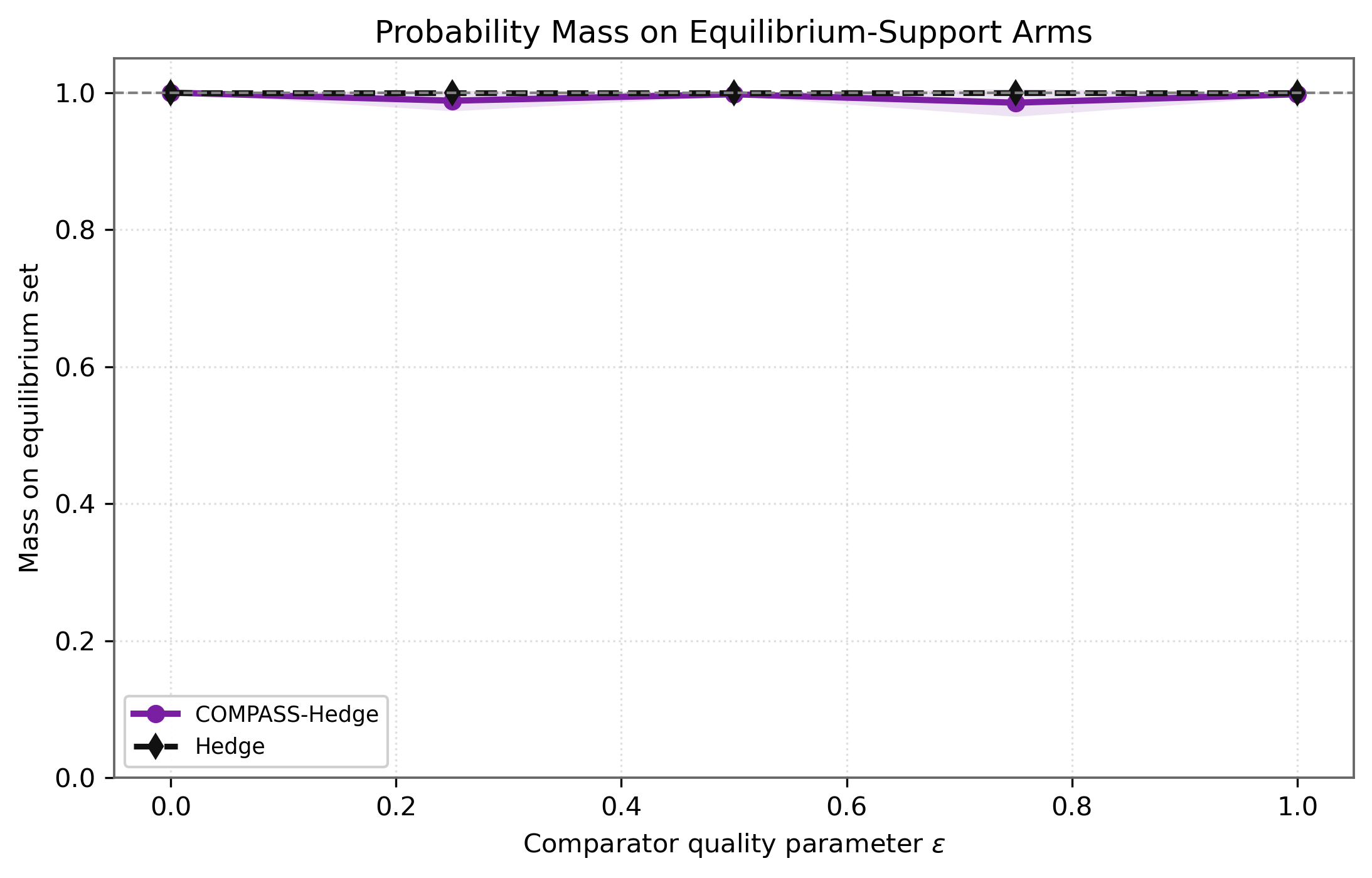}
        \caption{Terminal mass on equilibrium-support arms}
    \end{subfigure}
    \caption{Terminal statistics from the comparator-quality sweep.}
    \label{fig:minimax-sweep-terminal}
\end{figure}

Three patterns emerge consistently across Figures~\ref{fig:minimax-sweep-timeseries} and~\ref{fig:minimax-sweep-terminal}.
\begin{enumerate}[leftmargin=1.4em]
    \item \textbf{Good comparators are exploited, slowly.} When $\varepsilon$ is small and the comparator lies near $q_{\text{eq}}$, \textsc{Compass-Hedge} keeps $\alpha_t$ low for an extended window. Visually this looks like ``laziness''; mechanically it is the algorithm recognizing that the prior is trustworthy and that aggression would only add variance.
    \item \textbf{Bad comparators are abandoned, decisively.} As $\varepsilon$ grows, the algorithm ratchets $\alpha_t$ upward and transfers weight to the adaptive Hedge component. The transition is not binary but graded---it responds continuously to how much cumulative evidence has accumulated against the comparator.
    \item \textbf{Failure is graceful, not catastrophic.} The final excess-loss curve rises \emph{smoothly} with $\varepsilon$, and a non-trivial fraction of the learner's mass remains on equilibrium-support arms even when the comparator is nearly uninformative. A weak prior slows the learner down; it does not trap it.
\end{enumerate}
This is exactly the qualitative profile one wants from a safe-aggression mechanism: a strong prior is exploited as a shortcut, a weak prior is overridden, and the transition between the two is continuous rather than discontinuous.

\subsubsection{Experiment II: the right benchmark matters}
\label{app:exp2}

Comparator-aware algorithms are easy to misrepresent visually. Plotted against the comparator, any method that \emph{should} be close to the comparator looks unimpressive by construction. The cleaner benchmark in the minimax environment is the value line $V^\star t$, which is independent of $\varepsilon$.

\begin{figure}[t]
    \centering
    \begin{subfigure}[t]{0.49\textwidth}
        \centering
        \includegraphics[width=\linewidth]{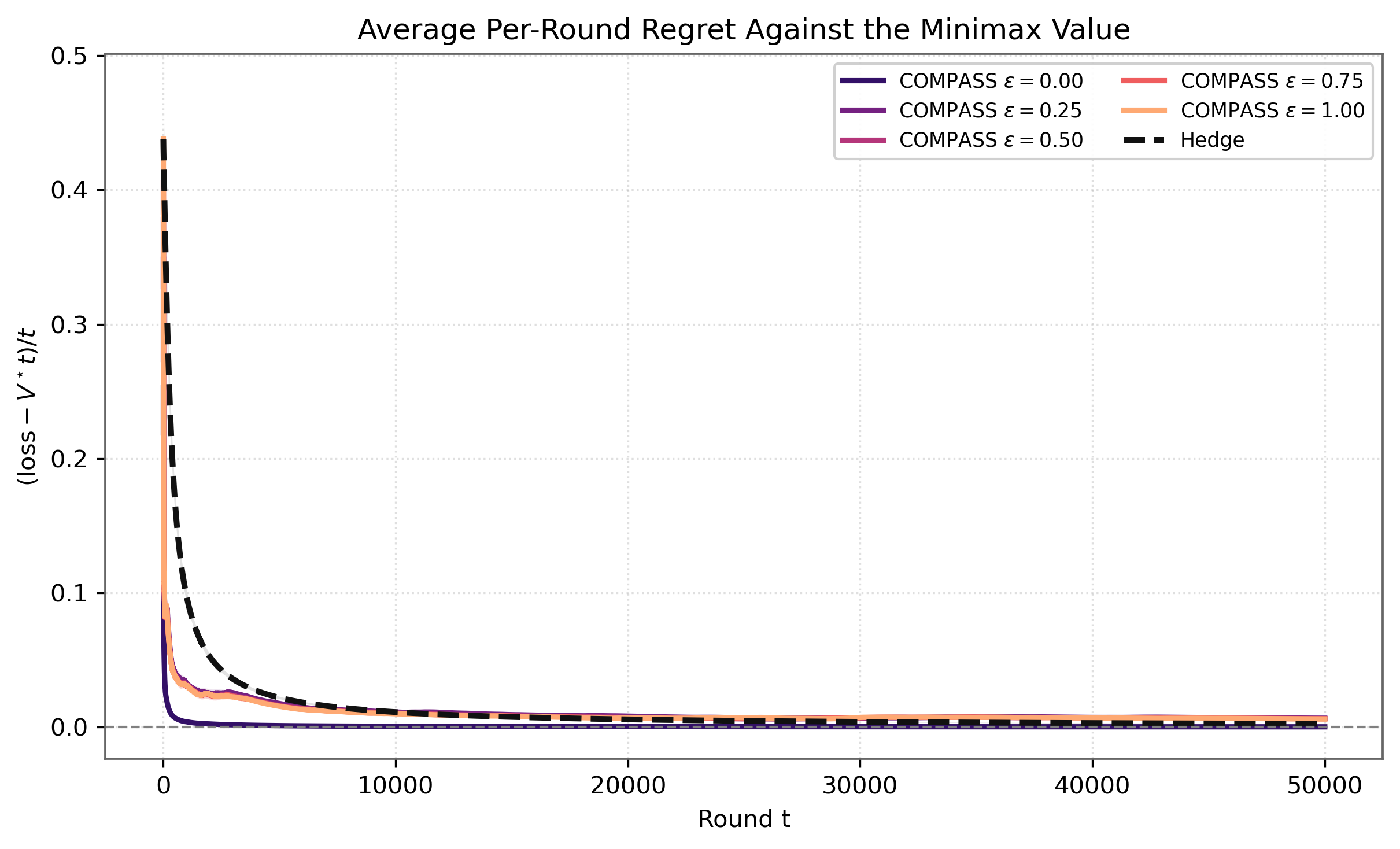}
        \caption{Average per-round regret vs.\ $V^\star$}
    \end{subfigure}
    \hfill
    \begin{subfigure}[t]{0.49\textwidth}
        \centering
        \includegraphics[width=\linewidth]{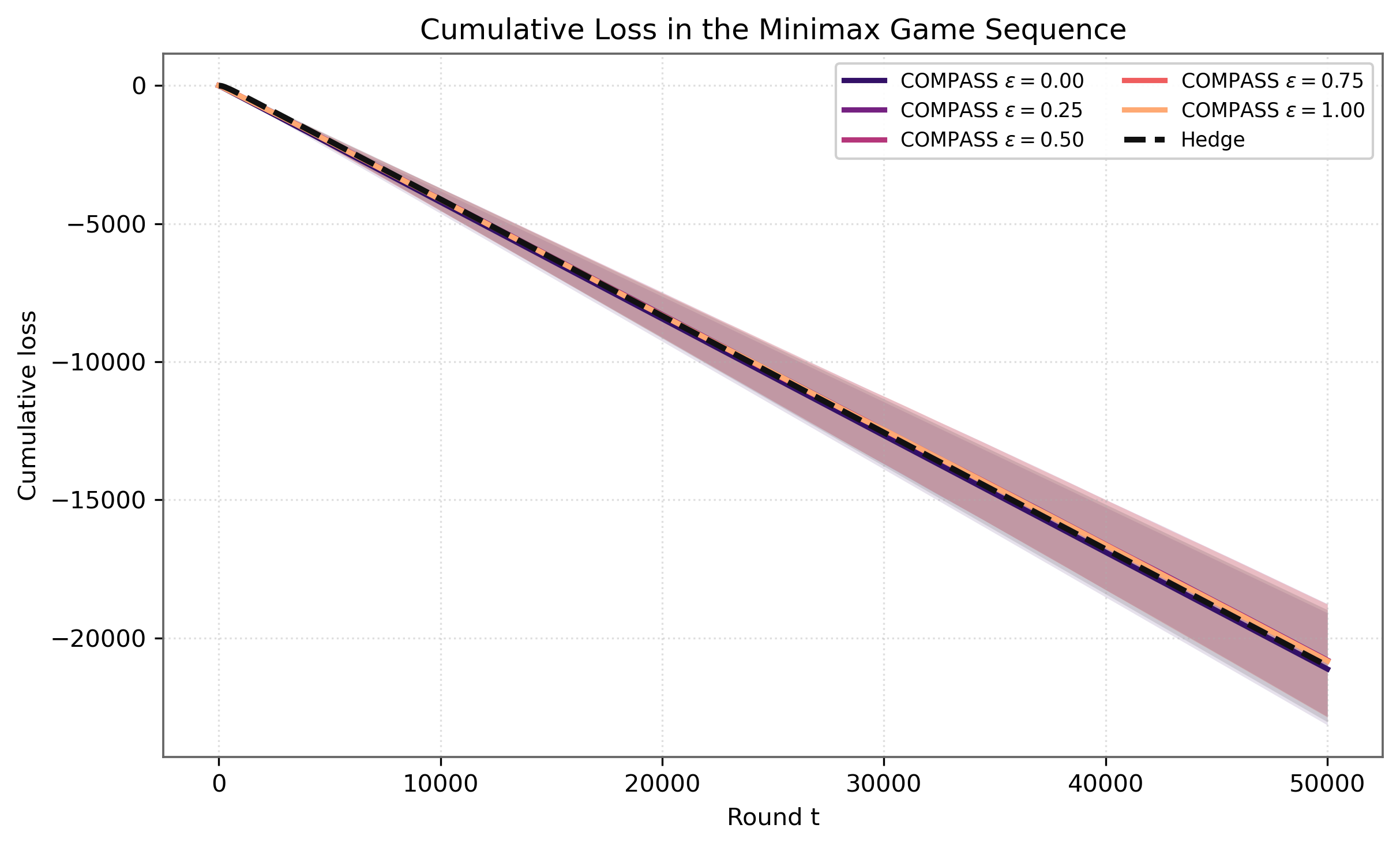}
        \caption{Cumulative loss trajectories}
    \end{subfigure}
    \caption{Regret and loss measured against the game value $V^\star t$.}
    \label{fig:minimax-value-regret}
\end{figure}

Figure~\ref{fig:minimax-value-regret} tells a cleaner story than any comparator-relative plot can. The average-regret curves trend downward---the correct visual signature of sublinear cumulative regret---and the cumulative-loss trajectories remain tightly clustered across comparator choices. Whatever damage a weak comparator inflicts, it is confined to the transient regime; the long-run slope is governed by $V^\star$, not by $q_\varepsilon$.

\subsubsection{Experiment III: is the growth really $\sqrt{T}$?}
\label{app:exp3}

A skeptic could worry that the flat-looking curves in Experiment II hide a slow linear drift. Experiment III is designed to make that worry answerable. We run two diagnostics:
\begin{enumerate}[leftmargin=1.4em]
    \item regret normalized by $\sqrt{t}$, plotted as a time series, and
    \item final regret at horizons $T \in \{\MinimaxScalingHorizons\}$, plotted against $\sqrt{T}$ and against $T$.
\end{enumerate}
The logic is simple: if the true growth were linear in $T$, the normalized time series would drift upward and the final-regret scaling would look linear in $T$ rather than in $\sqrt{T}$.

\begin{figure}[t]
    \centering
    \begin{subfigure}[t]{0.49\textwidth}
        \centering
        \includegraphics[width=\linewidth]{sqrt_regret_over_sqrt_t.png}
        \caption{Regret normalized by $\sqrt{t}$ over time}
    \end{subfigure}
    \hfill
    \begin{subfigure}[t]{0.49\textwidth}
        \centering
        \includegraphics[width=\linewidth]{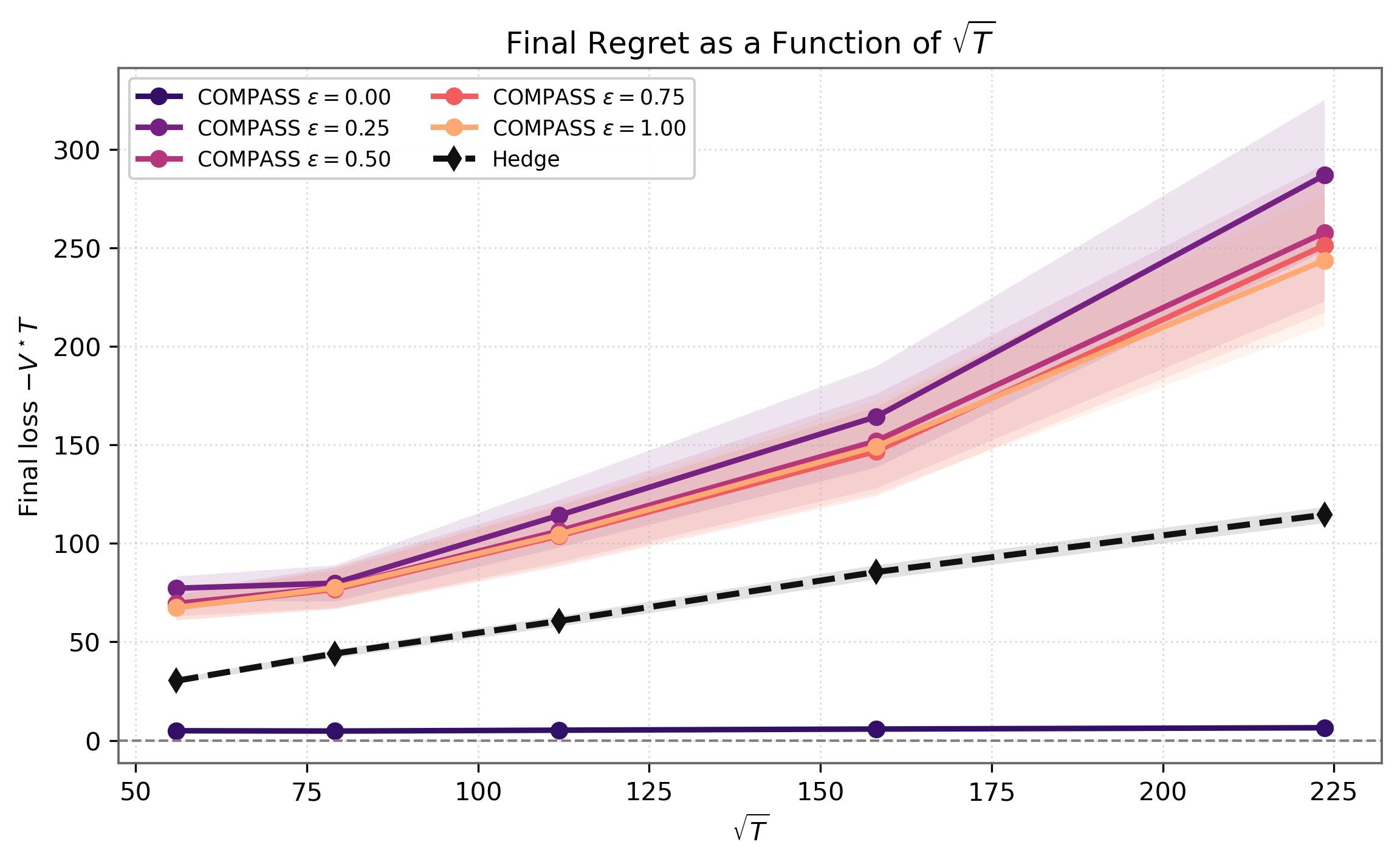}
        \caption{Final regret vs.\ $\sqrt{T}$ across horizons}
    \end{subfigure}
    \vspace{0.75em}
    \begin{subfigure}[t]{0.55\textwidth}
        \centering
        \includegraphics[width=\linewidth]{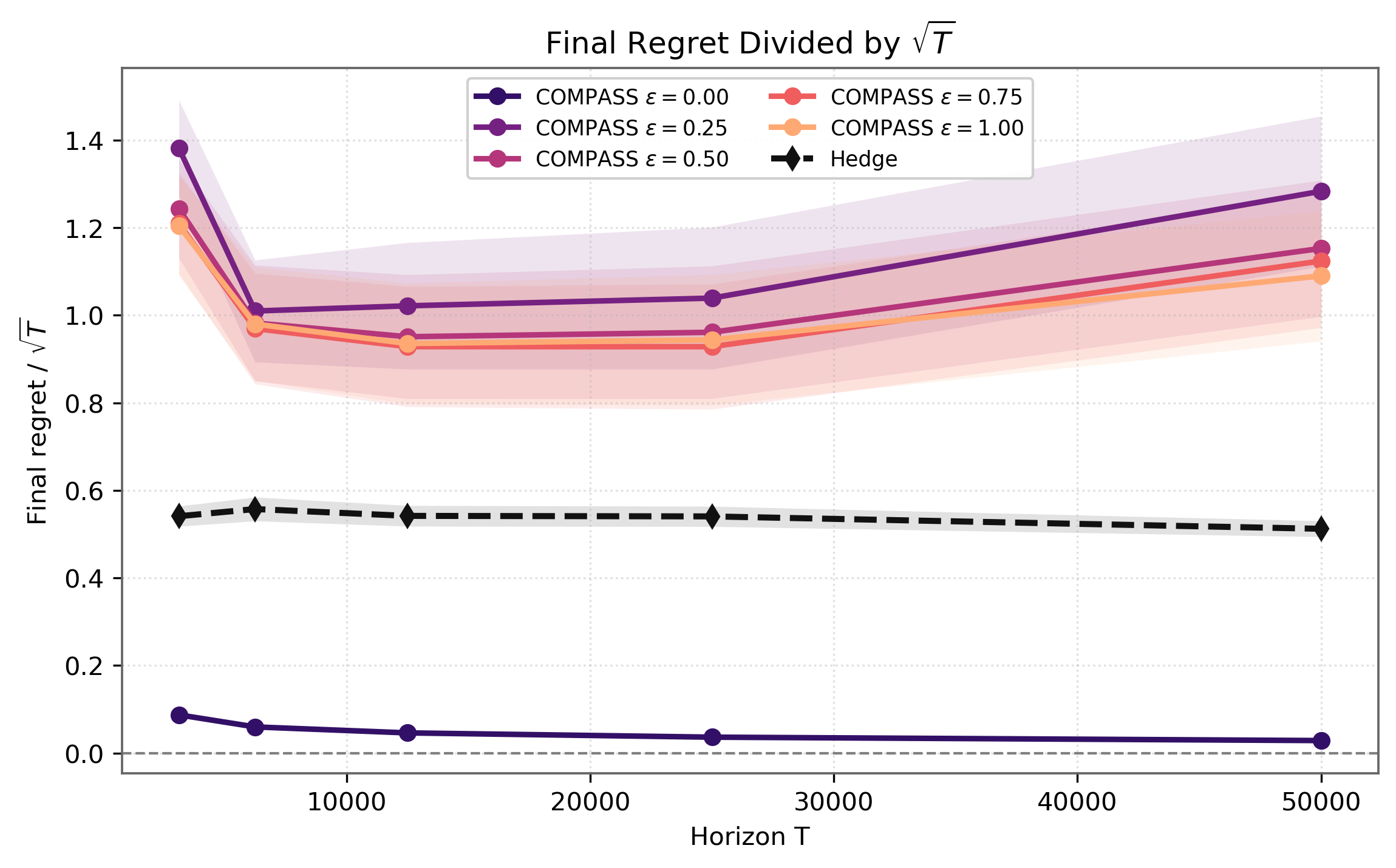}
        \caption{Final regret normalized by $\sqrt{T}$}
    \end{subfigure}
    \caption{\textbf{Distinguishing $\sqrt{T}$ from linear growth.} Curves for regret$/\sqrt{t}$ remain stable rather than drifting upward, final regret scales near-linearly in $\sqrt{T}$, and the normalized final regret stays essentially flat across horizons---three independent pieces of evidence consistent with $O(\sqrt{T})$ growth.}
    \label{fig:minimax-sqrt-scaling}
\end{figure}

Figure~\ref{fig:minimax-sqrt-scaling} speaks clearly. The regret$/\sqrt{t}$ series is essentially stationary rather than trending. The final regret plotted against $\sqrt{T}$ is much closer to linear than the same data plotted against $T$. Most importantly, the normalized final regret---arguably the sharpest diagnostic of all---remains within a narrow band across every horizon we tested. If any of these three diagnostics showed monotone growth with $T$, we would be forced to reject the $\sqrt{T}$ interpretation; none of them do.

\subsubsection{Complement to Experiment II: overlaying comparator trajectories}
\label{app:exp2_complement}

\begin{figure}[t]
    \centering
    \begin{subfigure}[t]{0.49\textwidth}
        \centering
        \includegraphics[width=\linewidth]{alpha_avg_regret.png}
    \end{subfigure}
    \hfill
    \begin{subfigure}[t]{0.49\textwidth}
        \centering
        \includegraphics[width=\linewidth]{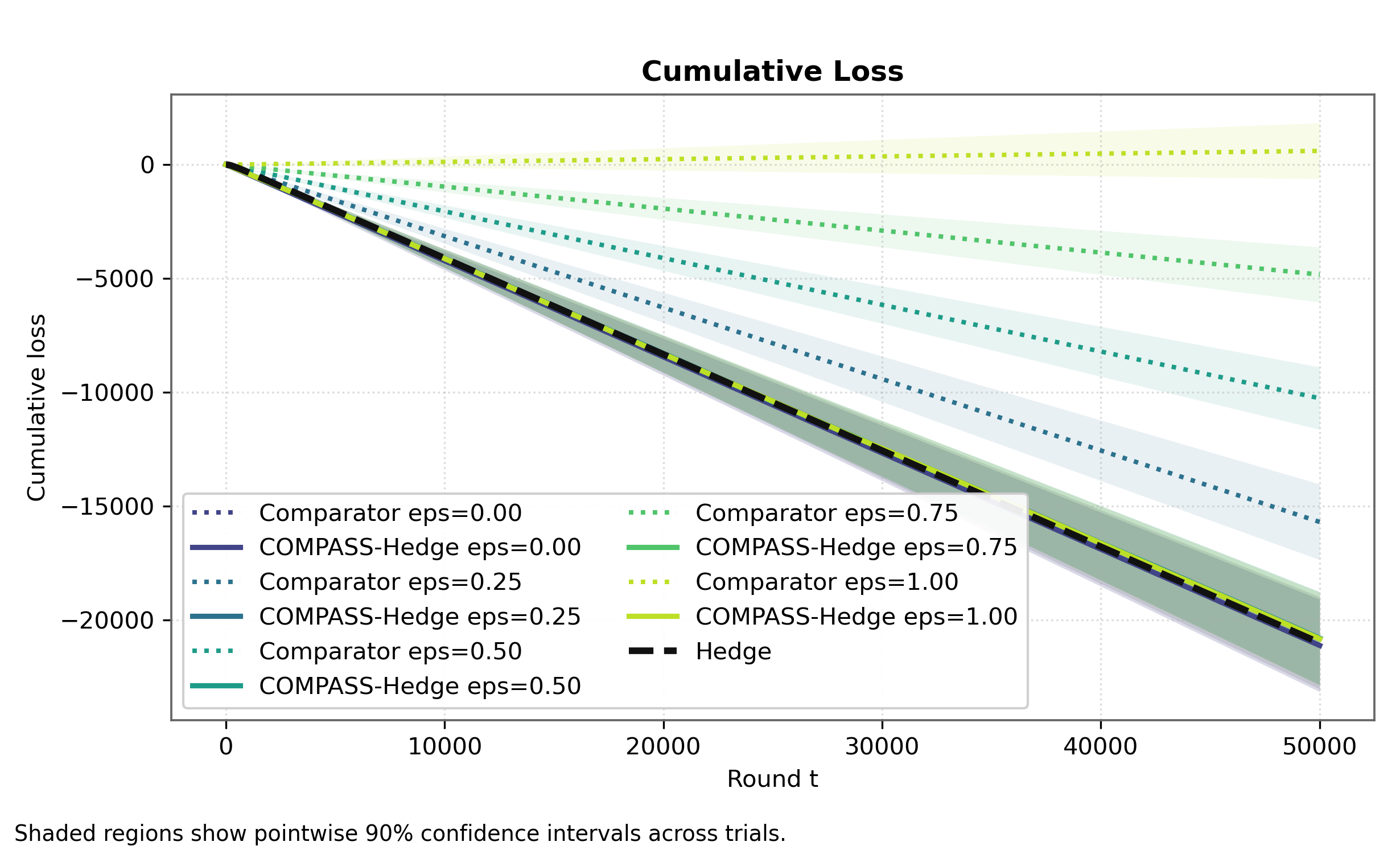}
    \end{subfigure}
    \caption{Regret and cumulative loss against $V^\star t$, with the matched comparator overlaid as a dotted baseline.}
    \label{fig:alpha_values_experiment}
\end{figure}

Figure~\ref{fig:alpha_values_experiment} revisits Experiment II with the comparator trajectories drawn in explicitly as dotted lines. This view makes it unambiguous which curves belong to the learner and which to the static baselines. In the implementation, the learner plays $\alpha_t \mu + (1-\alpha_t)\mu^c$, so small $\alpha_t$ means the method is leaning on the comparator and $\alpha_t$ near one means it has shifted toward the adaptive Hedge component.

Two patterns stand out. First, for small $\varepsilon$, \textsc{Compass-Hedge} (solid) stays deliberately conservative: average regret drops quickly and cumulative loss remains small---the algorithm is exploiting the good comparator rather than chasing marginal gains. Second, and more importantly, the deterioration is \emph{gradual} rather than \emph{catastrophic}. The comparator-alone baselines (dotted) degrade sharply as $\varepsilon$ increases, but \textsc{Compass-Hedge} tracks the dashed Hedge curve and the $V^\star$ benchmark closely throughout. The one-line summary is the same as before: strong comparators are exploited, weak comparators are refused the chance to trap the learner.

\subsubsection{Regret against the comparator family}
\label{app:comparator_regret}

\begin{figure}
    \centering
    \begin{subfigure}[t]{0.49\textwidth}
        \centering
        \includegraphics[width=\linewidth]{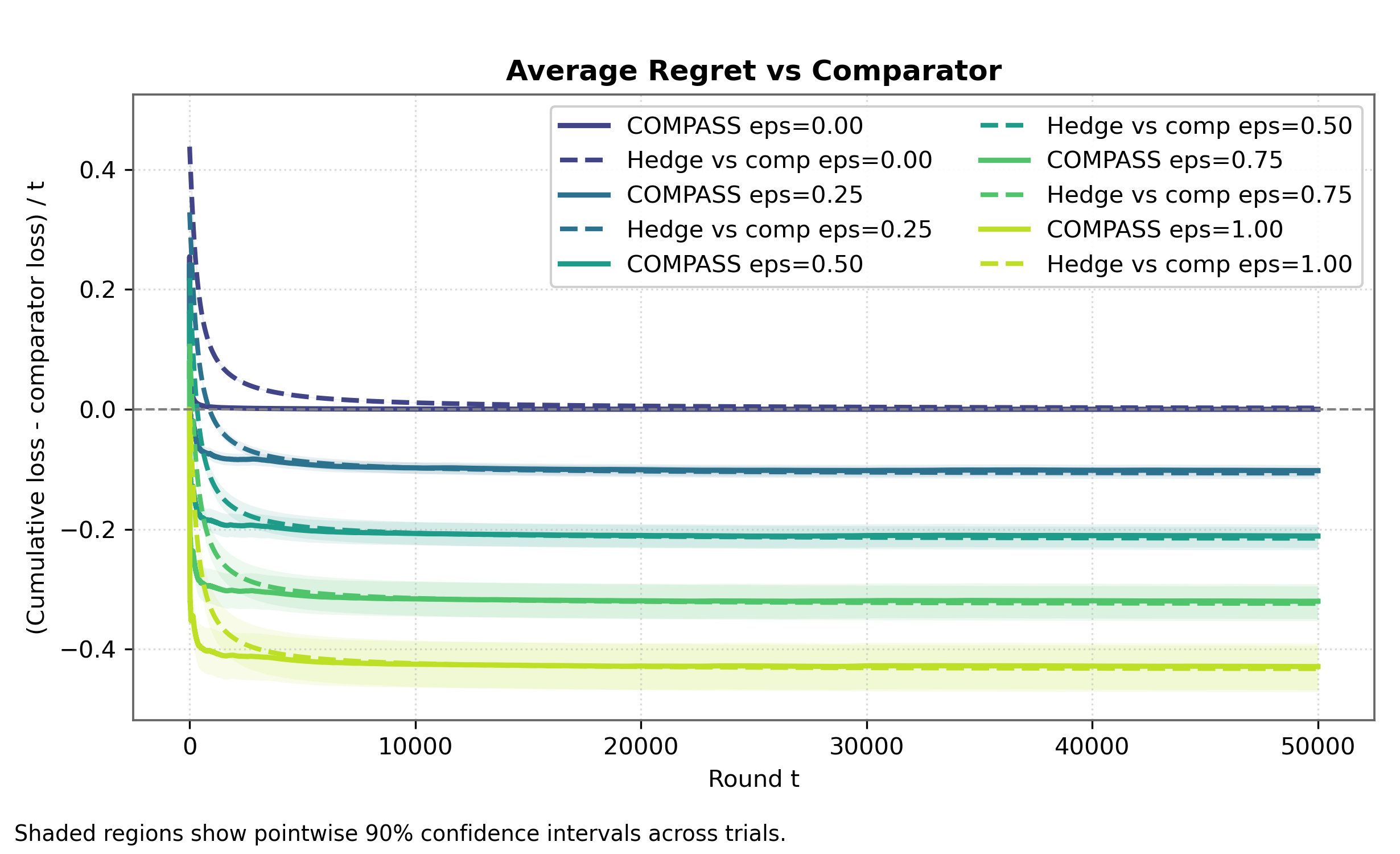}
    \end{subfigure}
    \hfill
    \begin{subfigure}[t]{0.49\textwidth}
        \centering
        \includegraphics[width=\linewidth]{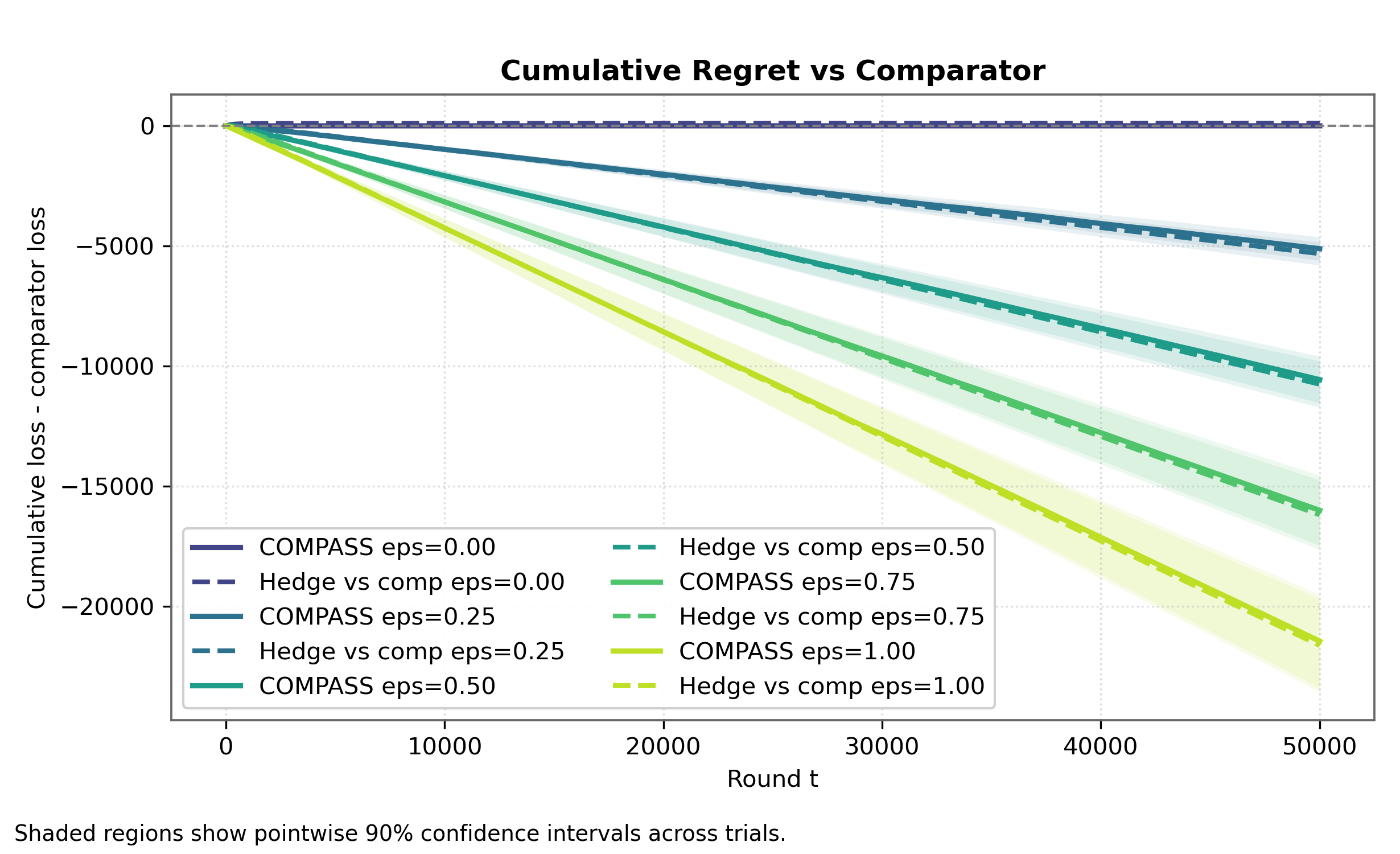}
    \end{subfigure}
    \caption{Regret measured against the matched comparator $q_\varepsilon$ rather than against $V^\star$.}
    \label{fig:regret_against_comparator}
\end{figure}

Figure~\ref{fig:regret_against_comparator} is the natural mirror image of the previous experiment: instead of plotting everything against a single fixed benchmark $V^\star t$, each curve is now measured relative to its own matched comparator $q_\varepsilon$. The left panel shows average regret versus the comparator, $(\text{cumulative loss} - \text{comparator loss})/t$, and the right panel shows the corresponding cumulative regret. Negative values mean the algorithm is outperforming the static comparator.

This plot answers a subtly different question from the value-relative plots. For $\varepsilon = 0$ the comparator \emph{is} (essentially) the equilibrium strategy, and \textsc{Compass-Hedge} sits near zero regret relative to it---exactly as the safety guarantee predicts. As $\varepsilon$ grows and the comparator weakens, both \textsc{Compass-Hedge} and Hedge achieve increasingly negative regret against that comparator, i.e.\ they beat it by a widening margin. The two viewpoints---against $V^\star$ and against $q_\varepsilon$---are complementary: the first establishes optimality, the second establishes safety. \textsc{Compass-Hedge} is the only algorithm in our comparison that does well on both simultaneously.

\subsection{Real-world testbed: the S\&P 500 universe}
\label{app:data_setup}
\paragraph{Why this dataset.}
We need an environment that is simultaneously rich enough to challenge any online learner and transparent enough that we can interpret what the algorithm is doing. Equity markets are a natural choice: they are highly liquid, informationally efficient on average, yet punctuated by heavy-tailed shocks, regime shifts, and occasional structural breaks. This combination---benign on most days, brutally adversarial on a few---maps almost one-to-one onto the ``three worlds'' setting our theory targets.

We use daily Adjusted Closing Prices for all constituents of the S\&P 500 index, sourced from the Center for Research in Security Prices (CRSP) via WRDS. The dataset spans a full decade, from \textbf{January 1, 2015 to December 31, 2024}, yielding $T = 2{,}516$ trading days. 

\paragraph{Survivorship bias and the ``dead arm'' problem.}
A common pitfall in financial backtesting is to restrict the universe to stocks that survive the entire sample window. This removes exactly the arms that should test the robustness of an online allocation rule. We deliberately avoid this survivorship bias by fixing the universe to the January 1, 2015 S\&P 500 cohort and retaining each arm even if it later disappears from the market. In the data pipeline, returns are kept through the asset's last observed trading day, but entries after delisting are left missing and then mapped directly to the worst admissible loss:
\[
c^t(a)=1 \qquad \text{for all } t>t_{\mathrm{dead}}.
\]
Thus, once an arm dies, it remains permanently catastrophic from the learner's perspective. This is exactly the ``dead arm'' scenario against which the safety guarantee of \textsc{Compass-Hedge} is meant to be evaluated.

\paragraph{Mapping returns to losses.}
The algorithms operate on losses in $[0,1]$, so the code uses a clipped linear transform of daily returns rather than a cross-sectional sigmoid. Let $r^t(a)$ denote the raw return of arm $a$ on day $t$, and let $\kappa=0.10$ be the default clipping threshold. We first clip returns to the interval $[-\kappa,\kappa]$,
\[
\tilde r^t(a)=\min\{\max\{r^t(a),-\kappa\},\kappa\},
\]
and then map them to losses via
\[
c^t(a)=\frac{\kappa-\tilde r^t(a)}{2\kappa}.
\]
Under this transformation, a return of $+\kappa$ maps to loss $0$, a return of $0$ maps to loss $1/2$, and a return of $-\kappa$ maps to loss $1$. The clipping step prevents extreme market moves from dominating the scale, while the linear rescaling preserves the ordinal ranking of non-delisted arms and yields bounded losses compatible with Hedge- and FTRL-type updates. Post-delisting missing values are then assigned loss $1$, so dead arms remain permanently worst-case. All regret comparisons are computed using the transformed bounded losses, not raw returns.
\paragraph{Hyperparameters.}
For \textsc{Compass-Hedge} we use $\eta_t = 2\sqrt{\frac{\log A} { t- \operatorname{start} + 1}}$ as the learning rate and keep the coefficient 2 in front of $\widehat{R}_s$ in Algorihtm~\ref{alg:doubling_ftrl_comparator} (\texttt{line 11}) unchanged, matching the prescription of Algorithm~\ref{alg:doubling_ftrl_comparator}. Critically, \emph{no hyperparameter tuning was performed on test data}; the configuration is the one predicted by the theory, which is the strongest form of evidence that the guarantees translate to practice. \textsc{Standard Hedge} uses $\eta= \sqrt{2 \log A / T}$ and initializes uniformly. \textsc{Anytime-Hedge} uses $\eta_t= 2 \sqrt{\log A / t}$. Both use the same transformed losses as \textsc{Compass-Hedge}.
\subsubsection{The two faces of regret: NVIDIA vs.\ SPY}
\label{app:hedge_family}
We use this dataset for Figure~\ref{fig:stochastic}, and choose the safe comparator to be SPY. Operationally, SPY is appended as an additional arm with its own transformed daily loss series, and $\mu^c$ is the unit vector on this SPY arm. This experiment is both the most practical and the most visually revealing. Rather than measuring regret in the abstract, we pit four algorithms---the static \textbf{Comparator} (a SPY-tracking market index), \textbf{Hedge}, its \textbf{Anytime-Hedge} variant, and \textsc{\textbf{Compass-Hedge}}---against two canonical benchmarks on the full ten-year S\&P 500 stream: the best realized arm in hindsight (NVIDIA), and the broad market index (SPY). The two panels tell essentially opposite stories, and only one algorithm emerges looking good on both.

\paragraph{Panel (a): regret against NVIDIA, the best arm in hindsight.}
NVIDIA is the realized top performer of the decade by an extraordinary margin, and all four algorithms accumulate roughly $20$ units of cumulative regret against it---as they should, since matching a single stock that rose over 200$\times$ is not the goal of any of them. The regret values in Figure~\ref{fig:stochastic} are measured in transformed bounded-loss units. The interesting information in this panel lies in the \emph{ordering and spacing} of the curves, which becomes legible in the inset zoom. There, two groups separate cleanly: the \textbf{Comparator} and \textsc{\textbf{Compass-Hedge}} sit at a lower regret level (around $20.0$--$20.5$), while \textbf{Hedge} and \textbf{Anytime-Hedge} sit about $0.7$--$1.0$ units higher (around $21.0$--$21.5$). The message is subtle but important: \textsc{Compass-Hedge} tracks the safe comparator closely even when measured against an aggressive benchmark, whereas the unprotected Hedge variants pay a visible, persistent penalty for their exploration.

\paragraph{Panel (b): regret against SPY, the safe comparator.}
This is the decisive panel. When the benchmark becomes the market comparator itself, the four algorithms fan out dramatically:
\begin{itemize}
    \item The \textbf{Comparator} sits at zero regret by definition (green dotted line).
    \item \textsc{\textbf{Compass-Hedge}} (purple) grows gently to roughly $0.15$ units of comparator regret and then \emph{flattens out} for the remainder of the decade---a textbook realization of the constant-regret safety guarantee.
    \item \textbf{Hedge} and \textbf{Anytime-Hedge} drift continuously upward, reaching above $1.0$ unit of comparator regret by 2020--2021, with further volatile swings afterward. That is almost an order of magnitude worse than \textsc{Compass-Hedge}, and critically, their trajectories show no sign of leveling off.
\end{itemize}

\paragraph{Why this picture matters.}
Panels (a) and (b) together stage the trade-off that motivated the whole paper. An algorithm can look excellent if you let it choose its benchmark: Hedge looks respectable in panel (a), where the comparator penalty is masked by the enormous regret against NVIDIA, but that apparent respectability dissolves the moment we hold it accountable to a safe baseline in panel (b). Conversely, the static Comparator is trivially unbeatable in panel (b) by construction, but it forfeits any hope of catching surges like NVIDIA's in panel (a).

\textsc{Compass-Hedge} is the only curve that remains near the bottom in \emph{both} panels. It neither pays the Hedge-style price of accumulating comparator regret nor commits to the pure Comparator's inability to exploit. This is the three-worlds promise made concrete: on a real, non-stationary, heavy-tailed decade of market data, the algorithm automatically interpolates between safe and adaptive behavior without any regime-specific tuning, and the gap it opens against naive Hedge is visible to the naked eye.

\subsubsection{Two regimes, one algorithm: why the S\&P timeline is a three-worlds stress test}

Our ten-year window naturally decomposes into regimes of very different character, and the per-regime behavior of \textsc{Compass-Hedge} is the empirical counterpart of its theoretical three-worlds guarantee.
\begin{enumerate}
    \item \textbf{Benign / stochastic (e.g.\ 2015--2019, 2021, 2023--2024).} Volatility is low and drift is positive. The algorithm detects that the environment rewards aggression: pseudo-regret flattens quickly as weight concentrates on consistently strong arms.
    \item \textbf{Adversarial / crisis (e.g.\ March 2020 Covid crash, 2022 rate-hike drawdown).} Correlations break down, previously reliable arms collapse, and the comparator itself may underperform transiently. Here the algorithm refuses to double down: the mixing weight $\alpha$ stays low, the learner leans on the uniform comparator, and the catastrophic losses observed in naive Hedge simply do not materialize.
\end{enumerate}

\paragraph{Monte Carlo protocol.}
We run $N = 500$ independent trials. Each trial samples a random subset of $A = 30$ stocks from the universe and instantiates three algorithms on the resulting loss stream: Standard Hedge (benchmark), \textsc{Compass-Hedge} with an Oracle comparator (a high-quality prior concentrated on historically strong performers), and \textsc{Compass-Hedge} with a Uniform comparator (the $1/A$ portfolio---an intentionally weak but safe baseline). The Uniform configuration is the interesting one: it forces the algorithm to demonstrate safety under a comparator that is deliberately weak but diversified; it offers no asset-specific information. The Oracle comparator is constructed offline after observing the full sequence by placing all mass on the best stock ranked by cumulative transformed loss within that trial’s sampled universe. It is not implementable online and is included only as a diagnostic high-quality prior

\paragraph{Interpreting the regret curves.}
We track two complementary metrics throughout the experiments.
\begin{itemize}
    \item \emph{Pseudo-regret} $\mathcal{R}_{\text{pseudo}}$ measures optimality against the best single arm in hindsight. A curve that flattens signals that the selection problem has been ``solved'' in the sense that further learning cannot extract meaningful additional performance.
    \item \emph{Comparator regret} $\mathcal{R}(\mu^c)$ measures safety: the cumulative damage done relative to the protected baseline. The flat trajectories of \textsc{Compass-Hedge} against the Uniform comparator (Fig.~\ref{app:fig:comp_uniform}) are the empirical fingerprint of the $\tilde O(1)$ safety bound proved in Theorem~\ref{thm:main}.
\end{itemize}

\paragraph{The price of protection.}
The initial lag of \textsc{Compass-Hedge} in Fig.~\ref{app:fig:pseudo_uniform} is not a weakness but an accounting of the \emph{cost of safety}: the algorithm needs a short burn-in to verify that aggression is warranted before shifting weight away from the comparator. On a ten-year horizon this cost is negligible; on shorter horizons it is the explicit price paid for the guarantee that nothing catastrophic can happen.

\begin{figure}[t]
    \centering
    \setlength{\abovecaptionskip}{4pt}
    \setlength{\belowcaptionskip}{0pt}
    \begin{subfigure}[b]{0.48\textwidth}
        \centering
        \includegraphics[width=0.98\textwidth]{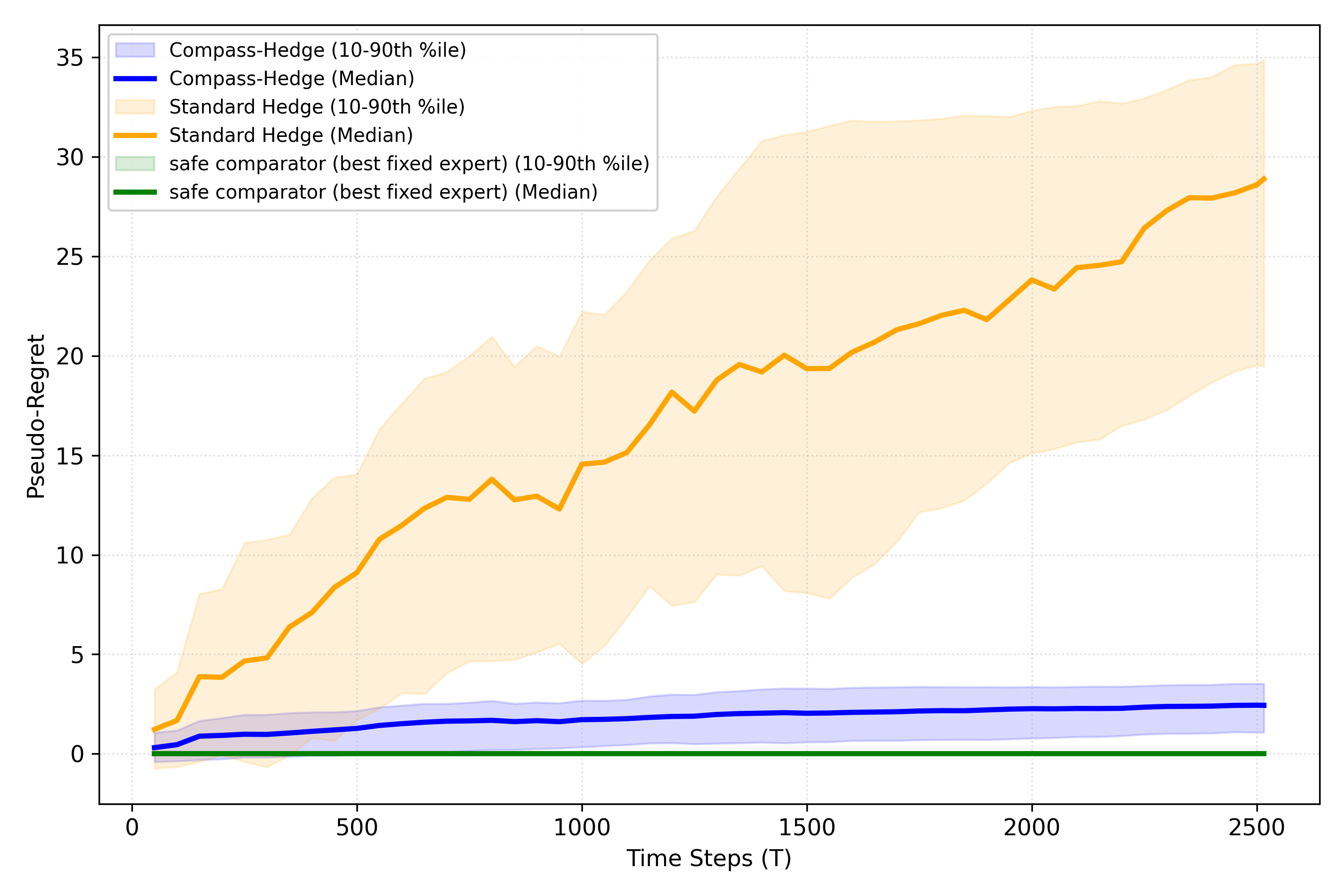}
        \caption{Pseudo-regret, Oracle comparator}
        \label{app:fig:pseudo_oracle}
    \end{subfigure}
    \hfill
    \begin{subfigure}[b]{0.48\textwidth}
        \centering
        \includegraphics[width=0.98\textwidth]{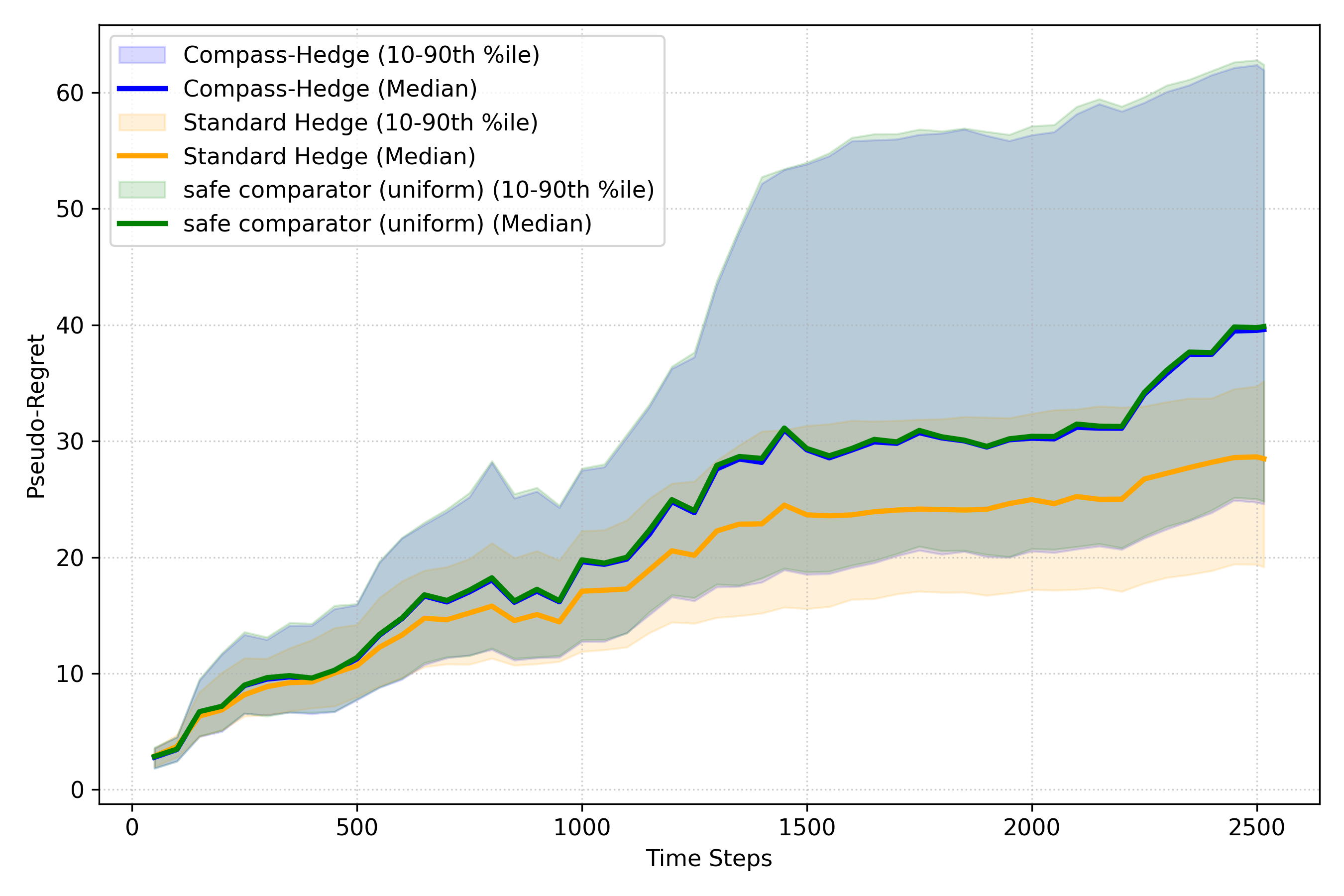}
        \caption{Pseudo-regret, Uniform comparator}
        \label{app:fig:pseudo_uniform}
    \end{subfigure}
    \vspace{1em}
    \begin{subfigure}[b]{0.48\textwidth}
        \centering
        \includegraphics[width=0.98\textwidth]{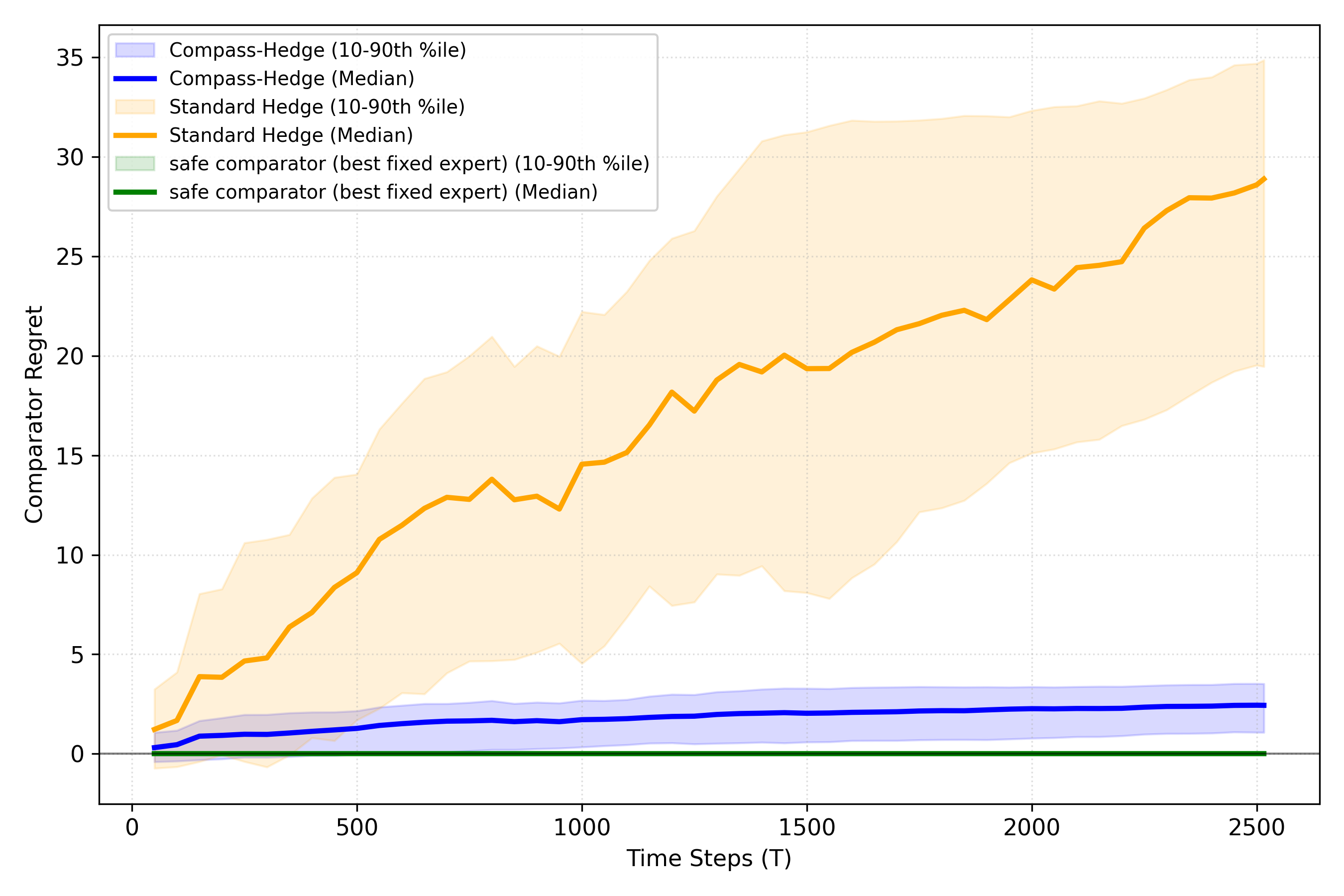}
        \caption{Comparator regret, Oracle}
        \label{app:fig:comp_oracle}
    \end{subfigure}
    \hfill
    \begin{subfigure}[b]{0.48\textwidth}
        \centering
        \includegraphics[width=0.98\textwidth]{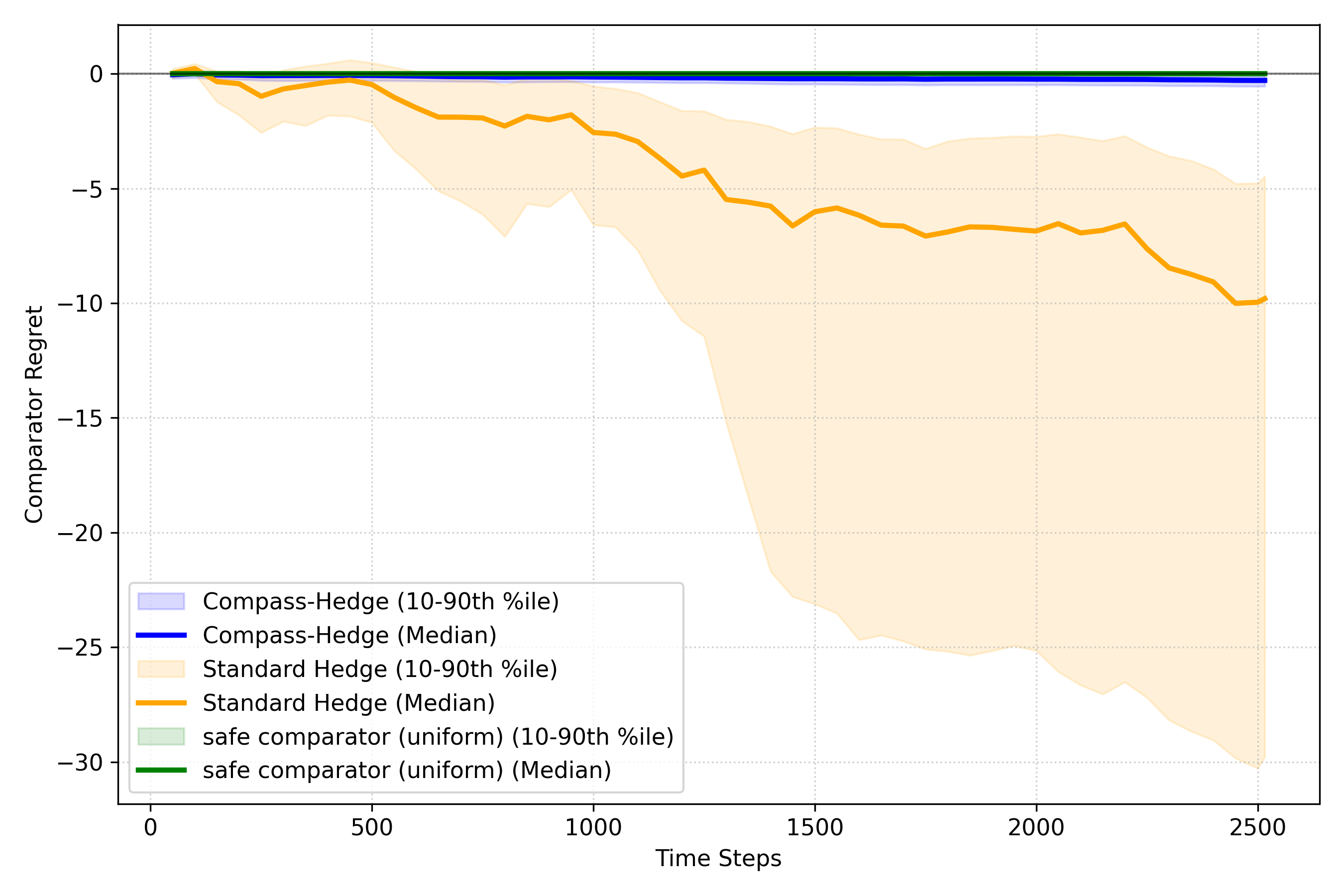}
        \caption{Comparator regret, Uniform}
        \label{app:fig:comp_uniform}
    \end{subfigure}
    \vspace{0.5em}
    \caption{\textbf{Median regret across 500 trials.} Shaded bands are $10^{\text{th}}$--$90^{\text{th}}$ percentiles. Panels (a) and (c): under an informative Oracle prior, \textsc{Compass-Hedge} (blue) tracks the best expert without incurring comparator regret. Panels (b) and (d): under a deliberately weak Uniform prior, Standard Hedge (orange) exhibits high variance and sporadic blow-ups, whereas \textsc{Compass-Hedge} maintains a nearly flat comparator-regret trajectory---the empirical signature of the $\tilde O(1)$ safety guarantee.}
    \label{app:fig:combined_results}
\end{figure}
\end{document}